\documentclass[twoside]{article}
\pdfoutput=1
\usepackage[accepted]{aistats2022}
\usepackage[round]{natbib}

\usepackage[utf8]{inputenc} % allow utf-8 input
\usepackage[T1]{fontenc}    % use 8-bit T1 fonts
\usepackage[breaklinks]{hyperref}
\usepackage{url}            % simple URL typesetting
\usepackage{booktabs}       % professional-quality tables
\usepackage{amsfonts}       % blackboard math symbols
\usepackage{nicefrac}       % compact symbols for 1/2, etc.
\usepackage{microtype}      % microtypography
\usepackage{microtype}
\usepackage{graphicx}
\usepackage{booktabs} % for professional tables
\usepackage{amsmath}
\usepackage[capitalise]{cleveref}
\usepackage{algorithm}      % algorithm
\usepackage[noend]{algpseudocode} % algorithm
\usepackage{xcolor}
\usepackage{mdframed}
\usepackage{amsthm}

\newtheorem{assumption}{Assumption}
\newtheorem{theorem}{Theorem}

\newtheorem{definition}{Definition}
\newtheorem{proposition}{Proposition}

\usepackage{pgfplots}
\pgfplotsset{width=8cm,height=5cm,compat=1.9}
\newlength\figureheight
\newlength\figurewidth
\usepackage{subfig}
\usepackage{amsmath}
\usepackage{amssymb}
\usepackage{amsfonts}
\usepackage{amstext}
\usepackage{xcolor}
\usepackage{xspace}
\usepackage{appendix}
\usepackage{multirow}
\usepackage{comment}

\usepackage[normalem]{ulem}

\makeatletter
\newcommand*{\addFileDependency}[1]{% argument=file name and extension
  \typeout{(#1)}
  \@addtofilelist{#1}
  \IfFileExists{#1}{}{\typeout{No file #1.}}
}
\makeatother

\newcommand*{\myexternaldocument}[1]{%
    \externaldocument{#1}%
    \addFileDependency{#1.tex}%
    \addFileDependency{#1.aux}%
}
\usepackage{xr}
\myexternaldocument{supplement}
\crefname{appsec}{Appendix}{Appendices}

\newcommand{\answerTODO}[1][]{\textcolor{red}{\bf [TODO]}}

\definecolor{mygreen}{rgb}{0.2, 0.7, 0.2}
\definecolor{myorange}{rgb}{0.9, 0.5, 0.0}

\definecolor{blue}{HTML}{5DA5DA}
\definecolor{orange}{HTML}{FAA43A} 
\definecolor{green}{HTML}{60BD68} 
\definecolor{pink}{HTML}{F17CB0} 
\definecolor{brown}{HTML}{B2912F} 
\definecolor{purple}{HTML}{B276B2} 
\definecolor{yellow}{HTML}{DECF3F} 
\definecolor{red}{HTML}{F15854} 
\definecolor{gray}{HTML}{4D4D4D} 

\newcommand{\name}[1]{{\textsc{#1}}\xspace}

\newcommand{\bnn}{\name{bnn}}
\newcommand{\bnns}{\textsc{bnn}s\xspace}
\newcommand{\relu}{\name{ReLU}}

\newcommand{\mcmc}{\name{mcmc}}

\newcommand{\sde}{\name{sde}}
\newcommand{\sdes}{\textsc{sde}s\xspace}

\newcommand{\ode}{\name{ode}}
\newcommand{\odes}{\textsc{ode}s\xspace}

\newcommand{\sg}{\name{sg}}
\newcommand{\sgd}{\name{sgd}}
\newcommand{\sgmcmc}{\name{sg-mcmc}}
\newcommand{\sgld}{\name{sgld}}

\newcommand{\sghmc}{\name{sghmc}}
\newcommand{\sdehmc}{\name{shmc}}
\newcommand{\hmc}{\name{hmc}}

\newcommand{\dataset}{\mathrm{D}}

\newcommand{\boston}{\name{boston}}
\newcommand{\concrete}{\name{concrete}}
\newcommand{\energy}{\name{energy}}
\newcommand{\yacht}{\name{yacht}}
\newcommand{\iono}{\name{ionosphere}}
\newcommand{\vehicle}{\name{vehicle}}

\newcommand{\mnll}{\name{mnll}}
\newcommand{\rmse}{\name{rmse}}

\newcommand{\acf}{\name{acf(1)}}

\newcommand{\thetadim}{d}

\newcommand{\zvect}{\mathbf{z}}
\newcommand{\dzvect}{d\mathbf{z}}

\newcommand{\euler}{\name{euler}}
\newcommand{\rktwo}{\name{rk2}}
\newcommand{\lietrotter}{\name{lie-trotter}}
\newcommand{\leapfrog}{\name{leapfrog}}
\newcommand{\symmetric}{\name{symmetric}}
\newcommand{\mtthree}{\name{mt3}}

\newcommand{\xvect}{\mathbf{x}}

\newcommand{\wvect}{\mathbf{w}}

\newcommand{\zerovect}{\mathbf{0}}

\newcommand{\thetavect}{\boldsymbol{\theta}}
\newcommand{\dthetavect}{d\boldsymbol{\theta}}

\newcommand{\Cfrict}{C\boldsymbol{I}}

\newcommand{\rvect}{\mathbf{r}}

\newcommand{\p}{\partial}

\newcommand{\eye}{\boldsymbol{I}}

\usepackage[bottom]{footmisc}

\begin{document}



\twocolumn[

\aistatstitle{Revisiting the Effects of Stochasticity for Hamiltonian Samplers}

\aistatsauthor{ Giulio Franzese \And Dimitrios Milios \And  Maurizio Filippone \And Pietro Michiardi}

\aistatsaddress{ EURECOM Data Science Department   Biot (France) } ]

%\author{ Giulio Franzese\\ EURECOM Data Science Department \\
%   Biot (France) \\ \And Dimitrios Milios \\ EURECOM Data Science Department \\
%   Biot (France) \\\And Maurizio Filippone \\ EURECOM Data Science Department \\
%   Biot (France) \\\And Pietro Michiardi\\ EURECOM Data Science Department \\
%   Biot (France) \\}

\begin{abstract}
We revisit the theoretical properties of Hamiltonian stochastic differential equations (\sdes) for Bayesian posterior sampling, and we study the two types of errors that arise from numerical \sde simulation: the discretization error and the error due to noisy gradient estimates in the context of data subsampling. 
Our main result is a novel analysis for the effect of mini-batches through the lens of differential operator splitting, revising previous literature results. 
The stochastic component of a Hamiltonian \sde is decoupled from the gradient noise, for which we make no normality assumptions.
This leads to the identification of a convergence bottleneck: when considering mini-batches, the best achievable error rate is $\mathcal{O}(\eta^2)$, with $\eta$ being the integrator step size.
Our theoretical results are supported by an empirical study on a variety of regression and classification tasks for Bayesian neural networks.

\end{abstract}

\section{INTRODUCTION}
\label{sec:introduction}

%% Hamiltonian Monte Carlo (\hmc) has remained popular in recent literature for the Bayesian treatment of complex non-linear models such as Bayesian Neural Networks (\bnns) \cite{Carbone2020,wenzel2020good}.
%% The sampling process is defined as a randomised simulation of a Hamiltonian system which allows a more efficient exploration of the parameter space compared to approaches based on purely random walks, such as Metropolis-Hastings.
%% Nevertheless, \hmc presents significant computational challenges for large datasets, 
%% as it requires access to the full gradient of the Hamiltonian system.

Hamiltonian Monte Carlo (\hmc) is a popular approach to obtain samples from intractable distributions \citep{neal1996bayesian,neal2011mcmc,Hoffman2014}. %\maurizio{cite the classic reference of Neal's MCMC handbook and the No U-Turn Sampler}.
%The sampling process is defined as a randomized simulation of a Hamiltonian system, which allows for a more efficient exploration of the parameter space compared to approaches based purely on random walks, such as the Metropolis-Hastings algorithm.
It presents, however, significant computational challenges for large datasets, as it requires access to the full gradient of the associated Hamiltonian system.
Stochastic-Gradient \hmc (\sghmc) \citep{chen2014stochastic} was proposed as a scalable alternative to \hmc, by admitting noisy estimates of the gradient using mini-batching. In \sghmc, the Hamiltonian dynamics is modified so as to include a friction term that counteracts the effects of the gradient noise.
This approach has proven effective in dealing with the difficulties in sampling from the posterior distribution over model parameters of Bayesian Neural/Convolutional Networks (\bnns) \citep{wenzel2020good,tran2020need}.
%\pietro{I thought this par was gone!}

%Historically, such methods have been proposed as scalable alternatives, to traditional sampling schemes, such as Hamiltonian Monte Carlo (\hmc).

Stochastic gradient (\sg) methods have been extensively studied as a means for Markov chain Monte Carlo (\mcmc)-based algorithms to scale to large data. 
Variants of \sgmcmc algorithms have been studied through the lenses of first \citep{welling2011bayesian,ahn2012bayesian,NIPS2013_4883} or second-order \citep{chen2014stochastic,ma2015complete} Langevin dynamics; these are mathematically convenient continuous-time processes which correspond to discrete-time gradient methods with and without momentum, respectively.
Langevin dynamics are formally captured by an appropriate set of stochastic differential equations (\sdes),
%We revisit the theoretical properties of Hamiltonian stochastic differential equations (\sdes) (also known as second order/ underdamped Langevin dynamics) for Bayesian posterior sampling. 
whose theoretical properties have been extensively studied \citep{kloeden2013numerical,debussche2012weak} with a particular emphasis on the stationary property of these processes \citep{abdulle2014high,MILSTEIN200781}.
As in \citet{abdulle2014high,abdulle2015long}, we are interested in the asymptotic (in time) performance of such sampling schemes. The reader is referred to \citet{vollmer2016exploration,gao2018global,gao2018breaking,futami2020accelerating,xu2018global} for additional insights into the non-asymptotic behavior of these methods.

In this work, we seek to re-evaluate the connections between \sg and stochastic Hamiltonian dynamics (also known as second order or underdamped Langevin dynamics) with the aim of improving our current understanding of the goodness of sampling from intractable distributions when considering mini-batching.
We consider a system with potential $U(\thetavect)$ which is the negative of the logarithm of the density function associated with the distribution we aim to sample from.
We then introduce position variables $\thetavect\in \mathbb{R}^{\thetadim}$ (i.e.,\ parameters) and momentum variables $\rvect\in \mathbb{R}^{\thetadim}$ obeying the following \sde:
\begin{equation}\label{hamsde_exp}
\begin{split}
d\rvect(t)&=-\nabla_{\thetavect}U(\thetavect(t))dt - C\mathbf{M}^{-1}\mathbf{r}(t)dt+\sqrt{2C}d\mathbf{w}(t)\\
d\thetavect(t)&=\mathbf{M}^{-1}\mathbf{r}(t)dt.
\end{split}
\end{equation}
This is an extension of an Hamiltonian system with a friction term and an appropriately
scaled
Brownian
motion $\mathbf{w}(t)$,
where $C>0$\footnote{
    Here for simplicity we consider $C \in \mathbb{R}$, but in general $C$ can be a matrix.
}
and $\mathbf{M}$ is a symmetric, positive definite matrix (a.k.a.\ mass matrix).
%The stationary distribution of \eqref{hamsde_exp} over $\theta$ can approximate a target Bayesian posterior for an appropriate specification of the potential.
A common assumption in the literature \citep{chen2014stochastic,
    ahn2012bayesian,
    ma2015complete}, is that the noise associated with stochastic estimates of $\nabla_{\thetavect}U(\thetavect(t))$
is normally distributed. This noise is linked with the additive Brownian motion term in
\cref{hamsde_exp}.

%the Brownian motion term can adequately model the noise associated with stochastic estimates of $\nabla_{\thetavect}U(\thetavect(t))$.

%
%We revisit the \sde formalism that describes the dynamics of Hamiltonian systems with a random component.
%We consider Hamiltonian-based \sdes with a target posterior as stationary distribution, and we review the role of friction so that it is disassociated from the gradient noise.
%
%\maurizio{How about promoting some background equations here to make things clearer for the reader?} 
%
%More specifically, we show that the Brownian motion term is not a good model for the stochasticity of the gradient introduced by the mini-batching scheme, and 
In this work we challenge this assumption and
we advocate that \sg should be completely decoupled from the \sde dynamics. More specifically, we show that the Brownian motion term is not a good model for the stochasticity of the gradient introduced by the mini-batching scheme (see also \cref{sec:sgdvssde} for an extended discussion).
In this sense, our framework is similar to \citet{chen2015convergence};
%however, we reevaluate the effects through a geometrical perspective.\dimitris{this sentence need rewriting}
however we propose an interpretation of the effect of mini-batching through the lenses of  differential operator splitting, by leveraging the huge literature concerning the simulation of high dimensional Hamiltonian systems \citep{childs2019theory,suzuki1977convergence,hatano2005finding,childs2019nearly,low2019well}.
Earlier attempts \citep{pmlr-v37-betancourt15,shahbaba2014split} to use operator splitting as a tool to describe mini-batching focus on \hmc, whereas Hamiltonian \sde dynamics have never been studied under this formalism, to the best of our knowledge. Recently \citet{zou2021convergence} have studied non-asymptotic convergence for \hmc when considering mini-batches,  but these results are limited to strongly convex potentials.

% \dimitris{I like this paragraph... but it does not fit here, does it?}
%The stochastic processes described by \sdes have been intensively studied \cite{kloeden2013numerical,debussche2012weak}, with some of the authors putting a particular emphasis on the stationary property of these processes \cite{abdulle2014high,MILSTEIN200781}. 
%In this work, we leverage these results to present quantitatively precise claims about sampling schemes in the context of Bayesian inference for modern machine learning problems.

    %\item First, it is not true that you cannot converge to the true posterior with \hmc and minibatches. Our theory shows that the order of convergence is 2 but the constant is really large. This explains why in practice you cannot get good results. This is discussed in Proposition \cref{hmcminib}.
    %\item Second, there is not such a thing as \sg gaussian noise. Our derivation in \cref{sec:minibatches} does not rely on Gaussianity, we also have a nice discussion in Appendix \cref{sec:sgdvssde}. We here show infact that with a perfect oracle, there is no improvement in counterbalancing the noise
    %\item Third, \sghmc is better than \hmc even with full batch. This is the key result of Proposition \cref{hmcfulb}. As a side note, this disproves the narrative of \cite{chen2014stochastic} in which noise and friction are considered to counterbalance the \sg noise. Even in the full batch case having a better discretization scheme is helpful, and when considering minibathces counterbalancing it does not have any impact.

\begin{figure}
    \centering
    \includegraphics[width=0.48\textwidth]{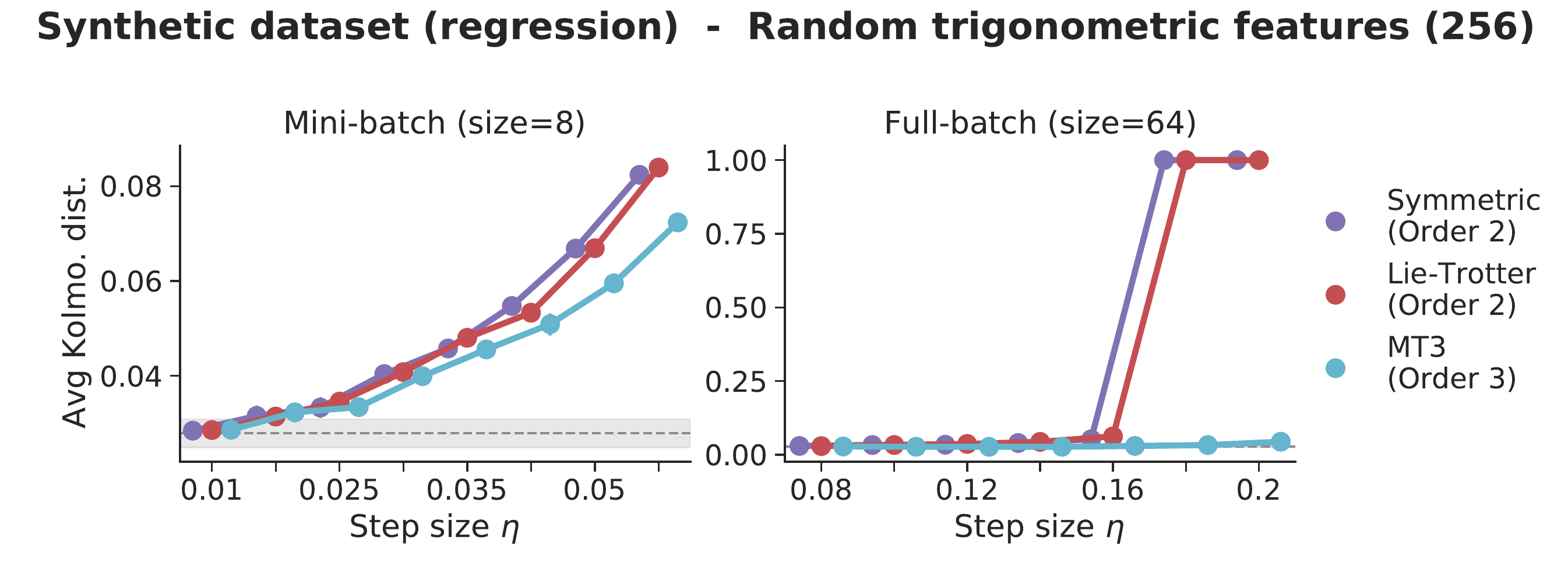}
    \caption{We evaluate integrators of different orders based on the distance from the true posterior distribution for a Gaussian-linear system. Left: for mini-batch size equal to $8$, there is an area for which the 3-rd order integrator does not perform any better than integrators of lower order. Right: There is no such bottleneck for the full-batch version.}\label{fig:minibottle} 
    %\maurizio{I would change the text to better harmonize with the text in the introduction. Why not adding that the observation in the figure on the left shows indeed the bottleneck etc...}
\end{figure}

Our main contribution can be summarized as the derivation of new convergence results considering the two types of error induced by the simulation of the proposed \sde scheme:
the discretization and the \sg errors.
% These are the discretization error that is due to the numerical simulation of \sdes, and the \sg error that is due to noisy estimates of the gradient.
The discretization of \sdes has been extensively studied in the literature \citep{kloeden2013numerical,debussche2012weak,abdulle2014high,MILSTEIN200781}. 
%with some of the authors putting a particular emphasis on the stationary property of these processes \cite{abdulle2014high,MILSTEIN200781}. 
In \cref{sec:revisit}, we leverage these results to present quantitatively precise claims about sampling schemes in the context of Bayesian inference for modern machine learning problems.
%We study their effects through the lenses of the differential operators and weak backward error analysis \cite{debussche2012weak}.
In \cref{sec:minibatches}, our treatment of mini-batches in terms of differential operator splitting does not rely on Gaussian assumptions regarding the form of the noise.
Given a numerical integrator of order $p$ and step size $\eta$, we show that although the discretization error vanishes at a rate $O(\eta^p)$, the \sg error vanishes at rate $O(\eta^2)$.
We thus identify a convergence bottleneck, as demonstrated on an example (details in \cref{ssec:experiment_setup_appendix}) in \Cref{fig:minibottle}, introduced by the mini-batching. 
%\maurizio{maybe place the figure on top of page 2}
%We show that for a careful selection of numerical integrator, both errors vanish at a rate $O(\eta^2)$ where $\eta$ is the step size of the integrator, and we identify a convergence bottleneck, \Cref{fig:minibottle}, introduced by the mini-batching. 
This result is in contrast with the current understanding of convergence properties of sampling algorithms \citep{chen2015convergence}.
%treated in \cite{chen2015convergence} for the case of generic numerical integrators and generic potentials.
Our work then updates the previous best result about convergence rates for a broad range of integrators and with minimal assumptions about the mini-batching process. 

%Thus we interpret the undeniable empirical success of methods such as \sghmc \cite{springenberg2016bayesian,wenzel2020good} under a new light.

Our second contribution (\cref{sec:hmcvssdehmc}) is a reinterpretation of the \hmc algorithm  with partial momentum refreshment \citep{neal2011mcmc,horowitz1991generalized} as an integrator for Hamiltonian \sdes \citep{abdulle2015long}. This connection provides insights into the behavior of classical \hmc schemes and extends the results of our theoretical analysis.

%The exact comparison of the convergence rates of \hmc with partial momentum refreshment \cite{neal2011mcmc,horowitz1991generalized} with respect to the proposed \sde formalism shows that although both approaches are of order $O(\eta^2)$, the constants are significantly larger for the former approach, which translates to poorer performance. 
%Importantly, we stress that our formalism can be used to provide similar results for the \hmc scheme used with mini-batches. 

In \cref{sec:experiments}, we conduct an extensive experimental campaign that corroborates our theory on the convergence rate of various Hamiltonian-based \sde schemes, by exploring the behavior of step size and mini-batch size for a large number of models and datasets. % \maurizio{we should maybe be consistent with the use of learning rate or step-size.}

\section{HAMILTONIAN SDES FOR SAMPLING}
\label{sec:revisit}

%\begin{mdframed}[backgroundcolor=green!20]
%   \textbf{Summary}: This section presents precise convergence rate conditions for generic class of integrators.
%\end{mdframed}

%\pietro{I removed the summary boxes.}
%\maurizio{Please use macros defined in notation.tex for bold symbols!}

%% Given a dataset of observations $\dataset=\{\xvect_i\}_{i=1}^{N}$, the Bayesian treatment of supervised machine learning can be summarized as the combination of a prior belief $p(\thetavect)$ and a likelihood function $p(\dataset | \thetavect)$, into a posterior distribution $p(\thetavect|\dataset) \propto p(\dataset|\thetavect) p (\thetavect)$.
Given a dataset of observations $\dataset=\{\xvect_i\}_{i=1}^{N}$, the Bayesian treatment of machine learning models can be summarized as the combination of a prior belief $p(\thetavect)$ and a likelihood function $p(\dataset | \thetavect)$, into a posterior distribution $p(\thetavect|\dataset) \propto p(\dataset|\thetavect) p (\thetavect)$.
Since the posterior is analytically intractable for non-linear models such as \bnns \citep{Bishop06}, our goal is to draw samples from the density $p(\thetavect|\dataset)$, which is expressed implicitly in exponential form as $p(\thetavect|\dataset)\propto\exp(-U(\thetavect))$, where:
\begin{flalign}
 &  - U(\thetavect)=\sum\limits_{i=1}^N\log p(\xvect_i|\thetavect) + \log p(\thetavect),
    \qquad \nonumber\\&\text{and}
    \qquad p(\dataset| \thetavect) = \prod_{i=1}^N p(\xvect_i |\thetavect).
    \label{eq:potential}
\end{flalign}
In Hamiltonian systems, $U(\thetavect)$ is known as the
\emph{potential}, which,
    together
with
a \emph{kinetic energy} term, yields the Hamiltonian function
$H(\zvect)=U(\thetavect)+\nicefrac{1}{2}||\mathbf{M}^{-1}\mathbf{r}||^2$.
The vector $\zvect=[\mathbf{r}^\top,\thetavect^\top]^\top$ denotes the overall system state,
which includes the position $\thetavect$ (i.e.,\ parameters) and the conjugate momentum $\mathbf{r}$,
while $\mathbf{M}$ is a symmetric, positive definite mass matrix. %, a.k.a. the mass matrix.
Then, we consider the following \sde, which is a compact form of \cref{hamsde_exp}:
\begin{equation}\label{hamsde}
    \dzvect(t)=\begin{bmatrix} -\Cfrict & -\eye\\ \eye & \zerovect \end{bmatrix}\nabla H(\zvect(t))dt+\begin{bmatrix} \sqrt{2C} d\mathbf{w}(t) \\ \zerovect \end{bmatrix}.
\end{equation}
%The stationary distribution of this stochastic process is associated 
We study the stationary behavior of the process of \cref{hamsde}, defining 
$\nabla_\zvect=[\nabla_\mathbf{r}^\top,\nabla_{\thetavect}^\top]^\top$, 
through the following differential operator:
\begin{flalign}
    &\mathcal{L}=-\left(\nabla^\top_{\thetavect}U(\thetavect)\right)\nabla_{\mathbf{r}}+\left(\left(\mathbf{M}^{-1}\mathbf{r}\right)^\top\right)\nabla_{\thetavect}\nonumber\\& -C\left(\left(\mathbf{M}^{-1}\mathbf{r}\right)^\top\right)\nabla_{\mathbf{r}}+C\nabla_{\mathbf{r}}^\top\nabla_{\mathbf{r}},\label{eq:infgen}
\end{flalign}
which is known as the \emph{infinitesimal generator}.
%We split $\mathcal{L}$ into two components $\mathcal{H},\mathcal{D}$ that correspond to a pure deterministic Hamiltonian evolution and a diffusion-like Ornstein-Uhlenbeck processes respectively. 
The Fokker Planck equation, that can be used to obtain the stationary distribution $\rho_{ss}(\zvect)$ of the stochastic process writes as  $\mathcal{L}^\dagger\rho_{ss}=0$, where $^\dagger$ indicates the adjoint of the operator. Some sample paths for a simple two-parameter logistic regression can be seen in \cref{fig:sample_paths},
where we vary the constant $C$. Although the stationary distribution is independent of $C$, the transient dynamics change, resulting in paths of different form, as we see in the figure.
Essentially, $C$ is a user-defined parameter whose effect we explore in \cref{ssec:explore_C} for a wide range of machine learning models.
Then we have the following theorem:
\begin{theorem}
\label{th:stationary}
    For an ergodic stochastic process described by the \sde of \cref{hamsde} with stationary distribution $\rho_{ss}(\zvect)$ we have:
    \begin{equation}
        \rho_{ss}(\zvect) \propto \exp(-H(\zvect))= \exp(-U(\thetavect)-\nicefrac{1}{2}||\mathbf{M}^{-1}\mathbf{r}||^2)
    \end{equation}
\end{theorem}
The proof can be found in \cref{sec:proofhamsde}.
A direct implication is that simulating the stochastic process allows us to compute ergodic averages of functions of
interest $\phi(\zvect)$, of the form $\int\phi(\zvect)\rho_{ss}(\zvect)\dzvect$ (details in \cref{sec2:appendix}). Since $\rho_{ss}(\zvect)=\rho_{ss}
(\thetavect)\rho_{ss}(\rvect)$, the procedure can be used to perform Bayesian averages of functions of $\thetavect$ only, $\phi(\thetavect)$ , as the following holds
$ \int \phi(\thetavect)p(\thetavect|\dataset)\dthetavect=\int \phi(\thetavect)\rho_{ss}(\thetavect)\dthetavect$. In this paper the theoretical derivations are carried out for generic functions of $\zvect$.

We shall refer to our scheme as \sde-based \hmc (\sdehmc).
This is different from \sghmc, for which \cref{hamsde} is modified so that the Brownian motion term has covariance $2 (\mathbf{C} - \tilde{\mathbf{V}})$, where $\tilde{\mathbf{V}}$ is an estimate of the covariance of the gradient.
This was done to counterbalance the effect of the stochastic gradient, which we believe is not necessary, as we discuss in \cref{sec:experiments}.
%To conclude, the friction term $C$, as well as the diffusion term $2 C$, are simply necessary to ensure that the \sde has stationary distribution $\rho_{ss}(\thetavect)$  that corresponds to the posterior of interest $p(\thetavect|\dataset)$.

%we consider $C=5$ in most of the experiments of \cref{sec:experiments}.
%which might have an effect on practical applications{\color{red} To be discussed toghether}.
%In \cref{sec:experiments} we explore experimentally the effect of $C$ on a wide range machine learning models. 
%Although an in-depth exploration of the effect of $C$ is out of the scope of this work.
\begin{figure*}[t]
    \centering
    \includegraphics[width=\textwidth]{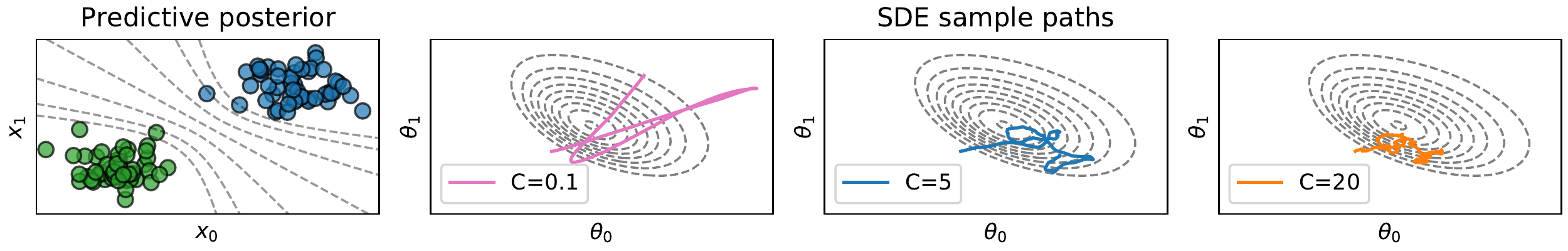}
    \caption{Sample paths for simple a two-parameter logistic regression problem (shown in the leftmost panel). We show \sde paths for different choices of friction coefficient $C$.}
    \label{fig:sample_paths}
\end{figure*}

%{\color{red} Importantly, we stress that this friction does not serve the purpose of counterbalancing the noise as in \cite{chen2014stochastic}}.

\emph{Remark:}
In recent literature, \cref{hamsde} has been associated with normally-distributed estimates of the gradient \citep{chen2014stochastic,mandt2017stochastic}.
%% It is our position however this connection is not well-justified, 
In our view, however, this connection is not well-justified, 
as for the Brownian motion term of \cref{hamsde} we have that $d\mathbf{w}(t) = \sqrt{dt} \mathcal{N}(\zerovect,\eye)$, while a (Gaussian) stochastic gradient term in the continuous limit becomes $dt \mathcal{N}(\zerovect,\eye)$.
A more detailed exposition can be found in \cref{sec:sgdvssde}.
Our alternative treatment for the study of the effect of mini-batches follows in \cref{sec:minibatches}.

%Notice that we do \emph{not} associate the friction with the noise of a stochastic gradient component 

%\maurizio{With this last sentence it seems that we want emphasize this wrt to previous work right? If so we should dedicate a short discussion on this and refs.}

\subsection{Ergodic errors of \sdes}

Except from a handful of cases, it is not possible to draw exact sample paths from arbitrary \sdes.
We consider a generic numerical integrator $\psi$, with step size $\eta$, whose purpose is to simulate the
stochastic evolution of the \sde of interest. 
Formally, we look for a stochastic mapping, that from a given initial condition $\zvect_0$, generates a new random variable $\zvect_{1}=\psi(\zvect_{0};\eta)$ by faithfully simulating the true continuous time stochastic dynamics $\zvect(\eta)|_{\zvect(0)=\zvect_{0}}$. 
%We refer to integrators applied to an \ode as deterministic.
%
Several quantitative metrics measuring the degree of accuracy of the simulation are available \citep{kloeden2013numerical}. The simplest is the strong error, which is the expected difference between true paths and simulated ones. Relaxing to the expected difference between functions of paths corresponds to quantifying the weak error:
\begin{definition}
\citep{debussche2012weak} Consider \cref{hamsde} with initial condition $\zvect_{0}$. The numerical integrator $\psi$ has weak order of convergence $p$ if
$
   | \mathbb{E}\left[\phi(\psi(\zvect_{0};\eta))\right]-\mathbb{E}\left[\phi(\zvect(\eta))\right]|=\mathcal{O}(\eta^{p+1})
$,
where $\zvect(\eta)$ is the value taken by \cref{hamsde} after a time $\eta$, the chosen step size.
\end{definition}

The transformation of the functions $\phi(\zvect)$ of the true stochastic dynamics $\zvect(t)$, starting from initial conditions $\zvect_0$ is described by the Kolmogorov differential equation $\mathbb{E}\left[\phi(\zvect(t))\right]=\exp{(t\mathcal{L})}\phi(\zvect)|_{\zvect=\zvect_0}$, that throughout the paper we indicate with abuse of notation as $\mathbb{E}\left[\phi(\zvect(t))\right]=\exp{(t\mathcal{L})}\phi(\zvect_0)$. The operator $\mathcal{U}$, defined as
$
    \mathcal{U}\phi(\zvect_{0})=\mathbb{E}\left[\phi(\psi(\zvect_{0};\eta))\right]
$, represents a $p_{\text{th}}$ order weak integrator if %\footnote{{\color{red}The statement $\mathcal{P}=\mathcal{O}(\eta^r)$ should be intended as $|\mathcal{P}\phi(\zvect)|\leq \eta^r C_l(1+||\zvect||^{k_l})$, \cite{abdulle2014high,abdulle2015long}. Whenever unambiguous we abuse the notation to simplify the exposition.}}%\footnote{{\color{red}Notice that the introduced notation is not formally precise. A precise definition would require the introduction of the concept of operator norm, the domain of the functions in which the operators can be applied. The statement $\mathcal{P}=\mathcal{O}(\eta^r)$ should be intended as $|\mathcal{P}\phi(\zvect)|\leq \eta^r C_l(1+||\zvect||^{k_l})$, \cite{abdulle2014high,abdulle2015long}. Whenever unambiguous we use the shorter notation to simplify the presentation.}}
$\mathcal{U}=\exp{(\eta\mathcal{L})}+\mathcal{O}(\eta^{p+1})$.

%since \begin{flalign*}
%   & | \mathbb{E}\left[\phi(\psi(\zvect_{0};\eta))\right]-\mathbb{E}\left[\phi(\zvect(\eta))\right]|=| \mathcal{U}\phi(\zvect_{0})-\exp{(t\mathcal{L})}\phi(\zvect_{0})|=\\&|(\exp{(t\mathcal{L})}+\mathcal{O}(\eta^{p+1}))\phi(\zvect_{0})-\exp{(t\mathcal{L})}\phi(\zvect_{0})|=\mathcal{O}(\eta^{p+1}).\end{flalign*}

Since we are only interested in samples from $p(\thetavect|D)$, it is sufficient to consider the even weaker \emph{ergodic average error}, defined hereafter. 
The numerical integrator iteratively induces a stochastic process $\zvect_{i}=\psi({\zvect_{i-1}};\eta),i=1,\dots N$. The ergodic average of a given function $\phi(\cdot)$
%, i.e.\ $ \lim\limits_{N\rightarrow \infty} \frac{1}{N}\sum\limits_{i=1}^{N}\phi(\zvect_{i})$, 
converges to the integral $\int \phi(\zvect)\rho^{\psi}_{ss}(\zvect) \dzvect$, where $\rho^{\psi}_{ss}$ is the stationary distribution of the stochastic process induced by the numerical integrator $\psi$. The ergodic average error is the difference between the ergodic average of the numerical integrator and the true average obtained with the stationary distribution of \cref{hamsde}:
\begin{equation}\label{ergerr}
    e(\psi,\phi)=\int \phi(\zvect)\rho^{\psi}_{ss}(\zvect) \dzvect-\int \phi(\zvect)\rho_{ss}(\zvect) \dzvect.
\end{equation}

We stress that, although outside the scope of this work, the results presented here can be extended to non-asymptotic (in time) settings. Indeed the error with respect to a finite time empirical average can be expressed as $
\frac{1}{N}\sum\limits_{n=0}^{N-1}\phi(\zvect_n)-\int \phi(\zvect)\rho_{ss}(\zvect) \dzvect    
$. By triangle inequality its absolute value can be bounded by $|
\frac{1}{N}\sum\limits_{n=0}^{N-1}\phi(\zvect_n)-\int \phi(\zvect)\rho^{\psi}_{ss}(\zvect) \dzvect| + |e(\psi,\phi)|   
|$. The first term can be bounded in expectation by considering the convergence speed (as a function of time) to the stationary distribution of the numerical integration scheme, either by classical arguments or attempting a weak backward error analysis characterization \citep{debussche2012weak}. In this work we consider only the quantity in \cref{ergerr}.

Our goal is to characterize the rate of convergence to zero for $e(\psi,\phi)$ as a function of the step size $\eta$, the most important free parameter. 
A sufficient condition for an integrator to be of a given ergodic order $p$, i.e. $e(\psi,\phi)=\mathcal{O}(\eta^p)$, is to have weak order $p$ \citep{abdulle2014high}.
See \cref{sec:abdulletheorem} for a more detailed exposition.

%For the purpose of understanding the contributions of this paper, the \cref{suffprop} in the appendix is sufficient, which leverages the results of \citet{abdulle2014high}. 

\subsection{Numerical integrators}

%As a direct consequence of \cref{prop:weak2ergodic}, this will guarantee that the error from the true stationary distribution will vanish at a rate $\mathcal{O}(\eta^2)$.

%Some of these integrators are outlined as position-momentum updates in \cref{tab:integrators}, where we consider $\mathbf{M}=\eye$ to reduce clutter; schemes with generic $\mathbf{M}$ can be found in \cref{sec:numint}.

\begin{comment}
\begin{table*}
    %\renewcommand{\arraystretch}{1.0}
    \centering
    \caption{Update rules for the integrators that we explore in this work. Notice that \leapfrog is provably of order $\mathcal{O}(\eta)$; however, the experimental results suggest that it is comparable with second-order methods.}
    \label{tab:integrators}
    \tiny
    \input{tables/summary_integrators_short}
\end{table*}
\end{comment}

We explore the effect of a number of \sde integrators of weak order two and three; these schemes are outlined in detail in \cref{sec:numint}.
Hamiltonian systems evolve on manifolds with peculiar geometrical properties whose study is the subject of \emph{symplectic geometry}. In \cref{sec:proofquasi} we present the class of quasi-symplectic integrators \citep{milstein2003quasi}. For the purpose of understanding the main text it is sufficient to know that such integrators are empirically known to outperform their non-symplectic counterpart. 

%For completeness we then explore both classes of integrators.
%For stochastic systems with additive noise, the direct extension of a two-stage Runge-Kutta (\rktwo) scheme is provably of second order \citep{kloeden2013numerical}; however, it is not quasi-symplectic.

The \leapfrog scheme is probably the simplest form of quasi-symplectic integrator,
which involves updating the position and momentum at interleaved time steps.
Importantly, while it has theoretically order-one of convergence both in weak and ergodic sense, we observe
empirically that its performance is very close to other quasi-symplectic order-two schemes. 
A similar phenomenon has been observed by \citet{milstein2003quasi}.
%Although what we present is not a proof, in \cref{sec:leapmir} we discuss the Taylor expansion of the operator induced by the scheme, providing insights about its performance. We leave further explorations for future work.

In the experimental section, we also examine the symmetric splitting integrator proposed in \cite{chen2015convergence} (\symmetric), which is also compatible with our \sdehmc scheme.
Finally we consider a \lietrotter splitting scheme \citep{abdulle2015long} that is of weak order-two, while it is also quasi-symplectic.
In \cref{sec:hmcvssdehmc} this scheme will be used to draw a theoretical connection between \hmc \citep{neal2011mcmc} and the proposed \sde framework in \cref{hamsde}.
%Notice that many possible different Lie Trotter schemes exist with different order of convergence \cite{abdulle2015long}.

In order to demonstrate the convergence bottleneck more clearly, we also include an integrator of third order in our comparisons.
We employ the 3-stage quasi-symplectic integrator of \citet{milstein2003quasi}.
This scheme might not be always practical, as it involves three stages and one computation of the Hessian.
However, it is sufficient to demonstrate the convergence bottleneck effect due to mini-batches.

\section{MINIBATCHES AS OPERATOR SPLITTING}
\label{sec:minibatches}

%   \begin{mdframed}[backgroundcolor=green!20]
%   \textbf{Contribution}: the effect of mini-batching can be understood through the purely geometrical perspective of differential operator splitting, without the need for any central limit theorem assumption on the statistics of the \sg noise. Mini-batching introduces an extra error of order 2 in the convergence rate.\end{mdframed}

%In the classical \sde literature, the numerical integration error is usually the main focus when studying practical schemes. 

%% TODO: move this to the intro?
%Regarding posterior sampling for machine learning problems, a second source of approximation is introduced by the use of subsets of the full dataset $\dataset$ for computing the gradient of the potential. This practice is known as minibatching and is introduced to design scalable sampling methods.

Traditionally, the effect of mini-batching has been modeled as a source of additional independent Gaussian noise \citep{chen2014stochastic} at every step of the simulated dynamics. 
We challenge this common modeling assumption by considering the geometrical perspective of differential operator splitting. 
We build the link between the true generator $\mathcal{L}$ and numerical integration performed using mini-batch subsets of the full potential.
Our analysis does not make any assumption about Gaussianity of \sg noise.
Our approach is rooted in the geometrical view of splitting schemes for high dimensional Hamiltonian systems \citep{childs2019theory,suzuki1977convergence,hatano2005finding,childs2019nearly,low2019well}. 
This allows us to derive \cref{genericintegratorerdogic}, which presents a result in terms of convergence rate to the desired posterior distribution.
A purely geometrical approach has not received much attention in the context of data subsampling for Bayesian inference, with the notable exceptions of \citet{pmlr-v37-betancourt15,shahbaba2014split} that explored related ideas for the case of \hmc only.

Without loss of generality, suppose that the dataset $\dataset$ is split into two mini-batches $\dataset_1,\dataset_2$. Following \cref{eq:infgen} we can define the infinitesimal generators $\mathcal{L}_1,\mathcal{L}_2$, for which we have $\mathcal{L}=\mathcal{L}_1+\mathcal{L}_2$.
Intuitively, given an operator in exponential form $\exp(\eta(\mathcal{L}_1+\mathcal{L}_2))$, we would like to determine under which conditions the following holds
\begin{equation}\label{eqappx}
   \exp(\eta(\mathcal{L}_1+\mathcal{L}_2))\simeq\exp(\eta\mathcal{L}_1)\exp(\eta\mathcal{L}_2),
\end{equation}
and to quantify the discrepancy error in a rigorous way. In the general case we consider splitting $\mathcal{L}$ into $K$ mini-batches of the form $\mathcal{L}=\sum\limits_{i=1}^K\mathcal{L}_i$ where
\begin{flalign}&\mathcal{L}_i=-\left(\nabla^\top_{\thetavect}U_i(\thetavect)\right)\nabla_{\rvect}+K^{-1}\left(\left(\mathbf{M}^{-1}\rvect\right)^\top\right)\nabla_{\thetavect}\nonumber\\&-K^{-1}C\left(\left(\mathbf{M}^{-1}\mathbf{r}\right)^\top\right)\nabla_{\mathbf{r}}+K^{-1}C\nabla_{\mathbf{r}}^\top\nabla_{\mathbf{r}},\end{flalign}
with $U_i(\thetavect)$ the potential computed using only the $i_{\text{th}}$ mini-batch.
The theorems presented in this section clarify how concatenations of the form $\prod_i \exp(\eta K\mathcal{L}_i)$ induce errors and clarify their relevance for the considered problem.

The following theorem characterizes the order of convergence of the randomized splitting scheme.
\begin{theorem}\label{randomorderop:specialcase}
Split the full dataset into $K$ mini-batches and consider numerical integrators $\psi_i$ of order $p$ obtained using the mini-batches, i.e. $\mathbb{E}\left[\phi_i(\psi(\zvect_{0};\eta))\right]=\exp(\eta K\mathcal{L}_i)\phi(\zvect_{0})+\mathcal{O}(\eta^{p+1})$.
 Extract uniformly $\boldsymbol{\pi}=\begin{bmatrix}\pi_1,\dots,\pi_K\end{bmatrix}\in\mathbb{P}$, the set of all the possible permutations of the indeces $\{1,K\}$. The scheme in which initial condition has stochastic evolution through the chain of integrators with order \begin{equation}
    \mathbf{z}_{fin}=\psi_{\pi_K}(\psi_{\pi_{K-1}}(\dots\psi_{\pi_1}(\mathbf{z}_0))),
\end{equation}
transforms the functions $\phi$ with an operator $\mathcal{U}$ that has the following expression
%\begin{equation}
%   \mathcal{U}=\exp(\eta\mathcal{E}_{tot}(\eta))+\mathcal{O}(K\eta^{p+1})+\mathcal{O}(K\eta^{3})
%\end{equation}
\begin{equation}
   \mathcal{U}=\exp(\eta K\mathcal{L})+\mathcal{O}(K\eta^{p+1})+\mathcal{O}(K\eta^{3})%+\mathcal{O}(K\eta^{\min{(p+1,3)}})
\end{equation}
\end{theorem}
The proof can be found in \cref{secproof:randomorderop}. 
Particularly relevant to our discussion are works that explore randomized splitting schemes for Hamiltonian simulation, but not in a sampling context \citep{childs2019faster,zhang2012randomized}.
%To explore related ideas, see also \cite{childs2019faster,zhang2012randomized}. 
%Intuitively, this means that if we pick a random permutation of the batches and apply the sequence of numerical integrators accordingly, we induce an operator that transforms functions $\phi$ with an error that has order $\mathcal{O}(K\eta^{p+1})+\mathcal{O}(K\eta^{3})$. 

%$\mathcal{O}(K\eta^{\min{(p+1,3)}})$. \Cref{randomorderop:specialcase} allows to state \cref{genericintegratorerdogic}, which is a central result of this discussion.
\begin{theorem}\label{genericintegratorerdogic}
Consider the settings described by \cref{randomorderop:specialcase}. Repeatedly apply the numerical integration scheme. Then the ergodic error has expansion 
\begin{equation}
    e(\psi,\phi)=\mathcal{O}( \eta^{\min{(p,2)}})%\mathcal{O}(\eta^{p})+\mathcal{O}(\eta^{2}).%
\end{equation}
\end{theorem}
See \cref{secproof:genericintegratorerdogic} for the proof.
This theorem tells us that the effect of mini-batches is an extra error of order two in the convergence rate. This
extra error is not due to an equivalent noise injection into the \sde dynamics, and attempting to counterbalance it, as suggested by the literature \citep{chen2014stochastic}, is irrelevant as we will show in \cref{sec:experiments}. Importantly, this error can become the bottleneck whenever $p>2$.

The experimental validation carried out in this work, of which \cref{fig:minibottle} is a prototypical example, confirms the presence of this bottleneck. 
We include a toy example taken from \citet{vollmer2016exploration}, where an analytical solution is available.
The complete specification can be found in \cref{sec:toyexapp}.
The results, reported in \cref{fig:toyminibach}, are obtained by comparing an order-2 integrator against an analytical solution, when considering or not mini-batches. 
Qualitatively, the presence of a bottleneck is clear: when considering mini-batches, the stationary distribution is not the desired one, even with a perfect integrator.

\begin{figure}
    \centering
    \includegraphics[width=0.5\textwidth]{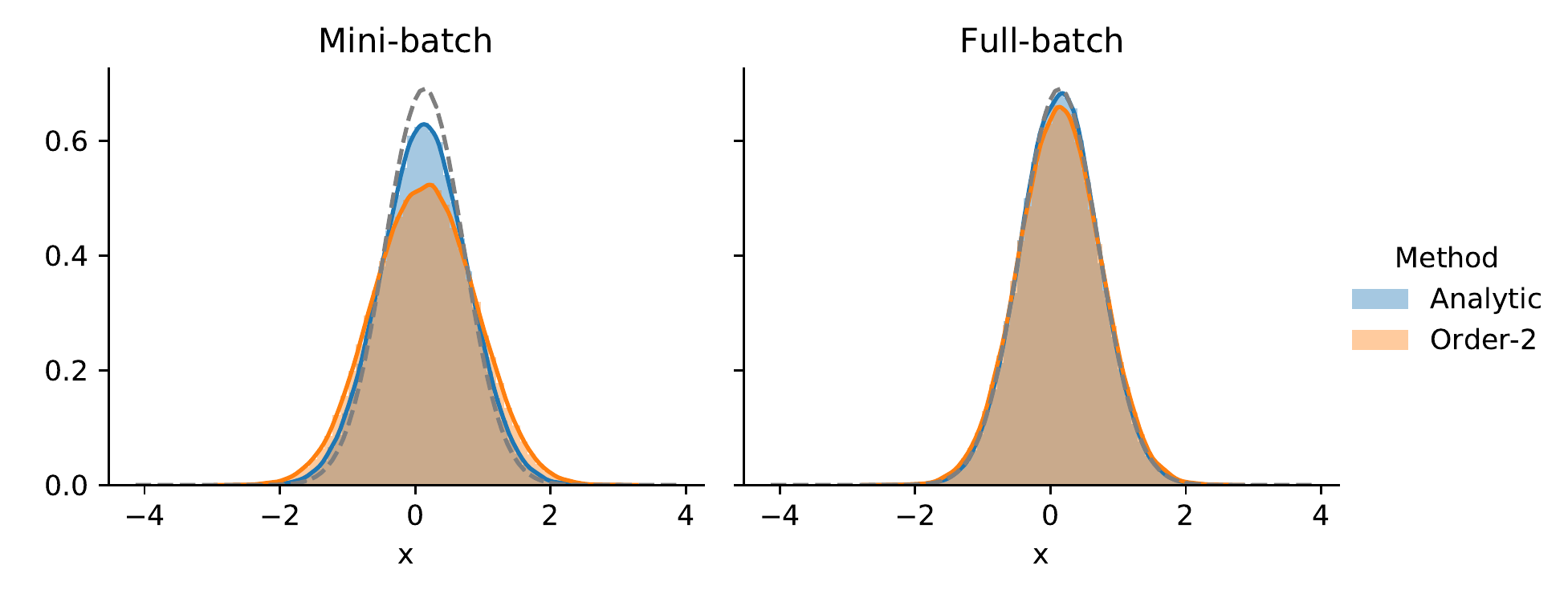}
    \caption{Histograms of stationary distributions. The grey dotted line denotes the true posterior density. Irrespectively of order, mini-batching prevents from convergence to true posterior.}
    \label{fig:toyminibach}
\end{figure}

\begin{comment}
We consider as prior distribution for the parameter $    p(\theta)=\mathcal{N}(\theta;0,\sigma_\theta^2),
$
while the likelihood for any given observation is
$
    p(x|\theta)=\mathcal{N}(x;\theta,\sigma_x^2),
$
where $\sigma_\theta,\sigma_x$ are arbitrary positive constants. Considering two datapoints, i.e. $N=2$, $\dataset=\{x_1,x_2\}$, the posterior distribution is shown to be of the form
$
    p(\theta|\dataset)=\mathcal{N}\left(\theta;\frac{x_1+x_2}{v},\sigma_l^2\right),
$
where $v=\frac{\sigma_x^2}{\sigma_\theta^2}+2$ and $\sigma_l^2=\left(\frac{1}{\sigma_\theta^2}+\frac{2}{\sigma_x^2}\right)^{-1}$. The potential mini-batch potentials are equal to $U_i(\theta)=\frac{1}{2\sigma_l^2}\left(\theta-\frac{2x_i}{v}\right)^2.
$

As mentioned, it is possible to build analytical integrators, that provide perfect simulation of paths. These integrators, that correspond to arbitrarily high order numerical integrators, are used for a numerical simulation. 

% {\color{red} Notice that the full batch is supposed to achieve exactly the stationary distribution, the epoch scramble (what we have in the main) $\mathcal{O}(\eta^2)$ and sampling at each step $\mathcal{O}(\eta^1)$. The proof for this last claim is not included here}.

% \dimitris{what is epoch scramble and sample scramble? These are not introduced anywhere!}

\end{comment}

%\subsection*{On the bottleneck and previous literature}
Our main result, the presence of a bottleneck, is in direct disagreement with previous results presented in \citet{chen2015convergence}, where the mini-batching does not affect the ergodic average error. Instead, our theoretical and empirical discussions clearly show that the introduction of mini-batches impedes convergence to the desired posterior. 
%A detailed analysis for the possible causes of disagreement between our results and \cite{chen2015convergence} are presented in \cref{sec:toyexapp}. We include a deeper analysis of the previous toy example, and expand on the original proofs of \cite{chen2015convergence}, providing to the reader a starting point for understanding the difference between the results. {\color{red}GF: I don't like this last sentence} 

%The results presented in this section allows to state that for the considered class of Hamiltonian \sdes, it is possible to split the purely deterministic part of the generator $\mathcal{L}$, solve it numerically, then solve analytically the friction and noise component and still achieve the same convergence rate as in \cref{genericintegratorerdogic}. From a practical point of view, as highlighted also in \cref{sec:experiments}, this does not provide any advantage. Importantly however, the considered result helps us in understanding the link, using a common language, between \sdehmc and \hmc. This connection is carefully explored in \cref{sec:hmcvssdehmc}.

%\subsection*{Details on the constants in the convergence results}
\noindent \textbf{Details on the constants in the convergence results}.
Some clarifications are in order. First, the constants in the $\mathcal{O}$ notation of weak and ergodic errors, could be refined considering the geometry of the potentials and the norms inequality of differential operators. We leave such a possibility for future works and refer the interested reader to \citet{childs2019theory,childs2019faster} and in particular also to \citet{zhang2012randomized}, that proposes randomized schemes strictly related to the one discussed here. Second, we stress that when performing sampling for Bayesian inference problems, the full potential is divided into mini-batches, and each subset rescaled by the constant $K$.
%This corresponds to modifying the infinitesimal generator as follows: $ \mathcal{L}\rightarrow K\mathcal{L}= \sum_i (K\mathcal{L}_i)$ and then splitting it. 
%In the classical \sde or \ode simulation literature this recaling is (usually) not present.
This rescaling allows one to ``cover'' the same amount of distance per steps independently from $K$, given that with a pure splitting one would need $K$ steps to simulate $\exp(\eta\mathcal{L})$.
%Importantly, the desired stationary distribution is the same since $\mathcal{L}^\dagger\rho_{ss}=0\iff (K\mathcal{L})^\dagger\rho_{ss}=0$.
%Equivalently one could interpret the procedure as corresponding the simulation of the splitted system with step size $\eta\rightarrow K\eta$. 

For the case of \hmc, this connection is acknowledged in \citet{pmlr-v37-betancourt15}, drawing an equivalence with the work of \citet{shahbaba2014split}.%, with new step size $K\eta$. 
%Inspection of their results, with the implicit assumption $K=\mathcal{O}(\frac{1}{\eta})$,{\color{red} matches our results}.

\section{A LIE-TROTTER INTEGRATION SCHEME AND CONNECTIONS WITH HMC}
\label{sec:hmcvssdehmc}
%For future reference we define the pure Hamiltonian part corresponding to a minibatch as $\mathcal{H}_i=-\left(\nabla^\top_{\thetavect}U_i(\thetavect)\right)\nabla_{\rvect}+K^{-1}\left(\left(\mathbf{M}^{-1}\rvect\right)^\top\right)\nabla_{\thetavect}$.

%   \begin{mdframed}[backgroundcolor=green!20]
%   \textbf{Contribution}: We show that is possible to interpret \hmc as an integration scheme for the considered class of \sdes. We derive converge rates for schemes with and without mini-batches and study the ergodic error and the stability of the schemes as a function of $N_l$, the \hmc loop size.\end{mdframed}

We discuss a \sde integration scheme that relies on a \lietrotter splitting of the infinitesimal generator $\mathcal{L}$.
That allows us to draw a connection between Hamiltonian \sdes and the original \hmc family of algorithms, by showing that the latter can be interpreted as an integration scheme of the \sde dynamics.

The infinitesimal generator in \cref{eq:infgen} can be expressed as the sum $\mathcal{L} = \mathcal{H} + \mathcal{D}$, where: 
\begin{flalign*}
&\mathcal{H}=-\left(\nabla^\top_{\thetavect}U(\thetavect)\right)\nabla_{\mathbf{r}}+\left(\left(\mathbf{M}^{-1}\mathbf{r}\right)^\top\right)\nabla_{\thetavect}, \\
&\mathcal{D}=-C\left(\left(\mathbf{M}^{-1}\mathbf{r}\right)^\top\right)\nabla_{\mathbf{r}}+C\nabla_{\mathbf{r}}^\top\nabla_{\mathbf{r}}.
\end{flalign*}
In its general form \citep{abdulle2015long}, the \lietrotter scheme is derived as an application of the Baker–Campbell–Hausdorff formula \citep{dynkin1947calculation}, where the dynamics are approximated as follows:
\begin{equation*}
    \exp(\eta \mathcal{L}) = \exp(\nicefrac{\eta}{2} \mathcal{D}) \exp(\eta \mathcal{H}) \exp(\nicefrac{\eta}{2} \mathcal{D}) + \mathcal{O}(\eta^{3})
\end{equation*}
In practice, the scheme consists of alternating steps that solve the $\mathcal{H}$ and $\mathcal{D}$ parts of $\mathcal{L}$.
The design space is the one of deterministic integrators for the Hamiltonian part $\mathcal{H}$, as the term $\mathcal{D}$ can be solved exactly. Having chosen a deterministic integrator $\psi$, a single step of the \sde simulation has the following form:
\begin{flalign}
    \label{eq:lietr}
    & \zvect^* = \begin{bmatrix}{\rvect^*}^\top,{\thetavect^*}^\top \end{bmatrix}^\top=\psi(\zvect_0), \\
    & \zvect_1 =
    \begin{bmatrix}{\exp(-\eta C){\rvect^*}^\top+\sqrt{1-\exp(-2\eta C)}\mathbf{w}^\top,\thetavect^*}^\top
    \end{bmatrix}
\end{flalign}
It can be shown \citep{abdulle2015long} that whenever the numerical integrator $\psi$ for the Hamiltonian part $\mathcal{H}$ has order $p$, i.e. $\mathcal{U}\phi=\exp(\eta\mathcal{H})+\mathcal{O}(\eta^{p+1})$, the \lietrotter scheme has convergence order $\mathcal{O}(\eta^p)$. 

In the experiments of \cref{sec:experiments}, we consider $\psi$ to be the deterministic version of \leapfrog,
% which of a popular choice for deterministic Hamiltonian systems, and 
for which we have $p=2$.
Therefore, our practical application of \lietrotter is of order $\mathcal{O}(\eta^2)$.

\paragraph{\hmc with partial momentum refreshment}
%The \hmc algorithm \cite{neal2011mcmc} is composed of two nested loops in which a pure deterministic Hamiltonian dynamic is simulated and then momentum is resampled from a Gaussian distribution. 
We examine a generalized \hmc variant featuring partial momentum refreshment 
%(\cite{neal2011mcmc}, Eq.~(5.19), \cite{horowitz1991generalized}). 
\citep{horowitz1991generalized,neal2011mcmc}, which consists of two (repeated) steps.
First, we have a numerical approximation of a Hamiltonian system by means of a (deterministic) integrator $\psi$:
\begin{equation}
    \label{eq:hmc1}
    \zvect^*=\begin{bmatrix}{\rvect^*}^\top,{\thetavect^*}^\top
    \end{bmatrix}^\top=\underbrace{\psi(\dots\psi(\psi(}_{N_l \text{times}}\zvect_0))),
\end{equation}
where $N_l$ denotes the number of integration steps and it is assumed to be finite. %, resulting in integration length equal to $\eta N_l$.
Then, we have a a partial momentum update as follows:
\begin{equation}
    \label{eq:hmc2}
    \zvect_1%=\begin{bmatrix}\boldsymbol{r_1}^\top,\boldsymbol{\theta_1}^\top\end{bmatrix}
    =\begin{bmatrix}{\alpha{\rvect^*}^\top+\sqrt{1-\alpha^2}\mathbf{w}^\top,\thetavect^*}^\top
    \end{bmatrix},
\end{equation}
where $\mathbf{w}\sim\mathcal{N}(\mathbf{0},\mathbf{I})$ and $\alpha>0$. % applications of the integrator $\psi$, resulting in integration length equal to $\eta N_l$.

%The most widely adopted variation of \hmc is the one with full momentum resampling, which corresponds to $\alpha=0$, and the choice of deterministic \leapfrog for $\psi$.
%However, we next show that for a particular choice for $\alpha>0$ we can recover a discretization of the dynamics of \sdehmc by means of a \lietrotter splitting scheme.

The connection between \hmc and the \lietrotter is drawn by noticing that mechanically, the two schemes simulate \sdes (and thus transform functions) in a similar fashion. \lietrotter transforms functions after a step as $\mathbb{E}\left[\phi(\zvect_1)\right]=\mathcal{U}\exp(\eta \mathcal{D})$, while \hmc as $\mathbb{E}\left[\phi(\zvect_1)\right]=\left(\mathcal{U}\right)^{N_l}\exp(\eta N_l \mathcal{D})$. The claim is true when we consider the $\alpha$ of \hmc to be $\alpha=\exp(-\eta N_l C)$ (the full momentum resampling corresponds to $C\rightarrow\infty$).

By exploiting this connection to classical \sde integrators, we can thus study the ergodic error of \hmc scheme with and without mini-batches.
The ergodic error for the two cases is $\mathcal{O}(\eta^p)$ and $\mathcal{O}(\eta^{\min(p,2)})$ respectively, showing again the bottleneck introduced by mini-batches.
See \cref{sec4:appendix} for the derivation details.
%With mini-batches, we consider for simplicity $N_l=KT$, where $T$ is an integer. 

In both cases, the orders of convergence are \textit{independent} on $N_l$. 
This is reflected in the experimental results in \cref{sec:experiments}. %, where for small enough step sizes (the region of interest in terms of $\mathcal{O}(\cdot)$ notation), the different curves tend to overlap. 
We note that, in practice, the approximation quality degrades by increasing $N_l$ for higher learning rates.

\section{EXPERIMENTS}
\label{sec:experiments}
% \begin{mdframed}[backgroundcolor=green!20] 
%    \textbf{Contribution}: we are the first that perform this kind of exploration...
% \end{mdframed}

\begin{comment}
{\color{red}\paragraph{Comparing convergence rates}
  Experimental comparisons of different schemes based on their theoretical order of convergence, or their constants, do not guarantee anything about the relative performance at finite step sizes.
  Whenever one scheme dominates others in terms of order of convergence, we can only say that it exists an $\eta_0$ such that for all $\eta<\eta_0$ the first scheme has better performance than the other.
  Such unknown $\eta_0$ depends, among the other things, on the geometry of the potentials and the particular properties of the integrators.
  All the results we collect and present in this Section are however consistent with the convergence order predicted by the theory, suggesting that the various $\eta_0$ of the orders of convergence are sufficiently high from a practical point of view.}
\end{comment}

\begin{figure*}[h]
    \centering
    \includegraphics[width=0.49\textwidth]{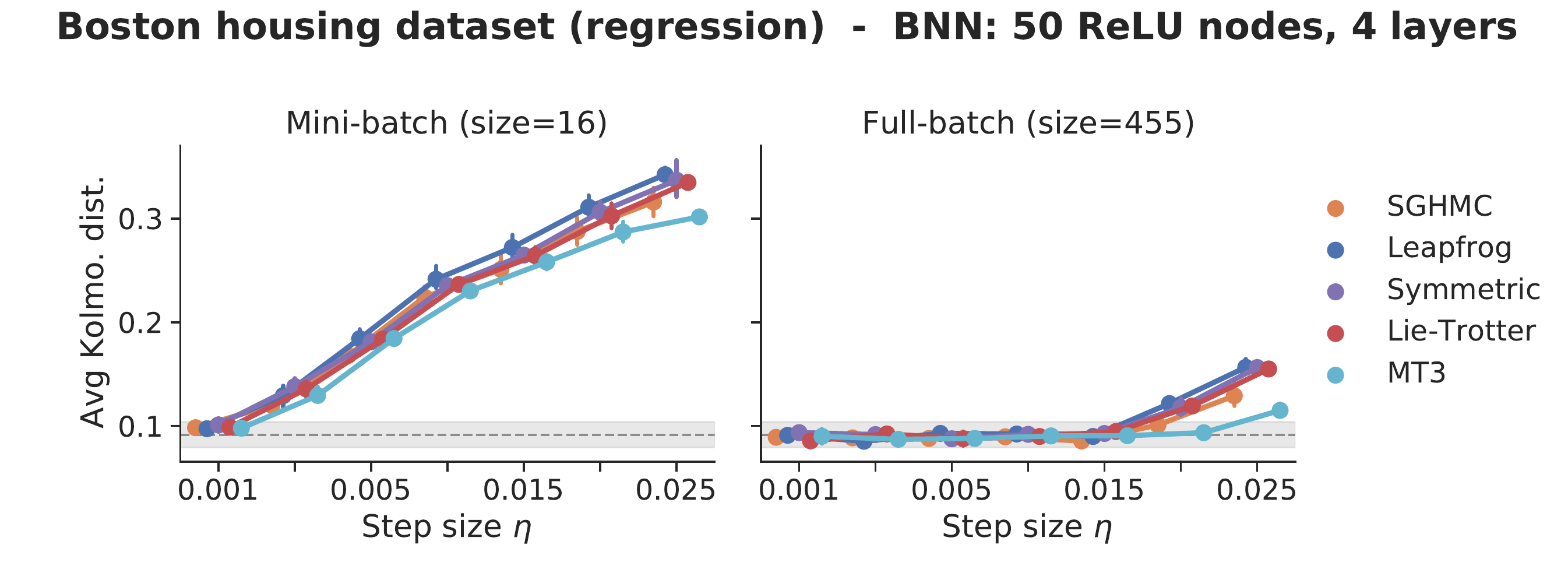}
    \includegraphics[width=0.49\textwidth]{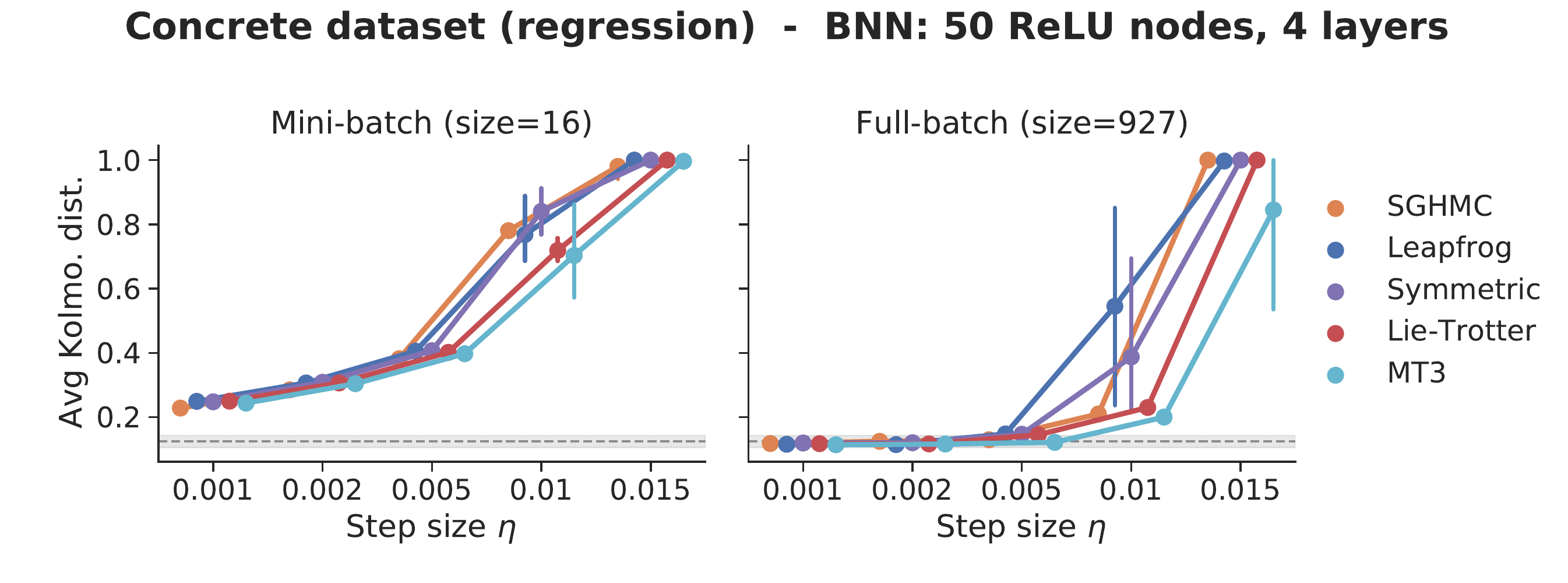}
    \includegraphics[width=0.49\textwidth]{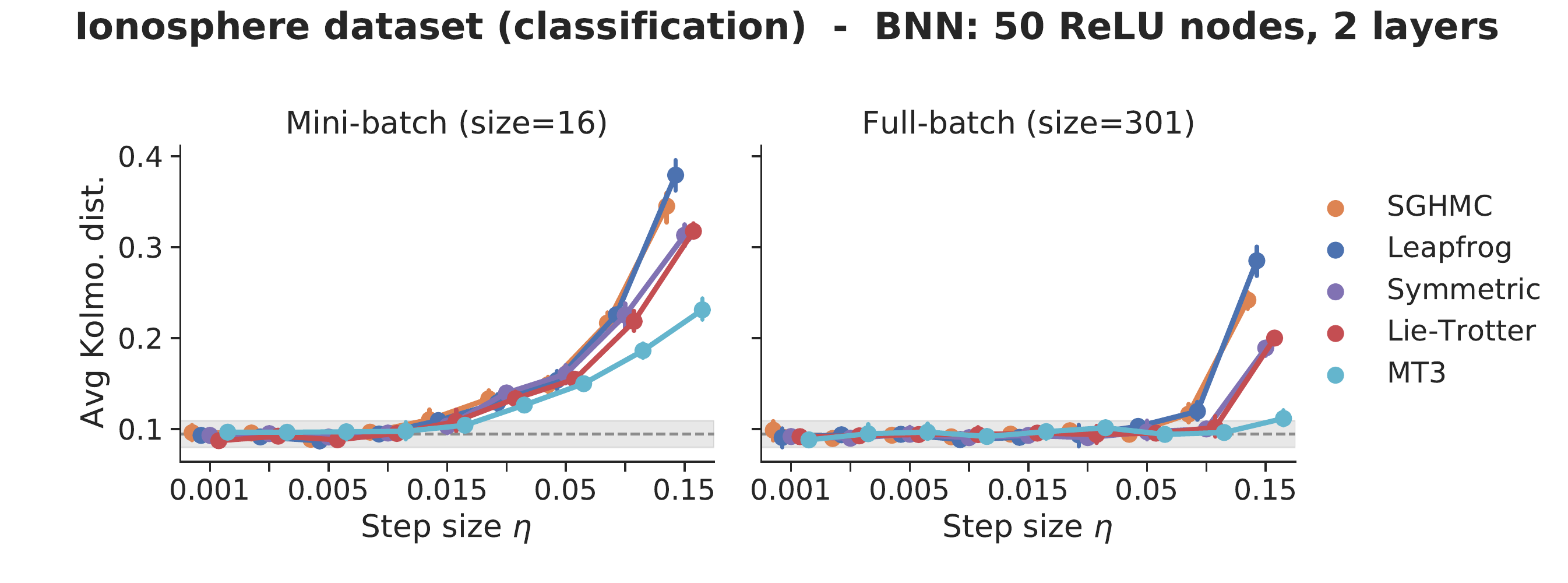}
    \includegraphics[width=0.49\textwidth]{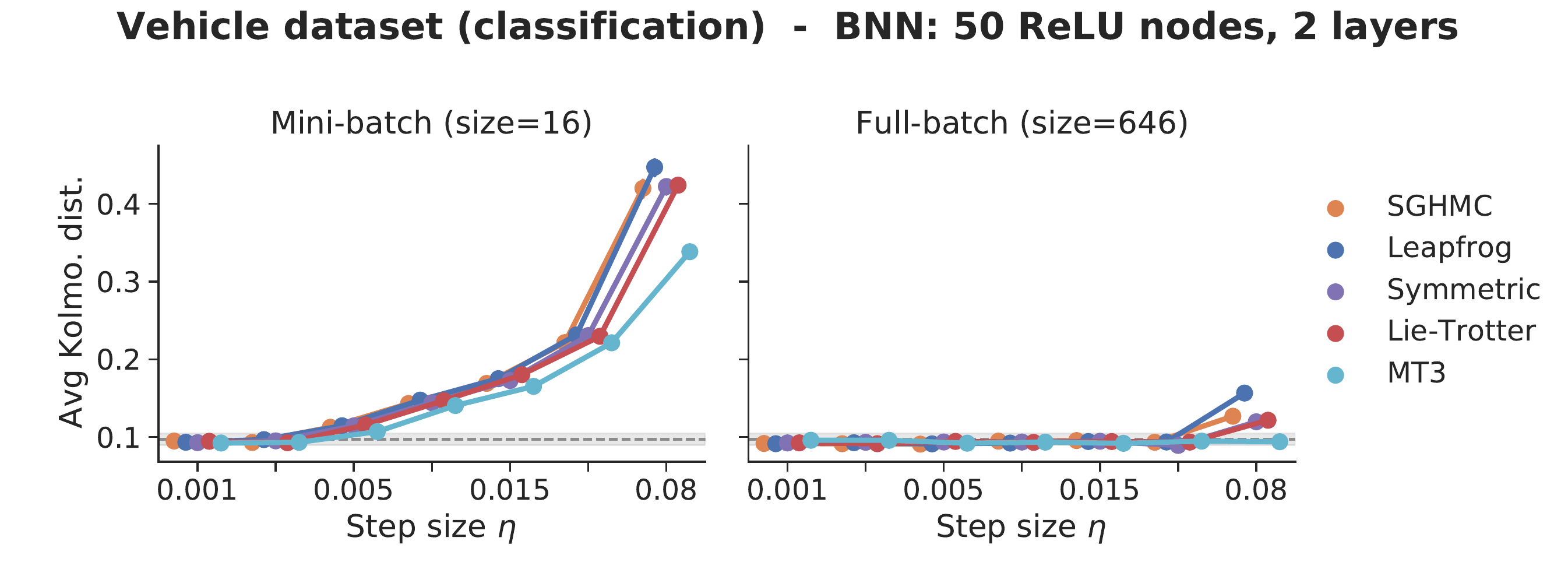}
    \caption{Exploration of step size and batch size for different Hamiltonian-based methods; the grey dotted line denotes the self-distance for the distribution of the oracle.}
    \label{fig:methods}
\end{figure*}

\textbf{Comparison framework.}
The main objective of the experiments is to investigate convergence to the true posterior distribution for a wide range of \bnn models and datasets.
Metrics that reflect regression and classification accuracy are not of main interest; nevertheless, some of these are reported in \cref{ssec:experiment_setup_appendix}.
%, where they are shown to be on par with relevant works in the literature.
Instead, we turn our attention to the quality of the predictive distribution.

We consider the true predictive posterior as the ground truth,
%this will be calculated analytically for the linear examples. 
which is approximated by a careful application of \sghmc \citep{springenberg2016bayesian} featuring a very small step size and full-batch gradient calculation; this is referred to as the \emph{oracle}.
Any comparison between high-dimensional empirical distributions gives rise to significant challenges. Therefore, we resort to comparing one-dimensional predictive distributions by means of the % Kolmogorov-Smirnov statistic, also known as
Kolmogorov distance.
For a given test dataset, we explore the average Kolmogorov distance from the true posterior predictive distributions for different methods, step sizes and mini-batch sizes.
These should be compared with average Kolmogorov \emph{self-distance} for the oracle, which is marked as grey dotted lines in the figures that follow\footnote{This emulates the famous Kolmogorov-Smirnov test, but with no Gaussianity assumptions}.
In all cases we compare empirical distributions of 200 samples; the example of \cref{fig:minibottle} is an exception, where we compare distributions of 2000 samples.
See \cref{ssec:framework_appendix} for a complete account.

\begin{figure*}[h]
    \centering
    \includegraphics[width=0.49\textwidth]{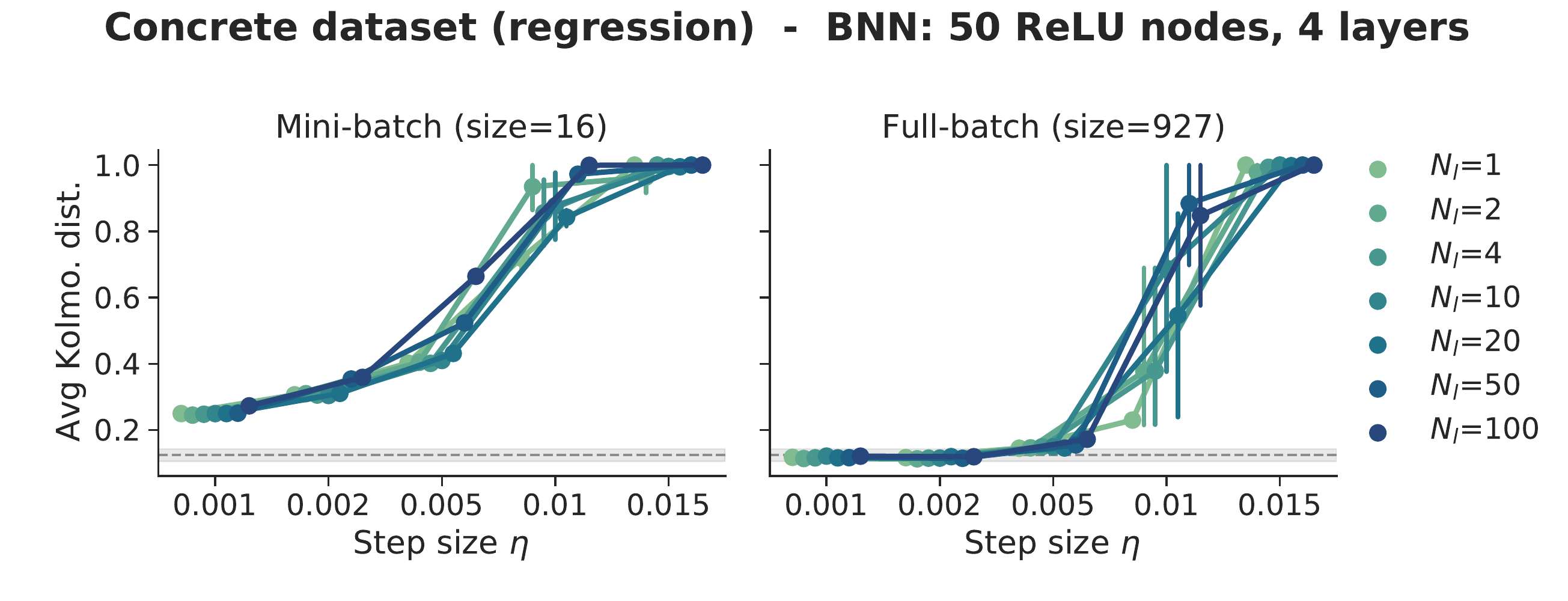}
    \includegraphics[width=0.49\textwidth]{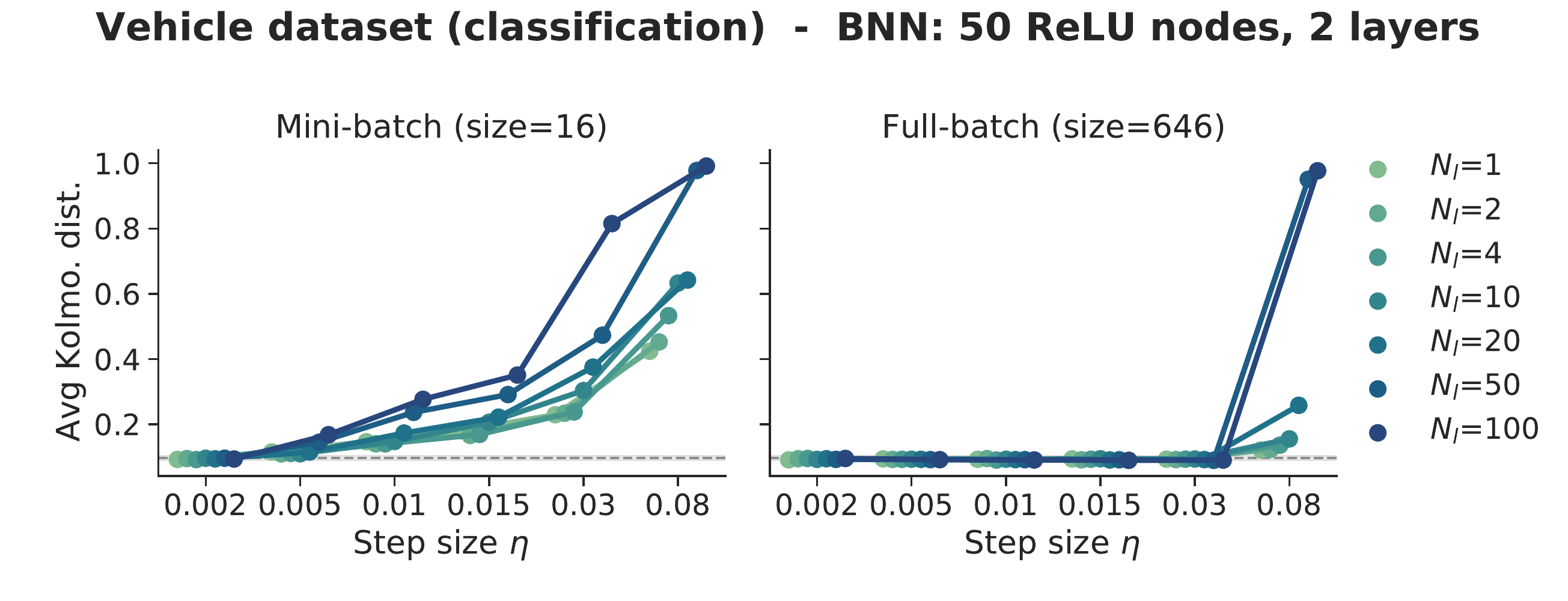}
    \caption{Generalized \lietrotter scheme: Exploration of step size and batch size for different values of the (deterministic) integration length $N_l$. The grey dotted line denotes the self-distance for the distribution of the oracle.}
    \label{fig:generalizedLT}
\end{figure*}

%Any estimated distances are put into perspective by comparing them with the so-called \emph{self-distance}, which is the Kolmogorov distance of a sample from its actual generating distribution. For every experiment, we report the average and as well as the 0.05 and 0.95-quantiles of the self-distance distribution, which is estimated by considering $5$ independent runs of the corresponding oracle.

For \sdehmc, we examine integrators of different orders: \leapfrog, \lietrotter and \mtthree, and we investigate whether they differ from \sghmc \citep{chen2014stochastic}.
The \leapfrog scheme, which is provably of order 1 for \sdes, is also used in \sghmc.
The \lietrotter integrator introduced in \cref{sec:hmcvssdehmc} is of order 2, while \mtthree denotes the order-3 quasi-symplectic integrator of Milstein \& Tretyakov \citep{milstein2003quasi}.
We also compare against the symmetric splitting integrator proposed in \citet{chen2015convergence}, which we refer to as \symmetric.
In all cases we set $C=5$; we find this to be a reasonable choice, as can be seen in the exploration of \cref{ssec:explore_C}.

We consider four regression datasets (\boston, \concrete, \energy, \yacht) and two classification datasets (\iono, \vehicle) from the UCI repository, as well as a 1-D synthetic dataset, for which the regression result is shown in \cref{ssec:experiment_setup_appendix}.
Due to space limitations, here we only present a summary of the results in \cref{fig:methods,fig:generalizedLT}.
A more detailed exposition, including a more fine-grained exploration of the batch size, can be found in \cref{ssec:extended_results} and \cref{ssec:explore_C}.
In what follows, we summarize some findings that apply to all the datasets and models we have considered.

\begin{comment}
\begin{figure*}[h]
    \centering
    \includegraphics[width=\textwidth]{figures/fig_friction_bostonHousing_relu_50_4.pdf}
    \caption{Exploration of step size and batch size for different values of a scalar friction coefficient $C$. The grey dotted line denotes the self-distance for the distribution of the oracle. The \leapfrog integrator is used in all cases.}
    \label{fig:friction}
\end{figure*}
\end{comment}

\textbf{Comparing with \sghmc.}
In \cref{fig:methods} we focus our attention on the comparison between \sdehmc (with \leapfrog integrator) and \sghmc \citep{chen2014stochastic}.
We note that \sdehmc (\leapfrog) is different from \sghmc in the sense that no counterbalancing of the noise is performed.
Nevertheless, we do not observe significant difference between the two approaches.
We argue that this result is compatible with our position that counterbalancing the gradient noise is not necessary to sample from the posterior.

\textbf{On the convergence bottleneck.}
As a general remark, we notice that for larger mini-batch sizes (or equivalently, for smaller gradient noise) the model can admit larger step sizes, while as the gradient noise becomes larger, then a smaller value for $\eta$ is required to guarantee a good approximation.

When comparing integrators of different order in \cref{fig:minibottle,fig:methods}, the most interesting finding is that all of our results appear to confirm the existence of the convergence bottleneck identified in \cref{sec:minibatches}.
This becomes particularly obvious in \cref{fig:minibottle}, where we use \sdehmc to draw $2000$ samples.
In the cases where mini-batches are used, there is a zone in which higher-order integrators do not deliver significant improvements.
Note that this zone is absent from the full-batch comparison, as higher order schemes are consistently better everywhere in the range of step size $\eta$.

\textbf{Comparison of integrators.}
%The next point of discussion is the behaviour of different integrators.
In \cref{fig:minibottle,fig:methods}, although all of the integrators considered converge to the true posterior regardless of the mini-batch size given a sufficiently small step size, we observe that the methods behave differently depending on their theoretical properties.

The \symmetric and \lietrotter schemes, which are both of second order, respond similarly to the changes of step size
and batch size.
Although the \leapfrog scheme is provably of order 1, it has been mostly competitive to higher order schemes (\lietrotter, \symmetric, \mtthree) in our experiments; its performance deteriorates for larger values for $\eta$ only.
Lastly, the third order \mtthree scheme consistently outperforms the rest of the integrators in the full batch case, especially for larger step sizes; this has little effect from a practical point of view however, as \mtthree requires extra calculations per step (i.e.\ for 3 gradients and the Hessian).

% order 3 integrator is only useful in the full batch case

\textbf{Exploring the connection with \hmc.}
As a last experiment, we consider a generalized \lietrotter scheme where we explore different values for the integration length $N_l$.
For $N_l=1$ we obtain the standard \lietrotter of \cref{sec:hmcvssdehmc}, while as $N_l$ grows we recover a \hmc algorithm with partial momentum refreshments.
We have claimed that regardless of the $N_l$ value, the scheme enjoys ergodic convergence of order $\mathcal{O}(\eta^2)$.
The results of the exploration can be seen in Figure \ref{fig:generalizedLT}.
For smaller step sizes $\eta$, the order of convergence is shown to be similar for all $N_l$, especially in the full-batch case.
Nevertheless, we observe larger errors as $\eta$ grows and $N_l$ approaches $100$.
This effect seems to be further pronounced by the introduction of mini-batches.
Note that divergence that we observe for larger step sizes does not contradict our claim for same ergodic convergence, as this is not the region of interest in terms of $\mathcal{O}(\cdot)$ notation.
%Our position is that this should be attributed to the instability induced by the error in simulating paths.
This finding is not surprising nevertheless: it is well-established in the literature that \hmc is sensitive to the choice of integration length \citep{neal2011mcmc,Hoffman2014}.
We remark that a sensible choice is to use $N_l=1$, which is the most direct simulation of the Hamiltonian \sde system.

\section{CONCLUSIONS}
\label{sec:conclusions}

\begin{comment}
We reviewed the problem of sampling from a Bayesian posterior distribution by means of simulating a Hamiltonian \sde.
These schemes can be interpreted as extensions of the dynamics of the \hmc algorithm with a friction term and an
appropriately scaled Brownian motion.

We investigated the effect of gradient noise that is due to data subsampling in the context of differential operator
splitting. This allowed us to produce new convergence results for Hamiltonian \sde-based sampling methods, as well as
to draw a theoretical connection to the original \hmc algorithm.

We found that, for a selection of integrators of weak order two, both the discretization and the mini-batch error
vanish at rate $\mathcal{O}(\eta^2)$, for integrator step size $\eta$, revising previous literature results.
We have demonstrated this convergence on a wide range of experiments, where we have meticulously documented the
deviation from the true posterior distribution.
As a practical implication of our work, we recommended a straightforward simulation of the \sde in \cref{hamsde_exp} by
means of a quasi-symplectic second order integrator. Indeed, the bottleneck introduced by mini-batching would
neutralize the benefits of higher order integrators.

We also showed that the \hmc scheme can be interpreted as an integration scheme for the same class of \sdes.
The derived convergence rates are independent on the \hmc inner loop size $N_l$, as we confirmed experimentally.
Our empirical results suggested that a sensible option is to select $N_l=1$.
\end{comment}

In this work, we revisited the connections between \sg and stochastic Hamiltonian dynamics to improve our current
understanding of the role played by mini-batching on the goodness of sampling from intractable distributions, by means
of simulating a Hamiltonian \sde.
We challenged the common assumption of associating stochastic gradient estimates that arise due to data
subsampling to the stochastic component of the \sde, arguing that the Brownian motion is a poor model of this kind of
gradient noise.

Our main contribution was to produce new convergence results for Hamiltonian \sde-based sampling methods, by
studying their properties through the lenses of differential operator splitting.
%Our formulation also allowed us to draw a theoretical connection to the original \hmc algorithm.
We found that, for an integrator of weak order $p$ and step size $\eta$, the discretization error may vanish at rate $\mathcal{O}(\eta^p)$, but the mini-batch error vanishes at rate $\mathcal{O}(\eta^2)$, which is a bottleneck that has been overlooked in the literature.

Using our theory, we also showed that \hmc with partial momentum refreshments can be interpreted as an
integration scheme for the same class
of \sdes.
Then, we showed that convergence rates are independent on the \hmc inner loop size $N_l$, as we confirmed experimentally.
% Our empirical results suggested that a sensible option is to select $N_l=1$, and that a \lietrotter scheme is a preferable choice for Bayesian posterior sampling.

We have demonstrated the validity of our theory on a wide range of experiments, where we have meticulously documented
the deviation from the true posterior distribution.
As a practical implication of our work, we recommended a straightforward simulation of the \sde in \cref{hamsde_exp}
by means of a quasi-symplectic second order integrator.
Indeed, the bottleneck introduced by mini-batching would neutralize the benefits of higher order integrators.

\section{ACKNOWLEDGEMENTS}
MF gratefully acknowledges support from the AXA Research Fund and the Agence Nationale
de la Recherche (grant ANR-18-CE46-0002 and ANR-19-P3IA-0002).

%\newpage
%\clearpage
\bibliographystyle{abbrvnat_nourl}
\bibliography{biblio}

\onecolumn

\appendix
%\title{Revisiting the Effects of Stochasticity for Hamiltonian Samplers: \\
%Supplementary Materials}
\section{LIMITATIONS OF CURRENT SDE APPROACHES TO ANALYZE SG-BASED ALGORITHMS}\label{sec:sgdvssde}
%In the literature, \sdes are often used to describe the errors induced by the computation of gradients with subsets of data \citep{chen2014stochastic}. In this context, \sdes are presented as continuous-time versions of stochastic gradient descent (\sgd)-like algorithms, for which the step size approaches zero. In this work, we challenge this use of \sdes as a modeling tool for the behavior of \sg algorithms.

%The derivations in the literature on two main assumptions: (1) the learning rate (or step size) is small enough such that transforming the discrete time update equations into continuous time ones is a valid approximation, and (2) the mini-batch subsampling operation can be modeled, with the help of central limit theorems, as an additional source of noise that interferes with the dynamics. Here, we raise some issues considering this interpretation of \sgd, but similar considerations hold for all the sampling schemes discussed and cited in this work. 

In this Section, we expose a limitation of the literature on \sde modeling of stochastic gradient-based sampling and optimization algorithms, which in our view is contributing to foster an imprecise understanding and use of such algorithms, as we discuss in the main paper. 
A common approach to analyze the errors introduced by the mini-batching is to take the limit for a step size going to zero \citep{chen2014stochastic}.
The derivations in the literature are based mainly on two assumptions:
(i) the learning rate (or step size) is small enough such that transforming the discrete time update equations into continuous time ones is a valid approximation, and
(ii) the mini-batch subsampling operation can be modelled, with the help of the central limit theorem, as an additional source of Gaussian noise that interferes with the dynamics.
These assumptions, combined together, are used to justify the use of \sdes as models to study and analyze the dynamics of stochastic gradient-based algorithms.

% In the following, we take vanilla \sgd \citep{RobbinsMonroe} as a working example, but similar considerations hold for all the sampling schemes discussed and cited in this work. Considering the potential $U(\cdot)$, a set of parameters $\thetavect$, and a stochastic gradient $\nabla \tilde{U}(\thetavect)$ computed with mini-batches, by means of the central limit theorem, we assume:

According to the standard practice \citep{chen2014stochastic,mandt2017stochastic}, we consider the potential $U(\cdot)$, a set of parameters $\thetavect$, and a minibatch gradient $\nabla \tilde{U}(\thetavect)$; then, by means of central limit theorem we assume:
\begin{equation}
	\nabla \tilde{U}(\thetavect) = \nabla U(\thetavect) + \mathcal{N}(\mathbf{0}, \mathbf{\Sigma}),
\end{equation}
such that the discrete time dynamics of \sgd becomes:
\begin{equation}
\label{eq:sgd_dynamics}
\begin{split}
       \delta\thetavect &=- \eta \nabla U(\thetavect) - \eta \mathcal{N}(\mathbf{0}, \mathbf{\Sigma}), \\
                        &=- \eta \nabla U(\thetavect) - \mathbf{\Sigma}^{1/2} \sqrt{\eta} \mathcal{N}(0, \eta\mathbf{I}).
\end{split}
\end{equation}
In the limit of vanishing step size, $\eta \to dt$, we can translate the discrete time dynamics into the continuous time ones as follows:
\begin{equation}
	\boldsymbol{d\theta}=-\nabla U(\thetavect)dt + \mathbf{\Sigma}^{1/2} \sqrt{\eta} \mathbf{dw}(t),
\end{equation} 
%\maurizio{would it not make more sense to denote this as $d\mathbf{w}_t$?}
where $\mathbf{w_t}$ is a standard Brownian motion. Loosely speaking, Brownian motion is a stochastic process whose differential satisfies
$
        \mathbf{dw}(t) \sim \mathcal{N}(0,dt\mathbf{I}),
$
and whose increments are independent with respect to $t$.
Notice that the stochastic term of \cref{eq:sgd_dynamics} involves two copies of $\eta$, in the covariance $\eta\eye$ and in the external multiplication factor $\sqrt{\eta}$, but we consider $\eta \to dt$ only for the copy that scales the covariance of the Gaussian component.

%Note that, while the discussion about whether the \sde approximation for models involving mini-batch-based stochastic gradients is correct, or in any case useful, is outside the scope of this paper, we stress that this is still an open problem for this literature. \maurizio{Please check that my rewriting of this last sentence captures what you wanted to say.}
The issue with this derivation becomes apparent by considering a proper stochastic calculus perspective, 
where both $\eta$ terms are considered as infinitesimals, yielding $\sqrt{dt} \mathbf{dw}(t)=\mathbf{0}$.
%We argue that this is the reason why the literature is partially imprecise.
This is also explicitly noted by \citet{yaida2018fluctuation}, where the approximation $\eta^2\simeq \eta dt$ was criticized.
Considering the true limit, it is not possible to arbitrarily decide that the same quantity can be considered both as an infinitesimal and as an arbitrarily small but finite value.
The matter is not only a mathematical subtlety.
If one considers the correct formulation based on Fokker-Planck equations without going into continuous time but remaining in the discrete time \citep{leen1992weight}, noncentral moments of the stochastic gradients need to be considered.
Crucially, these induce quantitatively and qualitatively different results for the dynamics.

Note that our work is not in contrast with the literature whose purpose is the analysis of the stochastic gradient noise \citep{chaudhari2018stochastic,shi2020learning} or its approximate validity as a sampling method \citep{mandt2017stochastic}, where the formalism of \sdes is justified as a modelling tool.
Our argument can be summarized as follows: if we are interested in the formal characterization of the convergence of these sampling schemes, such an assumption is invalid. 
When considering mini-batching in the context of sampling via \sde simulations, it is not correct to consider the stochastic gradients as an additional source of noise.

\section{ADDITIONAL DERIVATIONS FOR SECTION \ref{sec:revisit}}\label{sec2:appendix}
\subsection{Technical assumptions}\label{sec:technass}
In this Section we present the technical assumptions needed for the analysis in the paper.

We need to ensure that the Kolmogorov equation associated with the \sde \cref{hamsde} has regular smooth solutions and unique invariant stationary distribution independent of initial conditions.
To ensure smoothness and uniqueness , it is sufficient to prove that the infinitesimal generator satisfies certain regularity conditions \citep{hormander1967hypoelliptic,mattingly2010convergence}, whereas to ensure ergodicity we require some growth conditions on the potential \citep{talay2002stochastic,abdulle2015long}.
Then, we need to assume that the functions $\phi(\cdot)$ satisfy certain conditions, to ensure that Taylor expansions of the form \cref{tayL} are valid.
Finally, we need to assume that the expected value of the considered numerical integration schemes can be validly expanded into series of the form \cref{expans}.

In the following we present sufficient conditions that guarantee the validity of the aforementioned assumptions. The reader interested in relaxing such technical conditions is referred to the original sources, such as \citet{abdulle2014high,abdulle2015long} and the works cited therein.

The first assumption guarantees that the Kolmogorov equation has regular unique smooth solutions and that the considered \sde is ergodic.
%% The summary is in the following.
\begin{assumption}\label{ass_1}
The potential $U(\thetavect)$ has continuous smooth bounded derivatives of arbitrary order, and it satisfies the growth condition
\begin{equation}
    \thetavect^\top\nabla U(\thetavect)\geq R_1||\thetavect||^2-R_2,
\end{equation}
for some $R_1,R_2>0$
\end{assumption}
It is important to notice that these requirements are independent from considerations on the choice of integrators, and without the assumption of ergodicity none of the \sde based sampling methods are conceptually meaningful. 

Next, we restrict the class of considered $\phi(\cdot)$ functions to ensure that Taylor expansions of the transformation of equation is valid.
\begin{assumption}\label{ass_2}
The function $\phi(\cdot)$ belongs to the set of functions $l$ times continuously differentiable with all the partial derivatives with polynomial growth, for some positive $l>0$. Furthermore there exists an $s\geq 0$ such that $|\phi(\thetavect)|\leq R_3(1+|\thetavect|^s)$ for some $R_3>0$.
\end{assumption}
This ensures that we can write 
\begin{equation}\label{tayL}
    \exp(\eta\mathcal{L})\phi=\sum\limits_{j=0}^{l}\frac{(\eta\mathcal{L})^j}{j!}\phi+\mathcal{O}(\eta^{l+1}),
\end{equation}
for generic $l$, \citep{talay1990expansion}.

To be able to perform the same kind of expansions for numerical integrators, we need the following assumption.
\begin{assumption}\label{ass_3}
Given an initial condition $\zvect_0$, any considered numerical integrator $\psi$ is such that
\begin{equation*}
   |\mathbb{E}(|\zvect_1-\zvect_0|)|\leq R_4(1+|\zvect_0|)\eta,\quad |\zvect_1-\zvect_0|\leq M_n(1+|\zvect_0|)\sqrt{\eta},
\end{equation*}
with $\zvect_1=\psi(\zvect_0)$, $R_4>0$ is independent of $\eta$, assumed small enough, and $M_n$ has bounded moments of all orders independent of $\eta$.
\end{assumption}
This allows us to write the expansion of $\mathbb{E}[\phi(\zvect_{\mathbf{1}})|\zvect_{\mathbf{0}}]$ up to generic order $z$ as
\begin{equation}\label{expans}
  \mathbb{E}[\phi(\zvect_{\mathbf{1}})|\zvect_{\mathbf{0}}] = \mathcal{U}\phi=\left(\mathcal{I}+\eta \mathcal{A}_1+\frac{\eta^2}{2} \mathcal{A}_2+\dots \frac{\eta^z}{z!} \mathcal{A}_z+\dots\right)\phi+\mathcal{O}(\eta^{z+1})
\end{equation}

\subsection{Proof of Theorem \ref{th:stationary}}\label{sec:proofhamsde}

We prove that the stationary distribution of \cref{hamsde} is the posterior of interest \citep{gardiner2004handbook}.
Start from the infinitesimal generator
\begin{equation*}
    \mathcal{L}=\underbrace{-\left(\nabla^\top_{\thetavect}U(\thetavect)\right)\nabla_{\mathbf{r}}+\left(\mathbf{M}^{-1}\mathbf{r}\right)^\top\nabla_{\thetavect}}_{\text{pure Hamiltonian evolution}\quad= \mathcal{H}}\underbrace{-C\left(\mathbf{M}^{-1}\mathbf{r}\right)^\top\nabla_{\mathbf{r}}+C\nabla_{\mathbf{r}}^\top\nabla_{\mathbf{r}}}_{\text{friction and Noise}\quad =\mathcal{D}}. \displaystyle
\end{equation*}
The adjoint of the first term, the pure Hamiltonian, can be rewritten as 
\begin{equation}
   \mathcal{H}^\dagger= \left(\nabla^\top_{\thetavect}U(\thetavect)\right)\nabla_{\mathbf{r}}-\left(\mathbf{M}^{-1}\mathbf{r}\right)^\top\nabla_{\thetavect}=\nabla_\zvect^\top H(\zvect)\begin{bmatrix} \zerovect & -\eye\\ \eye & \zerovect \end{bmatrix} \nabla_\zvect
\end{equation}
Assuming $\rho_{ss}(\zvect)\propto \exp\left(-H(\zvect)\right)$ we have
\begin{flalign*}
  \mathcal{H}^\dagger\rho_{ss}(\zvect) \propto &  \, \mathcal{H}^\dagger\exp\left(-H(\zvect)\right)=\nabla_\zvect^\top H(\zvect)\begin{bmatrix} \zerovect & -\eye\\ \eye & \zerovect \end{bmatrix} \left(\nabla_\zvect\exp\left(-H(\zvect)\right)\right) = \\
  &-\nabla_\zvect^\top H(\zvect)\begin{bmatrix} \zerovect & -\eye\\ \eye & \zerovect \end{bmatrix}\nabla_\zvect H(\zvect)\exp\left(-H(\zvect) \right) =0.
\end{flalign*}
Similarly, considering the adjoint of $\mathcal{D}$
\begin{equation}
 \mathcal{D}^\dagger=\nabla_{\mathbf{r}}^\top\left(C\left(\mathbf{M}^{-1}\mathbf{r}\right)\cdot\right)+C\nabla_{\mathbf{r}}^\top\nabla_{\mathbf{r}}=\nabla_{\mathbf{r}}^\top\left(C\left(\mathbf{M}^{-1}\mathbf{r}\right)\cdot+ C\nabla_{\mathbf{r}}\right).
\end{equation}
We study the term
\begin{flalign*}
&\left(C\left(\mathbf{M}^{-1}\mathbf{r}\right)\cdot+ C\nabla_{\mathbf{r}}\right)\rho_{ss}(\zvect)\propto\\ &\left(C\left(\mathbf{M}^{-1}\mathbf{r}\right)\cdot+ C\nabla_{\mathbf{r}}\right)\exp(-U(\thetavect)-\nicefrac{1}{2}\left||\mathbf{M}^{-1}\mathbf{r}\right||^2)=\\
&\exp(-U(\thetavect))\left(C\left(\mathbf{M}^{-1}\mathbf{r}\right)\exp(-\nicefrac{1}{2}\left||\mathbf{M}^{-1}\mathbf{r}\right||^2)+ C\nabla_{\mathbf{r}}\left(\exp(-\nicefrac{1}{2}\left||\mathbf{M}^{-1}\mathbf{r}\right||^2)\right)\right)=\\
&\exp(-U(\thetavect))\left(C\left(\mathbf{M}^{-1}\mathbf{r}\right)\exp(-\nicefrac{1}{2}\left||\mathbf{M}^{-1}\mathbf{r}\right||^2)-C\left(\mathbf{M}^{-1}\mathbf{r}\right)\exp(-\nicefrac{1}{2}\left||\mathbf{M}^{-1}\mathbf{r}\right||^2) \right)=0.
\end{flalign*}
Consequently $\mathcal{L}^\dagger\rho_{ss}=(\mathcal{H}^\dagger+\mathcal{D}^\dagger)\rho_{ss}=0$, concluding the proof.

\subsection{Ergodic order of convergence}\label{sec:abdulletheorem}
In this section we present the result of \citet{abdulle2015long}, that we use to state the order of generic schemes considered in this work.
\begin{proposition}\label{suffprop}
A sufficient condition for an integrator to be of a given ergodic error $p$, i.e. $e(\psi,\phi)=\mathcal{O}(\eta^p)$, is to have weak order $p$. This is not a necessary condition, as carefully described in \cref{abdoulle}
\end{proposition}

The following Theorem describes the generic conditions to achieve a given ergodic error. 
\begin{theorem}\label{abdoulle}
  \citep{abdulle2015long}. Consider the assumptions of Section \ref{sec:technass}. Suppose that the elements of the expansion of eq.~\cref{expans} satisfy 
\begin{equation}
    \mathcal{A}_1=\mathcal{L},\quad \mathcal{A}_j^\dagger \rho=0,\quad j=2,\dots r.
\end{equation}
Then the ergodic average has order of convergence $r$, i.e.
\begin{equation}\label{errabd}
 e(\psi,\phi)=\int \phi(\zvect)\rho^{\psi}_{ss}(\zvect) \dzvect-\int \phi(\zvect)\rho_{ss}(\zvect) \dzvect =\mathcal{O}(\eta^{r}).
\end{equation}
The result can also be refined as follows:
\begin{equation}\label{errabd2}
 e(\psi,\phi) =-\eta^{r}\int\limits_{0}^\infty\int \left(\mathcal{A}_{r+1}\exp(t\mathcal{L})\phi(\zvect)\right)\rho(\zvect)\dzvect+\mathcal{O}(\eta^{r+1})
\end{equation}
\end{theorem}

The proof is presented in \citet{abdulle2015long}.
A similar characterization has been attempted in \citet{chen2015convergence}, where the requirement to achieve a given convergence rate is that the integrator has weak local error of some given order. We suggest to rely instead on the weaker requirements $\mathcal{A}_j^\dagger \rho=0$ for $j=1,\dots r$.
In fact, if a method has local error of order $k$, then, by standard backward error analysis arguments, we have that:
$
    \mathcal{U}=\mathcal{I}+\eta \mathcal{L}+\frac{\eta^2}{2} \mathcal{L}^2+\dots \frac{\eta^k}{k!} \mathcal{L}^k+ \mathcal{O}(\eta^{k+1})
$
    where we easily notice that $\mathcal{A}_j=\mathcal{L}^{j}$, with $j=1,\dots k$.
    Consequently $\mathcal{A}_j^\dagger \rho=0$, for $j=1,\dots r$, since $\mathcal{L}^\dagger \rho=0$ due to the Fokker Planck equation. In summary, the ergodic error order is at least the weak order of the integrator.

\subsection{Quasi-symplectic integrators}\label{sec:proofquasi}
In the context of numerical simulations, it is a known fact that integrators that preserve the underlying geometry (a.k.a.\ symplectic structure) perform better than generic ones \citep{milstein2002symplectic,milstein2003quasi}. Symplectic \citep{milstein2002symplectic} and quasi symplectic \citep{milstein2003quasi,abdulle2014high,bou2010long} \sdes have been studied in the past, albeit in a different context than Bayesian sampling. Recently, symplectic geometry has been revisited, to characterize generic optimization problems \cite{betancourt2018symplectic,francca2020conformal,francca2021dissipative}, whereas in this work we are interested only on sampling properties.

The stochastic evolution induced by the \sde in \cref{hamsde}, starting from generic initial conditions $\zvect(0)$,  $\zvect(t)=\mathbf{\tau}(\zvect(0))$, dissipates the volume of regions exponentially fast. 
In precise mathematical terms, the evolution has Jacobian $\mathbf{\Omega}_{m,n}=\frac{\partial \tau_m(z)}{\partial z_n}$ that satisfies  $\det\left(\mathbf{\Omega}\right)=\exp\left(-C\mathrm{Tr}\left(\mathbf{M}^{-1}\right)t\right)$. 
The underlying geometry is in many aspects related to the one of symplectic manifolds, where volume is preserved exactly. 
It is natural to expect that numerical integrators that almost preserve the symplectic structure, defined as quasi-symplectic, when applied to systems as \cref{hamsde}, will perform better than generic integrators. 
An excellent exposition of the details of such statement is presented in \cite{milstein2002symplectic,milstein2003quasi, MILSTEIN200781}, while hereafter we present a self-contained, shorter, discussion.

We start our discussion by considering the definition of a generic \textit{symplectic} mapping. This is an important condition, since it can be shown that the evolution of deterministic Hamiltonian \ode naturally respects this condition \cite{francca2021dissipative}.
\begin{definition}\label{definition:s_ode}
A mapping 
$\boldsymbol{\kappa}:\zvect_0=\begin{bmatrix}\rvect_0^\top,\thetavect_0^\top\end{bmatrix}^\top \rightarrow \zvect_1=\begin{bmatrix}\rvect_1^\top,\thetavect_1^\top\end{bmatrix}^\top$ is said to be symplectic iff
\begin{equation}\label{sympcond}
    \mathbf{\Omega}^\top\mathbf{J}\mathbf{\Omega}=\mathbf{J},
\end{equation}
where $\mathbf{\Omega}_{m,n}=\frac{\partial (z_1)_m}{\partial (z_0)_n}$ is the Jacobian.
where $\mathbf{J}=\begin{bmatrix}\mathbf{0}&-\mathbf{I}\\ \mathbf{I} &\mathbf{0}
\end{bmatrix}$.
\end{definition}
A direct consequence of the symplectic property is that a symplectic mapping preserves the volume. Precisely, if we consider a region $\mathit{O}\subseteq \mathbb{R}^{2D}$, the volume after the mapping $\boldsymbol{\kappa}$ remains unchanged
\begin{equation}
    \mathrm{Vol}\left(\boldsymbol{\kappa}(\mathit{O})\right)=\int\limits_\mathit{O} |\det(\mathbf{\Omega}(\zvect))|\dzvect=\int\limits_\mathit{O}\dzvect,
\end{equation}
since from \cref{sympcond} we easily derive, considering $\det(\mathbf{\Omega}^\top\mathbf{J}\mathbf{\Omega})=\det(\mathbf{\Omega})^2$, that $\det(\mathbf{\Omega})=\pm 1$.

In this work we are interested however in stochastic Hamiltonian evolutions in the form of \cref{hamsde}. The first simple consideration we can make is that in the limit of vanishing friction, i.e. $C\rightarrow0$, the dynamics are symplectic mappings. This is trivially derived considering that in such a limit the \sde becomes an Hamiltonian \ode and the symplectic property is proved as in \cite{francca2021dissipative}. When $C\neq0$, it is possible to show that the stochastic mapping contracts exponentially in time the volume of regions. The precise statement is reported in the following Theorem.

\begin{theorem}\label{theorem:qssde}
Consider an Hamiltonian \sde of the form \cref{hamsde} and  
the stochastic evolution induced by this equation starting from generic initial conditions, i.e., $\zvect(t)=\mathbf{\tau}(\zvect(0))$.
Then $\tau$ is contracts the volume exponentially in time as $\det\left(\mathbf{\Omega}\right)=\exp\left(-C\mathrm{Tr}\left(\mathbf{M}^{-1}\right)t\right)$, where $\mathbf{\Omega}_{m,n}=\frac{\partial \tau_m(z)}{\partial z_n}$ is the (random) Jacobian.
\end{theorem}
\begin{proof}
Starting from \cref{hamsde}, we have
\begin{equation}
    \dzvect^\top(t)=\nabla^\top H(\zvect(t))\begin{bmatrix} -\Cfrict & -\eye\\ \eye & \zerovect \end{bmatrix}^\top dt+\begin{bmatrix} \sqrt{2C}\mathbf{dw}^\top(t) & \zerovect^\top \end{bmatrix}.
\end{equation}
Consequently,
\begin{flalign*}
   d\nabla_{\zvect(0)}\zvect^\top(t) & = \nabla_{\zvect(0)}\dzvect^\top(t)=\nabla_{\zvect(0)}\nabla^\top H(\zvect(t))\begin{bmatrix} -\Cfrict & -\eye\\ \eye & \zerovect \end{bmatrix}^\top dt+\nabla_{\zvect(0)}\begin{bmatrix} \sqrt{2C}\mathbf{dw}^\top(t) & \zerovect^\top \end{bmatrix} \\
   & = \nabla_{\zvect(0)}\nabla^\top H(\zvect(t))\begin{bmatrix} -\Cfrict & -\eye\\ \eye & \zerovect \end{bmatrix}^\top dt=\nabla_{\zvect(0)}\zvect^\top(t)\Delta H(\zvect(t))\begin{bmatrix} -\Cfrict & -\eye\\ \eye & \zerovect \end{bmatrix}^\top dt,
\end{flalign*}
where we use the shorthand $\Delta=\nabla\nabla^\top$.
The last expression can be rewritten as
\begin{equation*}
    \mathbf{d\Omega}^\top(t)=\mathbf{\Omega}^\top(t)\begin{bmatrix} \mathbf{M}^{-1} & \zerovect \\ \zerovect & \Delta_{\thetavect}U(\thetavect(t)) \end{bmatrix}\begin{bmatrix} -\Cfrict & -\eye\\ \eye & \zerovect \end{bmatrix}^\top dt.
\end{equation*}

By the rule of differentiation of determinants, we can write
\begin{flalign*}
  &d\det(\mathbf{\Omega}^\top(t))=\mathrm{Tr}\left(\mathrm{adj}\left(\mathbf{\Omega}^\top(t)\right) \mathbf{d\Omega}^\top(t)\right)=\\& \mathrm{Tr}\left(\mathrm{adj}\left(\mathbf{\Omega}^\top(t)\right) \mathbf{\Omega}^\top(t)\begin{bmatrix} \mathbf{M}^{-1} & \zerovect \\ \zerovect & \Delta_{\thetavect}U(\thetavect(t)) \end{bmatrix}\begin{bmatrix} -\Cfrict & -\eye\\ \eye & \zerovect \end{bmatrix}^\top dt\right)=\\
  &\det(\mathbf{\Omega}^\top(t))\mathrm{Tr}\left(\begin{bmatrix} \mathbf{M}^{-1} & \zerovect \\ \zerovect & \Delta_{\thetavect}U(\thetavect(t)) \end{bmatrix}\begin{bmatrix} -\Cfrict & -\eye\\ \eye & \zerovect \end{bmatrix}^\top dt\right)=\\
  &\det(\mathbf{\Omega}^\top(t))\mathrm{Tr}\left(\begin{bmatrix} \mathbf{M}^{-1} & \zerovect \\ \zerovect & \Delta_{\thetavect}U(\thetavect(t)) \end{bmatrix}\begin{bmatrix} -\Cfrict & \zerovect\\ \zerovect & \zerovect \end{bmatrix}^\top dt\right)=-\det(\mathbf{\Omega}^\top(t))C\mathrm{Tr}\left(\mathbf{M}^{-1}\right) dt.
\end{flalign*}
This, toghether with the condition $\mathbf{\Omega}(0)=\eye$ proves that
\begin{equation*}
   \det(\mathbf{\Omega}^\top(t))=\exp( -C\mathrm{Tr}\left(\mathbf{M}^{-1}\right)t)
\end{equation*}

\end{proof}

Having acknowledged the rich geometrical structure of Hamiltonian \odes and \sdes, we are finally ready to introduce the concept of quasi-symplectic integrators \cite{milstein2003quasi}, as numerical schemes that aim at preserving the underlying geometry. 

\begin{definition}\label{definition:qsintegrator}
  An numerical (stochastic) integrator $\phi: \zvect_i\rightarrow \zvect_{i+1}$ with step size $\eta$ is defined to be quasi-symplectic if the following two properties holds:
\begin{enumerate}
    \item The determinant of the Jacobian  $\mathbf{\Omega}=(\nabla_{\zvect_{i-1}}\phi(\zvect_{i})^\top)$ of the numerical integrator does not depend on $\zvect_{i}=(\rvect_{i},\thetavect_i)$
    \item In the limit of vanishing friction, i.e. $C \rightarrow 0$, the Jacobian $\mathbf{\Omega}$ satisfies the symplectic condition
    \begin{equation}
            \mathbf{\Omega}^\top\mathbf{J}\mathbf{\Omega}=\mathbf{J}.
    \end{equation}
\end{enumerate}
\end{definition}

%In \cref{sec:experiments} we analyze the performance of different integration schemes (both quasi-symplectic and not) and clearly show that quasi-symplectic methods outperform non quasi-symplectic ones, in accordance with the literature \cite{milstein2003quasi}.
%Among the quasi-symplectic class, we test \leapfrog, \symmetric, \lietrotter (that is quasi symplectic provided that the deterministic Hamiltonian integrator is symplectic) schemes.
%As for non quasi-symplectic schemes, we test the \euler and \rktwo ones.

\subsection{Numerical integrators}\label{sec:numint}
In this section, we present some examples of integrators, which we explore in the experimental Section \ref{sec:experiments} in the main paper,  with their corresponding order of convergence.

\subsubsection{\leapfrog}\label{leapscheme}
The second integrator we consider is the widely used \leapfrog scheme
\begin{equation}\label{eq:stoleap}
    \begin{cases}
    \thetavect^*=\thetavect_{i-1}+\frac{\eta}{2}\mathbf{M}^{-1}\rvect_{i-1}\\
    \rvect_{i}=\rvect_{i-1}-\eta \nabla U(\thetavect^*)-\eta C \mathbf{M}^{-1}\rvect_{i-1}+\sqrt{2 C\eta}\mathbf{w},\quad \mathbf{w}\sim \mathcal{N}(\zerovect,\eye)\\
    \thetavect_{i}=\thetavect^*+\frac{\eta}{2}\mathbf{M}^{-1}\rvect_{i}\\
    \end{cases}
\end{equation}
This scheme is quasi-symplectic as shown hereafter.
We rewrite equivalently the update scheme as a unique step as
\begin{equation}\label{eq:stoleap2}
    \begin{cases}
    \rvect_{i}=\rvect_{i-1}-\eta \nabla U(\thetavect_{i-1}+\frac{\eta}{2}\mathbf{M}^{-1}\rvect_{i-1})-\eta C \mathbf{M}^{-1}\rvect_{i-1}+\sqrt{2 C\eta}\mathbf{w}\\
    \thetavect_{i}=\thetavect_{i-1}+\frac{\eta}{2}\mathbf{M}^{-1}\rvect_{i-1}+\frac{\eta}{2}\mathbf{M}^{-1}\left(\rvect_{i-1}-\eta \nabla U(\thetavect_{i-1}+\frac{\eta}{2}\mathbf{M}^{-1}\rvect_{i-1})-\eta C \mathbf{M}^{-1}\rvect_{i-1}+\sqrt{2 C\eta}\mathbf{w}\right)\\
    \end{cases}
\end{equation}

Define $\mathbf{q}=\thetavect_{i-1}+\frac{\eta}{2}\mathbf{M}^{-1}\rvect_{i-1}$, $q^{(k)}=\theta^{(k)}_{i-1}+\eta m_{kp}r^{(p)}_{i-1}$, with $\mathbf{M}^{-1}|_{kp}=m_{kp}$.
\begin{flalign}
   & \frac{\p}{\p r^{(m)}_{i-1}}\p_{n}U(\mathbf{q})=\p_{nk}U(\mathbf{q})\frac{q^{(k)}}{\p r^{(m)}_{i-1}}=\p_{nk}U(\mathbf{q})\frac{\eta}{2} m_{kp}\delta_{pm}=\frac{\eta}{2} \p_{nk}U(\mathbf{q})m_{km}.
\end{flalign}
Consequently $\nabla_{\rvect_{i-1}}\nabla^\top U(\mathbf{q})=\frac{\eta}{2} \mathbf{M}^{-1}\Delta U(\mathbf{q})$
\begin{flalign}
 &\boldsymbol{\Omega}^\top=\begin{bmatrix}
 \eye-\eta C\mathbf{M}^{-1}-\frac{\eta^2}{2}\mathbf{M}^{-1}\Delta U(\mathbf{q})& \eta\mathbf{M}^{-1}-\frac{\eta^3}{4}\mathbf{M}^{-2}\Delta U(\mathbf{q})-\frac{\eta^2}{2}C\mathbf{M}^{-2}\\
 -\eta \Delta U(\mathbf{q})&\eye-\frac{\eta^2}{2}\mathbf{M}^{-1}\Delta U(\mathbf{q})
 \end{bmatrix}   
\end{flalign}
Simple algebraic manipulations show that the transposed Jacobian can be rewritten as
\begin{flalign}
 \boldsymbol{\Omega}^\top=\begin{bmatrix}
     \eye &\frac{\eta}{2}\mathbf{M}^{-1}\\
     \zerovect&\eye
 \end{bmatrix} \begin{bmatrix}
     \eye-\eta C\mathbf{M}^{-1}&\zerovect\\
     -\eta\Delta U(\mathbf{q})&\eye
 \end{bmatrix} \begin{bmatrix}
     \eye &\frac{\eta}{2}\mathbf{M}^{-1}\\
     \zerovect&\eye
 \end{bmatrix}   
\end{flalign}
The determinant of the Jacobian is then easily calculated as the product of the three determinants
\begin{flalign}
  \det(\boldsymbol{\Omega}^\top)=
  \det\begin{bmatrix}
     \eye &\frac{\eta}{2}\mathbf{M}^{-1}\\
     \zerovect&\eye
 \end{bmatrix}
 \det\begin{bmatrix}
     \eye-\eta C\mathbf{M}^{-1}&\zerovect\\
     -\eta\Delta U(\mathbf{q})&\eye
 \end{bmatrix}
 \det\begin{bmatrix}
     \eye &\frac{\eta}{2}\mathbf{M}^{-1}\\
     \zerovect&\eye
 \end{bmatrix}
 =\det\left(\eye-\eta C\mathbf{M}^{-1}\right),
\end{flalign}
and is independent on $\zvect_0$, satisfying the first condition.
By taking the limit $C\rightarrow0$, to prove that the integrator converge to a symplectic one is sufficient to prove that
\begin{flalign}
 \begin{bmatrix}
     \eye &\frac{\eta}{2}\mathbf{M}^{-1}\\
     \zerovect&\eye
 \end{bmatrix}\begin{bmatrix}\mathbf{0}&-\mathbf{I}\\ \mathbf{I} &\mathbf{0}
\end{bmatrix}\begin{bmatrix}
     \eye &\frac{\eta}{2}\mathbf{M}^{-1}\\
     \zerovect&\eye
 \end{bmatrix}^\top=\begin{bmatrix}\mathbf{0}&-\mathbf{I}\\ \mathbf{I} &\mathbf{0}
\end{bmatrix},
\end{flalign}
and that

\begin{flalign}\label{jleapcond2}
  \begin{bmatrix}
     \eye&\zerovect\\
     -\eta\Delta U(\mathbf{q})&\eye
 \end{bmatrix}  \begin{bmatrix}\mathbf{0}&-\mathbf{I}\\ \mathbf{I} &\mathbf{0}
\end{bmatrix} \begin{bmatrix}
     \eye&\zerovect\\
     -\eta\Delta U(\mathbf{q})&\eye
 \end{bmatrix}^\top   =\begin{bmatrix}\mathbf{0}&-\mathbf{I}\\ \mathbf{I} &\mathbf{0}
\end{bmatrix}.
\end{flalign}
Simple calculations show that
\begin{flalign*}
& \begin{bmatrix}
     \eye &\frac{\eta}{2}\mathbf{M}^{-1}\\
     \zerovect&\eye
 \end{bmatrix}\begin{bmatrix}\mathbf{0}&-\mathbf{I}\\ \mathbf{I} &\mathbf{0}
\end{bmatrix}\begin{bmatrix}
     \eye &\frac{\eta}{2}\mathbf{M}^{-1}\\
     \zerovect&\eye
 \end{bmatrix}^\top=\begin{bmatrix}
   \frac{\eta}{2}\mathbf{M}^{-1}&-\eye\\\eye &\zerovect  
 \end{bmatrix} \begin{bmatrix}
     \eye &\zerovect\\
     \frac{\eta}{2}\mathbf{M}^{-1}&\eye
 \end{bmatrix}=\begin{bmatrix}\mathbf{0}&-\mathbf{I}\\ \mathbf{I} &\mathbf{0}
\end{bmatrix},
\end{flalign*}
and similarly we can prove eq. \cref{jleapcond2}, completing the proof that \leapfrog is quasi-symplectic.
This integrator has a theoretical order of convergence equal to one but its performance is on par with other integrators of higher order. %  See also \ref{sec:leapmir} for a detailed discussion.

\subsubsection{\mtthree} \label{mt3scheme}
We then consider the quasi-symplectic scheme of third order that can be found in Section 2.3 of \citet{milstein2003quasi}:

\begin{equation}
\begin{split}
    \boldsymbol{\theta}_1 &= \boldsymbol{\theta}_k + \frac{7}{24} \eta \mathbf{M}^{-1} \mathbf{r}_k \\
    \mathbf{r}_1 &= \mathbf{r}_k + \frac{7}{24} \eta \left[-\nabla U(\boldsymbol{\theta}_1) - C \mathbf{M}^{-1} \mathbf{r}_1 \right]
\end{split}
\end{equation}

\begin{equation}
\begin{split}
    \boldsymbol{\theta}_2 &= \boldsymbol{\theta}_k + \frac{25}{24} \eta \mathbf{M}^{-1} \mathbf{r}_k 
    + \frac{\eta^2}{2} \mathbf{M}^{-1} \left[ -\nabla U(\boldsymbol{\theta}_1) - C \mathbf{M}^{-1} \mathbf{r}_1 \right]
    \\
    \mathbf{r}_2 &= \mathbf{r}_k + \frac{2}{3} \eta \left[-\nabla U(\boldsymbol{\theta}_1) - C \mathbf{M}^{-1} \mathbf{r}_1 \right]
    + \frac{3}{8} \eta \left[-\nabla U(\boldsymbol{\theta}_2) - C \mathbf{M}^{-1} \mathbf{r}_2 \right]
\end{split}
\end{equation}

\begin{equation}
\begin{split}
    \boldsymbol{\theta}_3 &= \boldsymbol{\theta}_k + \eta \mathbf{M}^{-1} \mathbf{r}_k 
    + \frac{17}{36} \eta^2 \mathbf{M}^{-1} \left[ -\nabla U(\boldsymbol{\theta}_1) - C \mathbf{M}^{-1} \mathbf{r}_1 \right]
    + \frac{1}{36} \eta^2 \mathbf{M}^{-1} \left[ -\nabla U(\boldsymbol{\theta}_2) - C \mathbf{M}^{-1} \mathbf{r}_2 \right]
    \\
    \mathbf{r}_3 &= \mathbf{r}_k + \frac{2}{3} \eta \left[-\nabla U(\boldsymbol{\theta}_1) - C \mathbf{M}^{-1} \mathbf{r}_1 \right]
    - \frac{2}{3} \eta \left[-\nabla U(\boldsymbol{\theta}_2) - C \mathbf{M}^{-1} \mathbf{r}_2 \right]
    + \eta \left[-\nabla U(\boldsymbol{\theta}_3) - C \mathbf{M}^{-1} \mathbf{r}_3 \right]
\end{split}
\end{equation}

\begin{equation}
\begin{split}
    \boldsymbol{\theta}_{k+1} &= \boldsymbol{\theta}_3 + \eta^{3/2} \mathbf{M}^{-1} \sqrt{2C} \left(\mathbf{w}^{(k)}_1/2 + \mathbf{w}^{(k)}_2 \right) \\
    &\quad - C \eta^{5/2} \mathbf{M}^{-2} \sqrt{2C}~\mathbf{w}^{(k)}_1 / 6
\end{split}
\end{equation}

\begin{equation}
\begin{split}
    \mathbf{r}_{k+1} &= \mathbf{r}_3 + \eta^{1/2} \sqrt{2C}~\mathbf{w}^{(k)}_1
    - C \eta^{3/2} \mathbf{M}^{-1} \sqrt{2C}~\left(\mathbf{w}^{(k)}_1/2 + \mathbf{w}^{(k)}_2 \right) \\
    &\quad + \eta^{5/2} \sum_{j=1}^n \left[ 
        \sum_{i=1}^n \left(\mathbf{M}^{-1} \sqrt{2C}~\mathbf{e_j} \right)^i 
            \frac{\partial}{\partial \theta^i}(-\nabla U(\boldsymbol{\theta}_3))
    \right] \mathbf{w}^{(k)}_{1,j} / 6  \\
    &\quad + \eta^{5/2} C^2 \mathbf{M}^{-2} \sqrt{2C}~\mathbf{w}^{(k)}_1 / 6
\end{split}
\end{equation}

\begin{comment}

\subsubsection{\rktwo}
We then consider the Runge-Kutta scheme of second order~\citep{kloeden2013numerical}
\begin{equation}\label{eq:k2}
    \begin{cases}
    \thetavect^*=\thetavect_{i-1}+\eta\mathbf{M}^{-1}\rvect_{i-1}\\
    \rvect^*=\rvect_{i-1}-\eta \nabla U(\thetavect_{i-1})-\eta C \mathbf{M}^{-1}\rvect_{i-1}+\sqrt{2 C\eta}\mathbf{w},\quad \mathbf{w}\sim \mathcal{N}(\zerovect,\eye)\\
    \thetavect_{i}=\thetavect_{i-1}+\frac{\eta}{2}\left(\mathbf{M}^{-1}\rvect_{i-1}+\mathbf{M}^{-1}\rvect^*\right)\\
    \rvect_i=\rvect_{i-1}+\frac{\eta}{2}\left(- \nabla U(\thetavect_{i-1})-C \rvect_{i-1}- \nabla U(\thetavect^*)-C \rvect^*\right)+\sqrt{2 C\eta}\mathbf{w}
    \end{cases}
\end{equation}
The order of convergence can be derived as in \cite{kloeden2013numerical},\cite{milstein2003quasi}. The scheme is not quasi-symplectic, as discussed also in \cite{milstein2003quasi}.

\end{comment}
\subsubsection{\lietrotter}\label{sec:lietr}
Finally we consider an order-two Lie-Trotter splitting scheme \citep{abdulle2015long} %, in this case of order two
\begin{equation}
 \begin{cases}
    \thetavect^*=\thetavect_{i-1}+\frac{\eta}{2}\mathbf{M}^{-1}\rvect_{i-1}\\
    \rvect^*=\rvect_{i-1}-\eta \nabla U(\thetavect^*)\\
    \thetavect^{i}=\thetavect^*+\frac{\eta}{2}\mathbf{M}^{-1}\rvect^*\\
    \mathbf{r}_i=\exp(-C\mathbf{M}^{-1}\eta)\mathbf{r}^*+\sqrt{\mathbf{M}(1-\exp(-2C\mathbf{M}^{-1}\eta))}\mathbf{w},\quad \mathbf{w}\sim \mathcal{N}(\zerovect,\eye)
    \end{cases}
\end{equation}
We choose the particular case of a deterministic leapfrog for the Hamiltonian integration, but any other integrator $\psi$ would have been valid. 
To derive the order of convergence, it is sufficient to use the result of Theorem \ref{abdoulle} as in \citet{abdulle2014high}. 

The scheme can be interpreted as the cascade of a purely deterministic symplectic integrator (the leapfrog step)
\begin{equation}
  \begin{cases}
  \rvect_{i-1}\\
  \thetavect_{i-1}
  \end{cases}\rightarrow\begin{cases}
    \rvect^*=\rvect_{i-1}-\eta \nabla U(\thetavect_{i-1}+\frac{\eta}{2}\mathbf{M}^{-1}\rvect_{i-1})\\
    \thetavect^*=\thetavect_{i-1}+\frac{\eta}{2}\mathbf{M}^{-1}\rvect_{i-1}+\frac{\eta}{2}\mathbf{M}^{-1}\left(\rvect_{i-1}-\eta \nabla U(\thetavect_{i-1}+\frac{\eta}{2}\mathbf{M}^{-1}\rvect_{i-1})\right)\\
    \end{cases}
\end{equation}
and the analytic step
\begin{equation}
  \begin{cases}
  \rvect^*\\
  \thetavect^*
  \end{cases}\rightarrow
  \begin{cases}
  \rvect_{i}=\exp(-C\mathbf{M}^{-1}\eta)\mathbf{r}^*+\sqrt{\mathbf{M}(1-\exp(-2C\mathbf{M}^{-1}\eta))}\mathbf{w}\\
  \thetavect_{i}=\thetavect^*
  \end{cases}
\end{equation}
With calculations similar to the ones in Section \ref{leapscheme} it is easy to show that the purely deterministic leapfrog step is symplectic, i.e., its Jacobian $\boldsymbol{\Omega}_1$ satisfies $\boldsymbol{\Omega}_1^\top\mathbf{J}\boldsymbol{\Omega}_1=\mathbf{J}$.

Concerning the analytical step instead, it is easy to show that the Jacobian of the transformation is equal to 
\begin{flalign}
 \boldsymbol{\Omega}_2^\top=\begin{bmatrix}
     \exp(-C\mathbf{M}^{-1}\eta) &\zerovect\\
     \zerovect&\eye
 \end{bmatrix}.
\end{flalign}
Consequently, the overall Jacobian of the two steps taken jointly is
\begin{equation}
   \boldsymbol{\Omega}^\top=\boldsymbol{\Omega}_1^\top\boldsymbol{\Omega}_2^\top
\end{equation}
and it is easily shown to be satisfying the two requirements for the quasi-symplecticity.
The \lietrotter scheme is then quasi-symplectic.

\section{ADDITIONAL DERIVATIONS FOR SECTION \ref{sec:minibatches}}\label{sec3:appendix}
We present the Baker–Campbell–Hausdorff in \cref{secproof:bch}, useful for technical derivations in the rest of the section. 
\cref{backandforth} contains results on the average of back and forth of multiple operators (\cref{fixedorder}). 
In \cref{sec:randomseq} we study the effect of chaining multiple (different) numerical integrators.
Finally, \cref{secproof:randomorderop} combines the elements of the previous Sections to provide the proof for \cref{randomorderop:specialcase}. \cref{secproof:genericintegratorerdogic} contains the Proof of \cref{genericintegratorerdogic}.
\subsection{Baker–Campbell–Hausdorff formula}\label{secproof:bch}

We here present the Baker–Campbell–Hausdorff formula, and in particular the series expansion due to \citet{dynkin1947calculation}, that will prove useful in the technical derivations. Given two differential operators $\mathcal{A},\mathcal{B}$, we have that
\begin{equation}\label{bch}
    \exp(\mathcal{A})\exp(\mathcal{B})=\exp(\mathcal{Z})
\end{equation}

where
$\mathcal{Z}=\mathcal{A}+\mathcal{B}+{\frac {1}{2}}[\mathcal{A},\mathcal{B}]+{\frac {1}{12}}\left([\mathcal{A},[\mathcal{A},\mathcal{B}]]+[\mathcal{B},[\mathcal{B},\mathcal{A}]]\right) -{\frac {1}{24}}[\mathcal{B},[\mathcal{A},[\mathcal{A},\mathcal{B}]]]\\\quad -{\frac {1}{720}}\left([\mathcal{B},[\mathcal{B},[\mathcal{B},[\mathcal{B},\mathcal{A}]]]]+[\mathcal{A},[\mathcal{A},[\mathcal{A},[\mathcal{A},\mathcal{B}]]]]\right) +{\frac {1}{360}}\left([\mathcal{A},[\mathcal{B},[\mathcal{B},[\mathcal{B},\mathcal{A}]]]]+[\mathcal{B},[\mathcal{A},[\mathcal{A},[\mathcal{A},\mathcal{B}]]]]\right)\\\quad +{\frac {1}{120}}\left([\mathcal{B},[\mathcal{A},[\mathcal{B},[\mathcal{A},\mathcal{B}]]]]+[\mathcal{A},[\mathcal{B},[\mathcal{A},[\mathcal{B},\mathcal{A}]]]]\right)+\cdots
$

%This expansion is due to \citet{dynkin1947calculation}. %Using similar techniques we show how to approximate sum of operators by chaining single constituent elements.

%We consider the following three equivalent representations for an operator
%\begin{equation}
%    \mathcal{P}=\begin{cases}
%    \mathcal{I}+\eta  {\mathcal{R}}_1+\eta^2  {\mathcal{R}}_2+\eta^3  {\mathcal{R}}_3+\dots\\
%    \exp(\eta \mathcal{E}(\eta))\\
%    \exp(\eta  {\mathcal{Q}}_1+\eta^2  {\mathcal{Q}}_2+\eta^3  {\mathcal{Q}}_3+\dots)
%    \end{cases}
%\end{equation}
%where simple algebra based on Taylor expansions provide the transformation rules from one representation to an other.
%In this section we do provide new theoretical results concerning the convergence rates of different integration schemes, with a particular focus on the role of minibatches.
\subsection{Average of back and forth operators}\label{backandforth}
The following Theorem provides a technical support for the proofs of the order of convergence of chains of operators. 

\begin{theorem}\label{fixedorder}
Suppose the differential operator $\mathcal{L}$ admits the following decomposition: 
\begin{equation}
    \mathcal{L}=\sum\limits_{i=1}^{K}\mathcal{L}_i,
\end{equation}
Then the equivalent operator $\mathcal{P}$ obtained by averaging the back $\left(\text{i.e.~}\prod\limits_{i=1}^{K}\exp{(\eta K\mathcal{L}_i)}\right)$ and the forth $\left(\text{i.e.~}\prod\limits_{i=K}^{1}\exp{(\eta K\mathcal{L}_i)}\right)$ evolution of the single terms has error $\mathcal{O}(K\eta^3)$ for the operator $\exp(\eta K\mathcal{L})$, i.e.
\begin{equation}
      \mathcal{P}=\frac{\prod\limits_{i=K}^{1}\exp{(\eta K \mathcal{L}_i)}+\prod\limits_{i=1}^{K}\exp{(\eta K \mathcal{L}_i)}}{2}=\exp(\eta K\mathcal{L})+\mathcal{O}(K\eta^3).
\end{equation}
\end{theorem}
%The proof is postponed to \ref{secproof:fixedorder}

%\subsection{Proof of \cref{fixedorder}}\label{secproof:fixedorder}
\begin{proof}
The proof of \cref{fixedorder} is as follows:
\begin{equation}
 \exp{(\eta K\mathcal{L}_i)}=\mathcal{I}+\eta K\mathcal{L}_i+\frac{\eta^2 K^2\mathcal{L}_i^2}{2}+ \mathcal{O}(\eta^3) 
\end{equation}
Consequently:
\begin{equation}
 \prod\limits_{i=1}^{K}\left(\mathcal{I}+\eta K\mathcal{L}_i+\frac{\eta^2 K^2\mathcal{L}_i^2}{2}+ \mathcal{O}(\eta^3) \right) =\mathcal{I}+\eta\sum\limits_{i=1}^{K}K\mathcal{L}_i+\frac{\eta^2K^2}{2}\sum\limits_{i=1}^{K}\mathcal{L}_i^2+\eta^2K^2\sum\limits_{\substack{i_1=1\dots K-1,\\ i_2=i_1+1\dots K}}\mathcal{L}_{i_1}\mathcal{L}_{i_2}+ \mathcal{O}((K-1)\eta^3) 
\end{equation}
and 
\begin{equation}
 \prod\limits_{i=K}^{1}\left(\mathcal{I}+\eta K\mathcal{L}_i+\frac{\eta^2K^2\mathcal{L}_i^2}{2}+ \mathcal{O}(\eta^3) \right) =\mathcal{I}+\eta\sum\limits_{i=1}^{K}K\mathcal{L}_i+\frac{\eta^2K^2}{2}\sum\limits_{i=1}^{K}\mathcal{L}_i^2+\eta^2K^2\sum\limits_{\substack{i_1=2\dots K,\\ i_2=1\dots i_1-1}}\mathcal{L}_{i_1}\mathcal{L}_{i_2}+ \mathcal{O}((K-1)\eta^3)
\end{equation}
The fundamental equality of interest is:
\begin{flalign*}
    \mathcal{L}^2&=\left(\sum\limits_{i_1=1}^{K}\mathcal{L}_{i_1}\right)\left(\sum\limits_{i_2=1}^{K}\mathcal{L}_{i_2}\right)=\left(\sum\limits_{\substack{i_1=1\dots K,\\ i_2=1\dots K}}\mathcal{L}_{i_1}\mathcal{L}_{i_2}\right)=\left(\sum\limits_{\substack{i_1=1\dots K,\\ i_2=i_1\dots K}}\mathcal{L}_{i_1}\mathcal{L}_{i_2}+\sum\limits_{\substack{i_1=1\dots K,\\ i_2=1\dots i_1-1}}\mathcal{L}_{i_1}\mathcal{L}_{i_2}\right)\\
    &=\left(\sum\limits_{\substack{i_1=1\dots K,\\ i_2=i_1}}\mathcal{L}_{i_1}\mathcal{L}_{i_2}+\sum\limits_{\substack{i_1=1\dots K,\\ i_2=i_1+1\dots K}}\mathcal{L}_{i_1}\mathcal{L}_{i_2}+\sum\limits_{\substack{i_1=1\dots K,\\ i_2=1\dots i_1-1}}\mathcal{L}_{i_1}\mathcal{L}_{i_2}\right)\\
    &=\left(\sum\limits_{\substack{i_1=1\dots K,\\ i_2=i_1}}\mathcal{L}_{i_1}\mathcal{L}_{i_2}+\sum\limits_{\substack{i_1=1\dots K,\\ i_2=i_1+1\dots K}}\mathcal{L}_{i_1}\mathcal{L}_{i_2}+\sum\limits_{\substack{i_1=2\dots K,\\ i_2=1\dots i_1-1}}\mathcal{L}_{i_1}\mathcal{L}_{i_2}\right)
\end{flalign*}

By means of simple algebric considerations we can thus obtain the desired result, as
\begin{flalign*}
     \mathcal{P}&=\frac{\prod\limits_{i=K}^{1}\exp{(\eta K \mathcal{L}_i)}+\prod\limits_{i=1}^{K}\exp{(\eta K\mathcal{L}_i)}}{2}\\
     &=\frac{\prod\limits_{i=1}^{K}\left(\mathcal{I}+\eta K\mathcal{L}_i+\frac{\eta^2K^2\mathcal{L}_i^2}{2}+ \mathcal{O}(\eta^3) \right) +\prod\limits_{i=K}^{1}\left(\mathcal{I}+\eta K\mathcal{L}_i+\frac{\eta^2K^2\mathcal{L}_i^2}{2}+ \mathcal{O}(\eta^3) \right)}{2}\\
     &=\mathcal{I}+\eta\sum\limits_{i=1}^{K}K\mathcal{L}_i+\frac{\eta^2K^2}{2}\left((\sum\limits_{i_1=1}^{K}\mathcal{L}_{i_1})(\sum\limits_{i_2=1}^{K}\mathcal{L}_{i_2})\right)+\mathcal{O}((K-1)\eta^3)=\exp(\eta K\mathcal{L})+\mathcal{O}(K\eta^3)
\end{flalign*}
\end{proof}

\subsection{Randomized sequence of numerical integration steps}\label{sec:randomseq}
The first aspect we consider is the net effect of chaining multiple (different) numerical integrators on the transformation of $\phi$.
\begin{theorem}\label{detorder}
Consider a set of numerical integrators $\{\psi_i\}_{i=1}^K$ with corresponding functionals $\mathcal{U}_i$. If we build the stochastic variable $\zvect_{fin}$ starting from $\zvect_{0}$ through the following chain
\begin{equation}
    \zvect_{fin}=\psi_K(\psi_{K-1}(\dots\psi_1(\zvect_{0}))),
\end{equation}
then the generic function $\phi$ is transformed as follows
\begin{equation}
        \mathbb{E}\left[\phi(\zvect_{fin})\right]=\mathcal{U}_1\mathcal{U}_2\dots\mathcal{U}_K\phi(\zvect_{0})
\end{equation}
\end{theorem}
Notice the backward order of the differential operators and the difference with the ordering assumed in \citet{chen2015convergence}.
\begin{proof}
We give the proof for $K=2$ since the general result is trivially extended from this case.
In this case $
   \zvect_{fin}=\psi_2(\psi_1(\zvect_{0}))
$.
Define $\zvect_{1}=\psi_1(\zvect_{0})$.
Since
\begin{flalign*}
 &\mathbb{E}\left[\phi(\zvect_{fin})\right]=\int \phi(\zvect_{fin})p(\zvect_{fin}|\zvect_{0})\dzvect_{fin}=\int \phi(\zvect_{fin})p(\zvect_{fin}|\zvect_{1})p(\zvect_{1}|\zvect_{0})\dzvect_{1}\dzvect_{fin}=\\&\int \left(\int\phi(\zvect_{fin})p(\zvect_{fin}|\zvect_{1})\dzvect_{fin}\right)p(\zvect_{1}|\zvect_{0})\dzvect_{1}=\int\left( \mathcal{U}_2\phi(\mathbf{z_1})\right)p(\zvect_{1}|\zvect_{0})\dzvect_{1}=\mathcal{U}_1\mathcal{U}_2\phi(\zvect_{0})
\end{flalign*}
the result is proven.
\end{proof}

Next, we consider the case in which we consider different chains of numerical integrators and apply one of them according to some probability distribution.
\begin{theorem}\label{multiset}
Consider multiple sets of numerical integrators $\{\psi_{i,j}\}_{i=1}^K$ with corresponding functionals $\mathcal{U}_{i,j}$, with $j=1\dots S$. Choose an index $j^*\in[1,\dots S]$ according  to some probability distribution $p(j^*=j)=w_j$. Build the stochastic variable $\zvect_{fin}$ starting from $\zvect_{0}$ through the chain corresponding to the sampled index
\begin{equation}
    \zvect_{fin}=\psi_{K,j^*}(\psi_{K-1,j^*}(\dots\psi_{1,j^*}(\zvect_{0}))),
\end{equation}
then the generic function $\phi$ is transformed as follows
\begin{equation}
        \mathbb{E}\left[\phi(\zvect_{fin})\right]=\sum\limits_{j=1}^{S}w_j\mathcal{U}_{1,j}\mathcal{U}_{2,j}\dots\mathcal{U}_{K,j}\phi(\zvect_{0}).
\end{equation}
\end{theorem}

The result is trivially derived considering the previous theorem and the law of total expectation. Indeed
\begin{flalign*}
&\mathbb{E}\left[\phi(\zvect_{fin})\right]=\mathbb{E}_{j^{*}}\left[\mathbb{E}\left[\phi(\psi_{K,j^*}(\psi_{K-1,j^*}(\dots\psi_{1,j^*}(\zvect_{0}))))\right]\right]=\\
&\mathbb{E}_{j^{*}}\left[\mathcal{U}_{1,j^*}\mathcal{U}_{2,j^*}\dots\mathcal{U}_{K,j^*}\phi(\zvect_{0})\right]=\sum\limits_{j=1}^{S}w_j\mathcal{U}_{1,j}\mathcal{U}_{2,j}\dots\mathcal{U}_{K,j}\phi(\zvect_{0}).
\end{flalign*}

\subsection{Proof of \cref{randomorderop:specialcase}}\label{secproof:randomorderop}
As stated in \cref{randomorderop:specialcase}, we consider numerical integrators $\psi_i$ with order $p$. Consequently, we have representation $\mathcal{U}_i=\exp(\eta K\mathcal{L}_i)+\mathcal{O}(\eta^{p+1})$ of the operators $\mathcal{U}_i$ corresponding to the integrators $\psi_i$. The proof of \cref{randomorderop:specialcase} is obtained considering \cref{fixedorder} and \cref{multiset}.

Consider $K!$ sets of numerical integrators as follows:
\begin{flalign}
    &\{\psi_{\pi^1_i}\}_{i=1}^K,\quad\{\psi_{\pi^2_i}\}_{i=1}^K,\dots \{\psi_{\pi^{K!}_i}\}_{i=1}^K,
\end{flalign}
where $\boldsymbol{\pi}^1,\dots \boldsymbol{\pi}^{K!}$ are all the possible permutations of the $K$ operators, sampled uniformly with probability $\frac{1}{K!}$.
If, in accordance with the hypotheses of \cref{randomorderop:specialcase}, the initial condition has stochastic evolution through a chain of randomly permuted integrators, then the propagation operator is:
\begin{equation}
    \mathcal{U}=\sum\limits_{i=1}^{K!}\frac{1}{K!}\mathcal{U}_{\pi_K^i}\mathcal{U}_{\pi_{K-1}^i}\dots\mathcal{U}_{\pi_1^i}.
\end{equation}
The proof is a direct consequence of \cref{multiset}.

It is obviously possible to rewrite the set of all possible permutations as the union of two sets of equal cardinality, such that each element of one of the two set has a (unique) element of the other set that is its lexicographic reverse. 
\begin{flalign*}
    &\{\psi_{\pi^1_i}\}_{i=1}^K,\quad\{\psi_{\pi^2_i}\}_{i=1}^K,\dots \{\psi_{\pi^{K!/2}_i}\}_{i=1}^,K\\
    &\{\psi_{\pi^1_i}\}_{i=K}^1,\quad\{\psi_{\pi^2_i}\}_{i=K}^1,\dots\{\psi_{\pi^{K!/2}_i}\}_{i=K}^1.
\end{flalign*}

We rewrite the sum then as:
\begin{equation*}
      \mathcal{U}=\sum\limits_{i=1}^{K!/2}\frac{2}{K!}\left(\frac{1}{2}\mathcal{U}_{\pi_K^i}\mathcal{U}_{\pi_{K-1}^i}\dots\mathcal{U}_{\pi_1^i}+\frac{1}{2}\mathcal{U}_{\pi_1^i}\mathcal{U}_{\pi_{2}^i}\dots\mathcal{U}_{\pi_K^i} \right).
\end{equation*}
Then, by \cref{fixedorder}, since:
\begin{equation}
    \frac{1}{2}\mathcal{U}_{\pi_K^i}\mathcal{U}_{\pi_{K-1}^i}\dots\mathcal{U}_{\pi_1^i}+\frac{1}{2}\mathcal{U}_{\pi_1^i}\mathcal{U}_{\pi_{2}^i}\dots\mathcal{U}_{\pi_K^i}=\exp\left(\eta K\mathcal{L}\right)+\mathcal{O}(K\eta^3)+\mathcal{O}(K\eta^{p+1})%\mathcal{O}(K\eta^{\min(p+1,3)}),%
\end{equation}
we have
\begin{equation}
  \mathcal{U} = \exp\left(\eta K\mathcal{L}\right)+\mathcal{O}(K\eta^3)+\mathcal{O}(K\eta^{p+1}).%+\mathcal{O}(K\eta^{\min(p+1,3)}).%
\end{equation}

\subsection{Proof of \cref{genericintegratorerdogic} }\label{secproof:genericintegratorerdogic}
We are considering numerical integrators whose corresponding operators are
\begin{equation}
    \mathcal{U}_i=\exp(\eta K\mathcal{L}_i)+\mathcal{O}(\eta^{p+1})
\end{equation}
Applying the randomized ordering scheme to the considered setting induces then an operator (\cref{randomorderop:specialcase})
\begin{equation}
    \mathcal{U}=\exp(\eta K\mathcal{L})+\mathcal{O}(K\eta^{\min(p+1,3)}).%+\mathcal{O}(K\eta^{3})+\mathcal{O}(K\eta^{p+1})%
\end{equation}
To apply \cref{abdoulle} we first make a change of variable $h=K\eta$. Then we expand the operator as
\begin{equation}
    \mathcal{U}=\mathcal{I}+h\mathcal{L}+\frac{h^2}{2}\mathcal{L}^2+{O}(K^{-\min(p,2)}h^{\min(p+1,3)}).%\mathcal{O}(K^{-2}h^{3})+\mathcal{O}(K^{-p}h^{p+1}).%+\mathcal{O}(K\eta^{\min(p+1,3)})%
\end{equation}
Then, the application of \cref{abdoulle} provides the ergodic order of convergence is desired result, as the ergodic error is 
\begin{flalign*}
    {O}(K^{-\min(p,2)}h^{\min(p,2)})={O}(\eta^{\min(p,2)}).
\end{flalign*}

\section{ADDITIONAL DERIVATIONS FOR SECTION \ref{sec:hmcvssdehmc}}\label{sec4:appendix}
We here provide the proofs to the claims stated in the main paper concerning the convergence rates of \hmc schemes with and without mini-batches. We prove in \cref{hmcproof} that the convergence rate of \hmc is $\mathcal{O}(\eta^p)$. In \cref{hmcproofmb} we prove that the  convergence rate of \hmc with mini-batches $\mathcal{O}(\eta^{\min(p,2)})$. 

\subsection{Proof of \hmc convergence rate}\label{hmcproof}

\begin{comment}
\begin{equation}\label{hmcscheme}
    \zvect^*=\begin{bmatrix}{\rvect^*}^\top,{\thetavect^*}^\top
    \end{bmatrix}^\top=\underbrace{\psi(\dots\psi(\psi(}_{N_l \text{times}}\zvect_0))),\quad\begin{bmatrix}\rvect_1^\top,\thetavect_1^\top\end{bmatrix}=\begin{bmatrix}\alpha{\rvect^*}^\top+\sqrt{1-\alpha^2}\mathbf{w}^\top,{\thetavect^*}^\top
    \end{bmatrix}.
\end{equation}

The proof of the equivalence between \hmc and a variant of \lietrotter scheme, including its convergence rate, is done by construction.

The classical \lietrotter scheme consists in alternating steps of pure $\mathcal{H}$ dynamics, considering a numerical integrator $\psi$, and $\mathcal{D}$ dynamics, that can be solved exactly, as follows
\begin{equation}
    \zvect^*=\begin{bmatrix}{\rvect^*}^\top,{\thetavect^*}^\top
    \end{bmatrix}^\top=\psi(\zvect_0),\quad\begin{bmatrix}\rvect_1^\top,\thetavect_1^\top\end{bmatrix}=\begin{bmatrix}\exp(-C \eta)\boldsymbol{r}^*+\sqrt{1-\exp(-2 C \eta)}\mathbf{w},{\thetavect^*}^\top
    \end{bmatrix}.
\end{equation}
Similarly, it is possible to construct a \lietrotter scheme in which $N_l$ steps of $\mathcal{H}$ dynamics and $N_l$ steps of $\mathcal{D}$ dynamics are alternated, providing the following scheme  
\end{comment}

We shall rewrite for simplicity the \hmc scheme
\begin{flalign}\label{multilie}
   & \zvect^*=\begin{bmatrix}{\rvect^*}^\top,{\thetavect^*}^\top
    \end{bmatrix}^\top=\underbrace{\psi(\dots\psi(\psi(}_{N_l \text{times}}\zvect_0))),\nonumber\\
    &\begin{bmatrix}\rvect_1^\top,\thetavect_1^\top\end{bmatrix}=\begin{bmatrix}\exp(-C N_{l}\eta){\rvect^*}^\top+\sqrt{1-\exp(-2 C N_{l}\eta)}\mathbf{w}^\top,{\thetavect^*}^\top
    \end{bmatrix}.
\end{flalign}
where we selected $C$ such that $\exp(-C N_{l}\eta)=\alpha$. By doing so we can formally levarage the results concerning \sde integration schemes. The pure Hamiltonian steps transform the functions $\phi$ with operator 
\begin{equation}
    \mathcal{P}=\exp(\eta N_l \mathcal{H})+\mathcal{O}(N_l\eta^{p+1}).
\end{equation}
The dynamics of an \sde with generator term $\mathcal{D}$ are simulated exactly, and consequently the overall operator that transforms the $\phi$ functions can then be expressed as
\begin{flalign}\label{hmcintU}
    &\mathcal{U}=\mathcal{P}\exp(\eta N_l\mathcal{D})=(\exp(\eta N_l \mathcal{H})+\mathcal{O}(N_l\eta^{p+1}))\exp(\eta N_l\mathcal{D})=\nonumber\\
    &\exp(\eta N_l \mathcal{H})\exp(\eta N_l\mathcal{D})+\mathcal{O}(N_l\eta^{p+1}).
\end{flalign}

To apply directly \cref{abdoulle}, see also \citet{abdulle2014high,abdulle2015long}, it is convenient to rewrite \eqref{hmcintU} as a unique step with step-size $h=N_l\eta$, i.e.
\begin{equation}
\mathcal{U}= \exp(h \mathcal{H})\exp(h\mathcal{D})+\mathcal{O}(N_l^{-p}h^{p+1})   
\end{equation}

We expand in Taylor series (look at \citet{abdulle2015long} for a similar exposition) the terms of the product 
\begin{flalign*}
&\exp(h\mathcal{H})\exp(h\mathcal{D})=\sum\limits_{n=0}^{\infty}h^n\frac{(\mathcal{H})^n}{n!}\sum\limits_{m=0}^{\infty}h^m\frac{(\mathcal{D})^m}{m!}=\\&\sum\limits_{k=0}^{\infty}h^k \left(\sum\limits_{n+m=k}\frac{(\mathcal{H})^n}{n!}\frac{( \mathcal{D})^m}{m!}\right).
\end{flalign*}
Simple operator algebra shows that
\begin{equation}\label{eq:opexp}
 (\exp(h\mathcal{H})\exp(h\mathcal{D}))^\dagger=\sum\limits_{k=0}^{\infty}h^k \left(\sum\limits_{n+m=k}\frac{( \mathcal{D}^\dagger)^m}{m!}\frac{(\mathcal{H}^\dagger)^n}{n!}\right).
\end{equation}
Then
\begin{flalign}\label{uexpop}
    &\mathcal{U}=\sum\limits_{k=0}^{\infty}h^k \left(\sum\limits_{n+m=k}\frac{( \mathcal{D}^\dagger)^m}{m!}\frac{( \mathcal{H}^\dagger)^n}{n!}\right)+\mathcal{O}(N_l^{-p}h^{p+1}).
\end{flalign}

We use the fact that $\mathcal{H}^\dagger\rho_{ss}=\mathcal{D}^\dagger\rho_{ss}=0$ and with the help of Theorem \ref{abdoulle} prove that the ergodic error is $\mathcal{O}(N_l^{-p}h^{p})$. By simple substitution this is the desired $\mathcal{O}(\eta^{p})$.

\subsection{Proof of \hmc convergence rate with mini-batches}\label{hmcproofmb}

We then proceed to prove the generic ergodic error when considering \hmc with mini-batches. As stated in the main text, we assume that $N_l=TK$ with $T$ a positive integer. We define the \hmc scheme with mini-batches as follows.

\begin{proposition}
Consider a class of deterministic numerical integrators with order $p$. Split the datasets into $K$ batches and consider $\psi_i$ that solves $\mathcal{H}_i=-K\left(\nabla^\top_{\thetavect}U_i(\thetavect)\right)\nabla_{\mathbf{r}}+\left(\left(\mathbf{M}^{-1}\mathbf{r}\right)^\top\right)\nabla_{\thetavect}$. Sample independently $T=\frac{N_l}{K}$ random permutations $\{\boldsymbol{\pi}^{j}\}_{j=1}^{N_l}, \boldsymbol{\pi}^{j}\in\mathbb{P}$ and apply the scheme \begin{flalign}
    &\mathbf{z}^*=\psi_{\pi^{T}_K} \circ \dots\psi_{\pi^{T}_1} \dots\psi_{\pi^{2}_K}\circ \dots\psi_{\pi^{2}_1} \circ  
    \psi_{\pi^{1}_K} \circ \dots\psi_{\pi^{1}_1} \zvect_0,\quad  \begin{bmatrix}\thetavect_1^\top,\rvect_1^\top\end{bmatrix}=\begin{bmatrix}{\thetavect^*}^\top,\mathbf{w}^\top,
    \end{bmatrix}
\end{flalign}

\end{proposition}

We proceed then to prove the desired result. We have 
\begin{flalign}
 &\mathbb{E}(\phi(\mathbf{z}^*))=\prod\limits_{i=\pi_K^1,\dots,\pi_1^1}\left(\exp(\eta K\mathcal{H}_i)+\mathcal{O}(\eta^{p+1})\right)\dots\prod\limits_{i=\pi_K^T,\dots,\pi_1^T}\left(\exp(\eta K\mathcal{H}_i)+\mathcal{O}(\eta^{p+1})\right) \phi(\zvect_0)
\end{flalign}
where the expectation are not yet taken w.r.t the random permutations.
As the random permutations are independent, taking the expected value w.r.t them,
\begin{flalign}
 &\mathbb{E}[\phi(\mathbf{z}^*)]=\mathbb{E}_{\boldsymbol{\pi}^1}[\prod\limits_{i=\pi_K^1,\dots,\pi_1^1}\left(\exp(\eta K\mathcal{H}_i)+\mathcal{O}(\eta^{p+1})\right)]\dots\mathbb{E}_{\boldsymbol{\pi}^T}[\prod\limits_{i=\pi_K^T,\dots,\pi_1^T}\left(\exp(\eta K\mathcal{H}_i)+\mathcal{O}(\eta^{p+1})\right)] \phi(\zvect_0)=\\
 &\left(\exp(\eta K\mathcal{H})+\mathcal{O}(K\eta^{\min(p+1,3)})\right)\dots\left(\exp(\eta K\mathcal{H})+\mathcal{O}(K\eta^{\min(p+1,3)})\right)\phi(\zvect_0)=\\
 &\exp(\eta TK \mathcal{H})+\mathcal{O}(TK\eta^{\min(p+1,3)})\phi(\zvect_0).
\end{flalign}
The $\mathcal{D}$ step can be solved analytically, then the overall operator that transform functions is
\begin{flalign*}
    &\mathcal{S}=\exp(\eta TK \mathcal{D})\left(\exp(\eta TK\mathcal{H})+\mathcal{O}(TK \eta^{\min(p+1,3)})\right)=\exp(\eta T K\mathcal{D})\exp(\eta TK \mathcal{H})+\mathcal{O}(TK\eta^{\min(p+1,3)})
   % +\mathcal{O}(KN_l\eta^3)+\mathcal{O}(KN_l\eta^{p+1})
\end{flalign*}
Similarly to the full batch case, we define $h=TK\eta$, then the operator is rewritten as
\begin{flalign*}
    &\mathcal{S}=\exp(h \mathcal{D})\exp(h \mathcal{H})+\mathcal{O}((TK)^{-\min(p,2)}h^{\min(p+1,3)})
   % +\mathcal{O}(KN_l\eta^3)+\mathcal{O}(KN_l\eta^{p+1})
\end{flalign*}
The ergodic error is then of $\mathcal{O}(\eta^{\min(p,2)})$. 
\section{TOY EXAMPLE}\label{sec:toyexapp}
%In this section we elaborate on some important differences between our work and the work in \cite{chen2015convergence}. 

% Our main focus has been the stationary regime of the simulated dynamics, while in \cite{chen2015convergence} the whole non asymptotic regime is studied. More specifically, we focus only on the average ergodic error, which is called ``bias'' in \cite{chen2015convergence}. Instead the work in \cite{chen2015convergence} considers also the Monte Carlo error due to finite number of samples for the integration is considered. In the spirit of a clean presentation, we do not consider this extra source of error, as done in \cite{abdulle2014high, abdulle2015long}. Nevertheless we carefully explore experimentally its impact by providing confidence intervals of all the measured quantities in \cref{sec:experiments} obtained over multiple seeds. From a technical point of view, we furthermore relax some of the assumptions for the numerical integrators to achieve a given order of convergence, as explicitely characterized in \cref{sec:abdulletheorem}.

We remark that Theorem 4 in \citet{chen2015convergence} is not aligned with the results we present in this work. Indeed, the work in \citet{chen2015convergence} suggests that the only effects on the asymptotic invariant measure of an \sgmcmc algorithm are due to the order of the numerical integrators, independently of mini-batching. Instead, our results in \cref{genericintegratorerdogic}, show that an explicit bottleneck of order two is present due to mini-batches, i.e. due to the stochastic nature of gradient.

\begin{comment}
The source of the inconsistency lies in the fact that while our two works share a set of assumptions, to ensure ergodicity of the dynamics and smoothness of the Kolmogorov differential equation, \citet{chen2015convergence} requires some extra assumptions concerning the Poisson equation to prove the convergence rates. The formal justification of these extra assumptions, explicitly criticized in \citet{li2019stochastic}, has been deemed as out of the scope of \citet{chen2015convergence}. 
\end{comment}

In \cref{sec:toypexposition}, we present a simple toy example with known analytic solution, which contradicts the claims of \citet{chen2015convergence}, while our geometrical interpretation is compatible nevertheless.
We consider an analytical integrator that corresponds to an arbitrarily high order numerical integration. We show that when considering mini-batches the stationary distribution is not the desired one, hinting at the presence of a bottleneck in accordance with our theory. 
\subsection{Model details}\label{sec:toypexposition}
In this section we introduce a toy example -- whose dynamics has closed form solutions -- to understand the role of mini-batching. For simplicity we consider a one dimensional case.

The prior distribution for the parameter $\theta$ is
\begin{equation}
    p(\theta)=\mathcal{N}(\theta;0,\sigma_\theta^2),
\end{equation}
while the likelihood for any given observation is
\begin{equation}
    p(x|\theta)=\mathcal{N}(x;\theta,\sigma_x^2),
\end{equation}
where $\sigma_\theta,\sigma_x$ are arbitrary positive constants.

We consider the simplest case in which the dataset is composed by two datapoints, i.e. $N=2$, $\dataset=\{x_1,x_2\}$.
Simple calculations show that the posterior distribution is of the form
\begin{equation}\label{eq:toyp}
    p(\theta|\dataset)=\mathcal{N}\left(\theta;\frac{x_1+x_2}{v},\sigma_l^2\right),
\end{equation}
where $v=\frac{\sigma_x^2}{\sigma_\theta^2}+2$ and $\sigma_l^2=\left(\frac{1}{\sigma_\theta^2}+\frac{2}{\sigma_x^2}\right)^{-1}$.

The potential $U(\theta)$ corresponding to the distribution \cref{eq:toyp} has expression
\begin{equation}\label{eq:pot0}
    U(\theta)=\frac{1}{2\sigma_l^2}\left(\theta-\left(\frac{x_1+x_2}{v}\right)\right)^2+\text{const}.
\end{equation}
Notice that, up to constants of $\theta$, the potential can be rewritten as
\begin{equation}\label{eq:pot}
U(\theta)=\frac{1}{4\sigma_l^2}\left(\theta-\frac{2x_1}{v}\right)^2+\frac{1}{4\sigma_l^2}\left(\theta-\frac{2x_2}{v}\right)^2. 
\end{equation}
Consequently the gradient has form
\begin{equation}\label{eq:grapot}
    \nabla U(\theta)=\frac{1}{\sigma_l^2}\left(\theta-\frac{x_1+x_2}{v}\right)=\frac{1}{2\sigma_l^2}\left(\theta-\frac{2x_1}{v}\right)+\frac{1}{2\sigma_l^2}\left(\theta-\frac{2x_2}{v}\right).
\end{equation}
The potential and gradient with a sampled mini-batch corresponding to \cref{eq:pot} and \cref{eq:grapot} are respectively
\begin{equation}\label{eq:potmb}
U_i(\theta)=\frac{1}{2\sigma_l^2}\left(\theta-\frac{2x_i}{v}\right)^2,
\end{equation}
and
\begin{equation}\label{eq:grapotmb}
    \nabla U_i(\theta)=\frac{1}{\sigma_l^2}\left(\theta-\frac{2x_i}{v}\right).
\end{equation}

As shown in \cref{sec:toyanalyticintegr}, it is possible to build analytical integrators, that provide perfect simulation of paths. These integrators, that correspond to arbitrarily high order numerical integrators, are used for a numerical simulation with $\eta=0.4, C=2,\sigma_x^2=2,\sigma_\theta^2=0.5$ and two sampled datapoints $x_1=4,x_2=-3.2$.  The resulting posterior distributions are summarized in the histograms (100K samples) of \cref{fig:toyminibach}. Even considering a perfect integrator, when considering mini-batches, the stationary distribution is not the desired one. Importantly, this is a different result from the one suggested in \citet{chen2015convergence} where it is claimed that asymptotically in time only the order of the numerical integrator determines the convergence rate.

% {\color{red} Notice that the full batch is supposed to achieve exactly the stationary distribution, the epoch scramble (what we have in the main) $\mathcal{O}(\eta^2)$ and sampling at each step $\mathcal{O}(\eta^1)$. The proof for this last claim is not included here}.

% \dimitris{what is epoch scramble and sample scramble? These are not introduced anywhere!}

\begin{figure}
    \centering
    \includegraphics[width=0.8\textwidth]{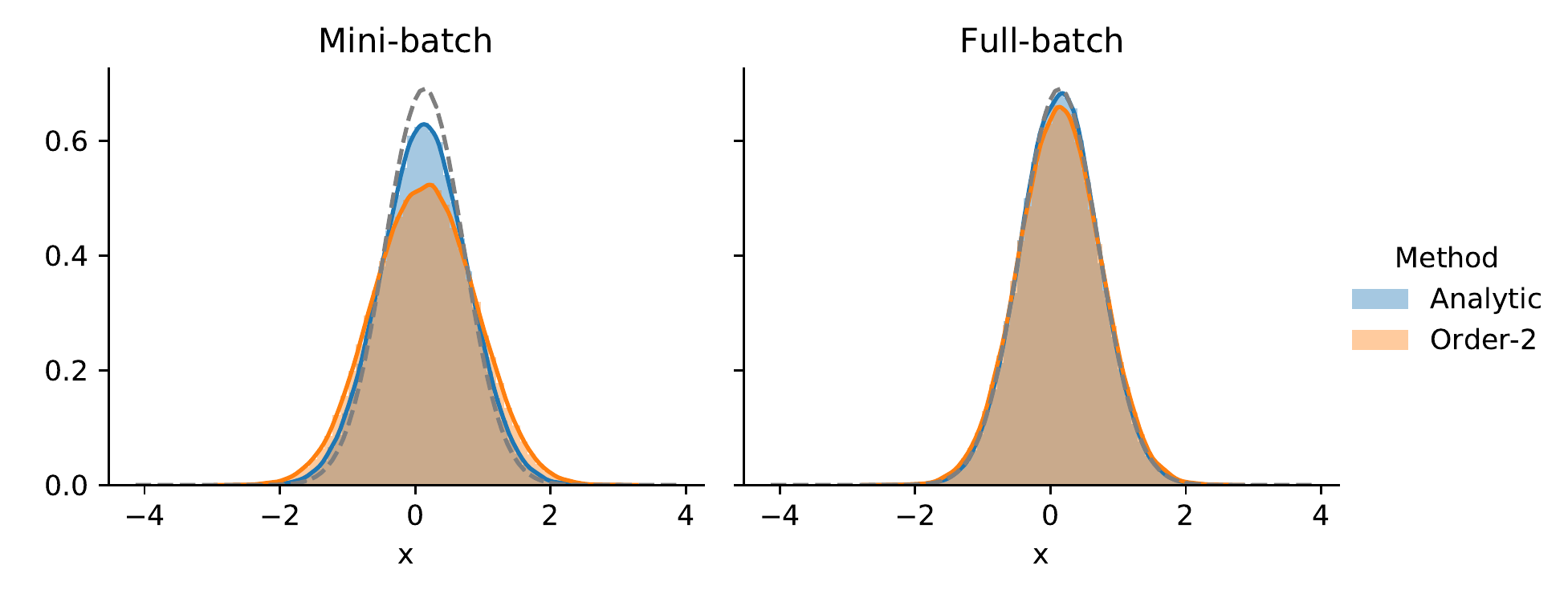}
    \caption{Histograms of stationary distributions (100K samples). The grey dotted line denotes the true posterior density. Irrespectively of order, mini-batching prevents from convergence to true posterior.}
    \label{fig:toyminibach}
\end{figure}

We further support our numerical evidence with a through theoretical exploration in \cref{sec:toyanalyticintegr}.

\subsection{Derivation of the analytical integrators}\label{sec:toyanalyticintegr}
We consider the dynamics of \cref{hamsde}, where for simplicity $\mathbf{M}=1$. We then have, remembering $\zvect(t)=\begin{bmatrix}r(t)\\ \theta(t)\end{bmatrix}$, the following \sde

\begin{equation}
    \dzvect(t)=\begin{bmatrix}-C&-1\\1&0
    \end{bmatrix}\begin{bmatrix}r(t)\\ {\sigma_l^{-2}}(\theta(t)-\frac{x_1+x_2}{v})\end{bmatrix}dt+\begin{bmatrix}\sqrt{2C}dw(t)\\0
    \end{bmatrix},
\end{equation}
that we rewrite in the more convenient form 
\begin{equation}\label{fullb}
    \dzvect(t)=\begin{bmatrix}-C&-{\sigma_l^{-2}}\\1&0
    \end{bmatrix}\left(\zvect(t)-\begin{bmatrix}0\\ \bar{x}\end{bmatrix}\right)dt+\begin{bmatrix}\sqrt{2C}dw(t)\\0
    \end{bmatrix},
\end{equation}
where $\bar{x}=\frac{x_1+x_2}{v}$.

Starting from initial conditions $\zvect(0)=\zvect_0$, we can shown that the system has analytic solution
\begin{equation}\label{eq:ansol}
    \zvect(t)=\exp\left(t\mathbf{A}\right)\left(\zvect_0-\begin{bmatrix}0\\ \bar{x}\end{bmatrix}\right)+\begin{bmatrix}0\\ \bar{x}\end{bmatrix}+\int\limits_{s=0}^{t}\exp\left((t-s)\mathbf{A}\right)\begin{bmatrix}\sqrt{2C}dw(s)\\0
    \end{bmatrix}
\end{equation}
where for simplicity we have defined $\mathbf{A}=\begin{bmatrix}-C&-{\sigma_l^{-2}}\\1&0\end{bmatrix}$.
Considering that ${\sigma_l^{-2}},C>0$ simple calculations show that the eigenvalues of $\mathbf{A}$ have negative real part, ensuring that as $t\rightarrow\infty$ we have $\exp\left(t\mathbf{A}\right)\rightarrow\zerovect$. 

A probabilistically equivalent representation of eq. \cref{eq:ansol} is the following
\begin{equation}\label{eq:algosol}
    \zvect(t)=\exp\left(t\mathbf{A}\right)\left(\zvect_0-\begin{bmatrix}0\\ \bar{x}\end{bmatrix}\right)+\begin{bmatrix}0\\ \bar{x}\end{bmatrix}+\mathbf{n},
\end{equation}
where
\begin{equation}
    \mathbf{n}\sim\mathcal{N}\left(\zerovect, \int\limits_{s=0}^{t}\exp\left((t-s)\mathbf{A}\right)\begin{bmatrix}2C&0\\0&0
    \end{bmatrix}\exp\left((t-s)\mathbf{A}^\top\right)ds\right).
\end{equation}
Equation \cref{eq:algosol} guarantees that it is possible to build a numerical integrator for any chosen step size $\eta$ that simulates exactly the dynamics, corresponding to an arbitrarily high order of weak integration.

\subsubsection*{Convergence to the posterior}

We then explore the transition probability starting from a given initial state $\zvect_0$ induced by such a numerical integrator with step size $\eta$
\begin{flalign}\label{eq:transp}
    &p(\zvect(\eta)|\zvect(0)=\zvect_0)=\nonumber\\
    &\mathcal{N}\left(\zvect(\eta);\exp\left(\eta\mathbf{A}\right)\left(\zvect_0-\begin{bmatrix}0\\ \bar{x}\end{bmatrix}\right)+\begin{bmatrix}0\\ \bar{x}\end{bmatrix},\int\limits_{s=0}^{\eta}\exp\left((\eta-s)\mathbf{A}\right)\begin{bmatrix}2C&0\\0&0
    \end{bmatrix}\exp\left((\eta-s)\mathbf{A}^\top\right)ds\right).
\end{flalign}

As shown in \cref{sec:proofcov}, the covariance matrix of the distribution \cref{eq:transp} can be expressed as
\begin{equation}\label{eqexp}
 \int\limits_{s=0}^{\eta}\exp\left((\eta-s)\mathbf{A}\right)\begin{bmatrix}2C&0\\0&0
    \end{bmatrix}\exp\left((\eta-s)\mathbf{A}^\top\right)ds=\begin{bmatrix}1&0\\0&{\sigma_l^{2}}\end{bmatrix}-\exp\left(\eta\mathbf{A}\right)\begin{bmatrix}1&0\\0&{\sigma_l^{2}}\end{bmatrix}\exp\left(\eta\mathbf{A}^\top\right),
\end{equation}
and consequently, we can rewrite \cref{eq:transp} as
\begin{flalign}
&p(\zvect(\eta)|\zvect(0)=\zvect_0) =  \nonumber\\& \mathcal{N}\left(\zvect(\eta);\exp\left(\eta\mathbf{A}\right)\left(\zvect_0-\begin{bmatrix}0\\ \bar{x}\end{bmatrix}\right)+\begin{bmatrix}0\\ \bar{x}\end{bmatrix},\mathbf{\Sigma}(\eta)\right),
\end{flalign}
where we introduced the $\mathbf{\Sigma}(\eta)=\begin{bmatrix}1&0\\0&{\sigma_l^{2}}\end{bmatrix}-\exp\left(\eta\mathbf{A}\right)\begin{bmatrix}1&0\\0&{\sigma_l^{2}}\end{bmatrix}\exp\left(\eta\mathbf{A}^\top\right)$.

Importantly, we have that $\rho_{ss}(\zvect)$ is the stationary distribution of the stochastic process, if and only if the following holds:
\begin{equation}\label{eigenker}
  \rho_{ss}(\zvect^*)=\int p(\zvect(\eta)=\zvect^*|\zvect(0)=\zvect)\rho_{ss}(\zvect)\dzvect
\end{equation}
where obviously the equality must hold for all $\eta$. To check that eq. \cref{eigenker} holds with:
\begin{equation}
\rho_{ss}(\zvect)=\mathcal{N}(r;0,1)\, \mathcal{N}\left(\theta;\frac{x_1+x_2}{v},\sigma_l^2\right)=\mathcal{N}\left(\zvect;\begin{bmatrix}0\\\bar{x}\end{bmatrix},\begin{bmatrix}1&0\\0&\sigma_l^2\end{bmatrix}\right)
\end{equation} we substitute directly:
\begin{flalign*}
&\int p(\zvect(\eta)=\zvect^*|\zvect(0)=\zvect)\rho_{ss}(\zvect)\dzvect=\\&\int \mathcal{N}\left(\zvect^*;\exp\left(\eta\mathbf{A}\right)\left(\zvect-\begin{bmatrix}0\\ \bar{x}\end{bmatrix}\right)+\begin{bmatrix}0\\ \bar{x}\end{bmatrix},\mathbf{\Sigma}(\eta)\right)
\mathcal{N}\left(\zvect;\begin{bmatrix}0\\\bar{x}\end{bmatrix},\begin{bmatrix}1&0\\0&\sigma_l^2\end{bmatrix}\right)\dzvect\\
&=\mathcal{N}\left(\zvect^*;\exp\left(\eta\mathbf{A}\right)\left(\begin{bmatrix}0\\ \bar{x}\end{bmatrix}-\begin{bmatrix}0\\ \bar{x}\end{bmatrix}\right)+\begin{bmatrix}0\\ \bar{x}\end{bmatrix},\mathbf{\Sigma}(\eta) +\exp\left(\eta\mathbf{A}\right)\begin{bmatrix}1&0\\0&{\sigma_l^{2}}\end{bmatrix}\exp\left(\eta\mathbf{A}^\top\right)\right)\\
&=\mathcal{N}\left(\zvect^*;\begin{bmatrix}0\\\bar{x}\end{bmatrix},\begin{bmatrix}1&0\\0&{\sigma_l^{2}}\end{bmatrix}\right)
\end{flalign*}
proving that the stationary distribution is indeed the desired one.

\subsubsection*{Failure of convergence in the case of mini-batches}

In the following, we show that when considering mini-batches instead, we fail to converge to the true posterior. Importantly this is true even with the analytical solution for the steps of the integrator, stressing that numerical integration and mini-batching have two independent effects.
With mini-batches, at every step of the integration a $x_1,x_2$ are sampled with probability $\frac{1}{2}$. %{\color{red} Notice that this is a different, simpler scheme than the randomized ordering scheme that we present in Theorem \ref{randomorderop}.}

Instead of simulating eq. \cref{fullb} at every step of the numerical integration, the following stochastic process is then considered:
\begin{equation}\label{minib}
    \dzvect(t)=\begin{bmatrix}-C&-{\sigma_l^{-2}}\\1&0
    \end{bmatrix}\left(\zvect(t)-\begin{bmatrix}0\\ \bar{x}_i\end{bmatrix}\right)dt+\begin{bmatrix}\sqrt{2C}dw(t)\\0
    \end{bmatrix}
\end{equation}
where, $\bar{x}_i=\frac{2x_i}{v}$, and $i=1$ or $i=2$ is sampled randomly. 

Similarly to the previous case it is possible to construct a numerical integrator that solves exactly the dynamics, this time with the potential computed using randomly only one of the two datapoints. The transition probability induced by such an integrator will then be equal to: 
\begin{flalign}
    p(\zvect(\eta)=\zvect^*|\zvect(0)=\zvect_0)&=\frac{1}{2}\mathcal{N}\left(\zvect^*;\exp\left(\eta\mathbf{A}\right)\left(\zvect_0-\begin{bmatrix}0\\ \bar{x}_1\end{bmatrix}\right)+\begin{bmatrix}0\\ \bar{x}_1\end{bmatrix},\mathbf{\Sigma}(\eta) \right)\nonumber\\
    &+\frac{1}{2}\mathcal{N}\left(\zvect^*;\exp\left(\eta\mathbf{A}\right)\left(\zvect_0-\begin{bmatrix}0\\ \bar{x}_2\end{bmatrix}\right)+\begin{bmatrix}0\\ \bar{x}_2\end{bmatrix},\mathbf{\Sigma}(\eta) \right)\\
    &=\frac{1}{2}f_1(\zvect^*,\eta,\zvect_0)+\frac{1}{2}f_2(\zvect^*,\eta,\zvect_0),
\end{flalign}
where $f_1,f_2$ are the two Gaussian transition probabilities.
Importantly, this allows to state the following central statement: the distribution $p(\theta|\mathcal{D})$ is not the stationary distribution of the process. 
This claim can be proven by contradiction.
Suppose that the stationary distribution is the one of interest, $\mathcal{N}\left(\zvect;\begin{bmatrix}0\\\bar{x}\end{bmatrix},\begin{bmatrix}1&0\\0&{\sigma_l^{2}}\end{bmatrix}\right)$. Performing one step of integration the new distribution will be of the form:
\begin{flalign}\label{transminib}
    \rho_{new}(\zvect^*) &= \int \left(\frac{1}{2}f_1(\zvect^*,\eta,\zvect)+\frac{1}{2}f_2(\zvect^*,\eta,\zvect)\right)\rho_{ss}(\zvect)\dzvect\\ 
    &=\nonumber\frac{1}{2}\mathcal{N}\left(\zvect^*;\exp\left(\eta\mathbf{A}\right)\left(\begin{bmatrix}0\\ \bar{x}\end{bmatrix}-\begin{bmatrix}0\\ \bar{x}_1\end{bmatrix}\right)+\begin{bmatrix}0\\ \bar{x}\end{bmatrix},\begin{bmatrix}1&0\\0&{\sigma_l^{2}}\end{bmatrix}\right)\\
    &+\frac{1}{2}\mathcal{N}\left(\zvect^*;\exp\left(\eta\mathbf{A}\right)\left(\begin{bmatrix}0\\ \bar{x}\end{bmatrix}-\begin{bmatrix}0\\ \bar{x}_2\end{bmatrix}\right)+\begin{bmatrix}0\\ \bar{x}\end{bmatrix},\begin{bmatrix}1&0\\0&{\sigma_l^{2}}\end{bmatrix}\right)\\&\neq \rho_{ss}(\zvect^*).
\end{flalign}
By definition, we should have the functional equivalence between $\rho_{new}$ and $\rho_{ss}$.
Equation \cref{transminib} implies however that after the mixture transition probability the new density $\rho_{new}$ will be a mixture of Gaussian distributions, and consequently $\rho_{new}\neq \rho_{ss}$. This concludes the demonstration that the stationary distribution is not the one of interest. 

%{\color{red} Despite the stationary distribution is not the desired one, the mean is correct, so it is not too drastic...}

%Notice however that an analytical solution for the stationary distribution is out of reach.

\subsubsection{Derivation of eq. \cref{eqexp} }\label{sec:proofcov}
Define:
\begin{equation}
    \mathbf{K}(t)=\int\limits_{s=0}^{t}\exp\left((t-s)\mathbf{A}\right)\begin{bmatrix}2C&0\\0&0
    \end{bmatrix}\exp\left((t-s)\mathbf{A}^\top\right)ds,
\end{equation}
and
\begin{equation}
    \mathbf{S}(t)=\begin{bmatrix}1&0\\0&{\sigma_l^{2}}\end{bmatrix}-\exp\left(t\mathbf{A}\right)\begin{bmatrix}1&0\\0&{\sigma_l^{2}}\end{bmatrix}\exp\left(t\mathbf{A}^\top\right).
\end{equation}

Immediately we see that $\mathbf{K}(0)=\mathbf{S}(0)=\zerovect$.

The derivative of $\mathbf{K}(t)$ has expression:
\begin{flalign*}
    \frac{d\mathbf{K}(t)}{dt}&=\exp\left((t-t)\mathbf{A}\right)\begin{bmatrix}2C&0\\0&0
    \end{bmatrix}\exp\left((t-t)\mathbf{A}^\top\right)\\
    &+\mathbf{A}\int\limits_{s=0}^{t}\exp\left((t-s)\mathbf{A}\right)\begin{bmatrix}2C&0\\0&0
    \end{bmatrix}\exp\left((t-s)\mathbf{A}^\top\right)ds\\
    &+\int\limits_{s=0}^{t}\exp\left((t-s)\mathbf{A}\right)\begin{bmatrix}2C&0\\0&0
    \end{bmatrix}\exp\left((t-s)\mathbf{A}^\top\right)ds\mathbf{A}^\top\\
    &=\begin{bmatrix}2C&0\\0&0
    \end{bmatrix}+\mathbf{A}\mathbf{K}(t)+\mathbf{K}(t)\mathbf{A}^\top.
\end{flalign*}

Similarly we have:
\begin{flalign*}
    \frac{d\mathbf{S}(t)}{dt}&=-\mathbf{A}\exp\left(t\mathbf{A}\right)\begin{bmatrix}1&0\\0&{\sigma_l^{2}}\end{bmatrix}\exp\left(t\mathbf{A}^\top\right)-\exp\left(t\mathbf{A}\right)\begin{bmatrix}1&0\\0&{\sigma_l^{2}}\end{bmatrix}\exp\left(t\mathbf{A}^\top\right)\mathbf{A}^\top\\
    &=-\mathbf{A}\left(\begin{bmatrix}1&0\\0&{\sigma_l^{2}}\end{bmatrix}-\mathbf{S}(t)\right)-\left(\begin{bmatrix}1&0\\0&{\sigma_l^{2}}\end{bmatrix}-\mathbf{S}(t)\right)\mathbf{A}^\top.
\end{flalign*}

Since $\mathbf{A}\begin{bmatrix}1&0\\0&{\sigma_l^{2}}\end{bmatrix}=\begin{bmatrix}-C&-1\\1&0\end{bmatrix}$, we can prove
\begin{equation*}
  \frac{d\mathbf{S}(t)}{dt}=  \begin{bmatrix}2C&0\\0&0
    \end{bmatrix}+\mathbf{A}\mathbf{S}(t)+\mathbf{S}(t)\mathbf{A}^\top.
\end{equation*}

Definining the difference matrix $\mathbf{H}(t)=\mathbf{K}(t)-\mathbf{S}(t)$ we recognize that it satisfies the differential equation
\begin{equation*}
  \frac{d\mathbf{H}(t)}{dt}=  \mathbf{A}\mathbf{H}(t)+\mathbf{H}(t)\mathbf{A}^\top
\end{equation*}
that has explicit solution
\begin{equation}
    \mathbf{H}(t)=\exp\left(t\mathbf{A}\right)\mathbf{H}(0)\exp\left(t\mathbf{A}^\top\right)
\end{equation}
Since $\mathbf{H}(0)=\zerovect$ we conclude $\mathbf{H}(t)=\zerovect$ proving the desired equality.

\section{EXPERIMENTS: ADDITIONAL DETAILS}\label{sec:experiments_appendix}

\begin{comment}
\paragraph{Erratum:} The caption of \cref{fig:methods} in the main paper corresponds to an out-of-date version of the figure; the comparison with \hmc is now done in \cref{fig:generalizedLT} of the paper.
The correct caption should be: \lq\lq Exploration of step size and batch size for different Hamiltonian-based methods; the grey dotted line denotes the self-distance for the distribution of the oracle\rq\rq.
\end{comment}

\subsection{Experimental setup}
\label{ssec:experiment_setup_appendix}

Throughout our experimental campaign, we work on a number of datasets from the UCI repository\footnote{\url{https://archive.ics.uci.edu/ml/index.php}}.
For each dataset, we have considered $5$ random splits into training and test sets; for the regression datasets, we have adopted the splits in \citet{gal2018w} which are available online\footnote{\url{https://github.com/yaringal/DropoutUncertaintyExps}}  under the Creative Commons licence.
The primary point of interest is in this work is the predictive distribution over the test set and how this deviates from the true posterior.
\Cref{tab:regression} and \cref{tab:classification} summarize the details of the regression and classification datasets in this work.
In the same tables we report some error metrics; the latter are not a central part of the evaluation of this paper, but they are simply indicative of the performance of \sdehmc.
Our implementation is loosely based on the PYSGMCMC framework\footnote{\url{https://github.com/MFreidank/pysgmcmc}}.

\paragraph{Network architecture.}
Throughout this experimental campaign, we consider \bnns featuring $L$ layers, where each layer is defined as follows:
\begin{equation}
f_{l}(\mathbf{x}) = \frac{1}{\sqrt{D_{l-1}}} \mathbf{W}_{l}\, \varphi(f_{l-1}(\mathbf{x}))  + \mathbf{b}_{l}, \quad l \in \{1, ..., L+1\},
\end{equation}
where $\varphi$ is a non-linear activation function.
The model parameters are summarized as $\thetavect \equiv \{\mathbf{W}_{l} \in \mathbb{R}^{D_l \times D_{l-1}}, \mathbf{W}_{l} \in \mathbb{R}^{D_l \times D_{l-1}}\}$, and they denote the matrix of weights and the vector of biases for layer $l$.
Note that we divide by the square root of the input dimension $D_{l-1}$; this scheme is known as the \emph{NTK parameterization} \citep{jacot2018neural,lee2020finite}, and it ensures that the asymptotic variance neither explodes nor vanishes.
In this context, we place a standard Gaussian prior for both weights and biases, i.e.\ $p(\thetavect) = \mathcal{N}(0, I)$.

\paragraph{Regression and classification tasks.}
For regression tasks we consider \bnns with 4 layers and 50 nodes per layer with \relu activation.
\Cref{tab:regression} outlines the predictive performance of \sdehmc for four regression datasets (\boston, \concrete, \energy and \yacht), as well as the noise variance $\sigma^2$ used in each case.
We have drawn $200$ samples from the posterior and the regression performance is evaluated in terms of \emph{root mean squared error} (\rmse) and \emph{mean negative log-likelihood} \mnll.
In order to put these values into perspective, we also show the results for the adaptive \sghmc scheme from \citet{springenberg2016bayesian} featuring a network of identical structure as in our setup.
In both cases, the full-batch gradient was used.

\begin{table*}[ht]
\caption{Regression results for \bnns with 4 layers and 50 nodes per layer with \relu activation. For \sghmc we follow \citet{springenberg2016bayesian}.} \label{tab:regression}
\begin{center}
{\footnotesize
\begin{tabular}{lrrr||cc|cc}
                 & \multirow{2}{0.8cm}{Training Size} & \multirow{2}{0.5cm}{Test Size} 
                 &
                 &\multicolumn{2}{c|}{\textbf{Test \rmse}$(\downarrow)$} &\multicolumn{2}{c}{\textbf{Test \mnll}$(\downarrow)$} \\
\textbf{DATASET} &     &     & $\sigma^2$ & \sdehmc        & \sghmc         &\sdehmc         & \sghmc \\
\hline \\
Boston           &455  &51   & 0.2        & 2.821$\pm$0.61 & 2.825$\pm$0.63 & 2.600$\pm$0.09 & 2.600$\pm$0.09 \\
Concrete         &927  &103  & 0.05       & 4.907$\pm$0.39 & 4.833$\pm$0.46 & 2.971$\pm$0.07 & 2.948$\pm$0.09 \\
Energy           &691  &77   & 0.01       & 0.501$\pm$0.07 & 0.489$\pm$0.07 & 1.191$\pm$0.07 & 1.120$\pm$0.03 \\
Yacht            &277  &31   & 0.005      & 0.420$\pm$0.12 & 0.436$\pm$0.13 & 1.180$\pm$0.02 & 1.176$\pm$0.02 \\
\end{tabular}
}
\end{center}
\end{table*}

For classification we consider \iono and \vehicle from UCI, for which we sample from \relu \bnns with 2 layers and 50 nodes per layer.
We have drawn $200$ samples from the posterior, and the predictive accuracy and test \mnll of \sdehmc can be seen in \cref{tab:classification}, where we also show the performance of the adaptive version of \sghmc \citep{springenberg2016bayesian}.
The results have been very similar for these two cases; this has been expected, as both methods are supposed to converge to the same posterior.

\begin{table*}[ht]
\caption{Classification results for \bnns with 2 layers and 50 nodes per layer with \relu activation. For \sghmc we follow \citet{springenberg2016bayesian}.} \label{tab:classification}
\begin{center}
{\footnotesize 
\begin{tabular}{lrrr||cc|cc}
                 &
                 & \multirow{2}{0.8cm}{Training Size} & \multirow{2}{0.5cm}{Test Size} 
                 &\multicolumn{2}{c|}{\textbf{Test accuracy}$(\uparrow)$} &\multicolumn{2}{c}{\textbf{Test \mnll}$(\downarrow)$} \\
\textbf{DATASET} & Classes &     &    & \sdehmc        & \sghmc         &\sdehmc         & \sghmc \\
\hline \\
Ionosphere       & 2       &301  &34  & 0.920$\pm$0.03 & 0.916$\pm$0.02 & 0.294$\pm$0.04 & 0.295$\pm$0.04 \\
Vehicle          & 4       &646  &200 & 0.777$\pm$0.02 & 0.783$\pm$0.02 & 0.588$\pm$0.03 & 0.588$\pm$0.03 \\
\end{tabular}
}
\end{center}
\end{table*}

\newcommand{\wvec}{\mathbf{w}}

\paragraph{Synthetic regression example.}
We also consider as simple regression model applied on the synthetic one-dimensional dataset of \cref{fig:synthetic}.
In this case we choose a linear model with fixed basis functions, in order to demonstrate \sdehmc convergence in a fully controlled environment where the true posterior can be calculated analytically.
We consider $D=256$ trigonometric basis functions: $f(x) = \sqrt{\nicefrac{2}{D}}\, \wvect^\top \cos(\omega x-\pi/4)$, where $\wvect \in \mathbb{R}^{D\times1}$ contains the weights of $D$ features and $\omega \in \mathbb{R}^{D\times1}$ is a vector of fixed frequencies.
In the experiments that follow, we have a Gaussian likelihood with variance $0.1$ and prior $p(\wvect) = \mathcal{N}(0, I_D)$.
The predictive posterior can be seen in \cref{fig:synthetic}, where the test set consists of $200$ uniformly distributed points.

\begin{figure}
    \centering
    \includegraphics[width=0.5\textwidth]{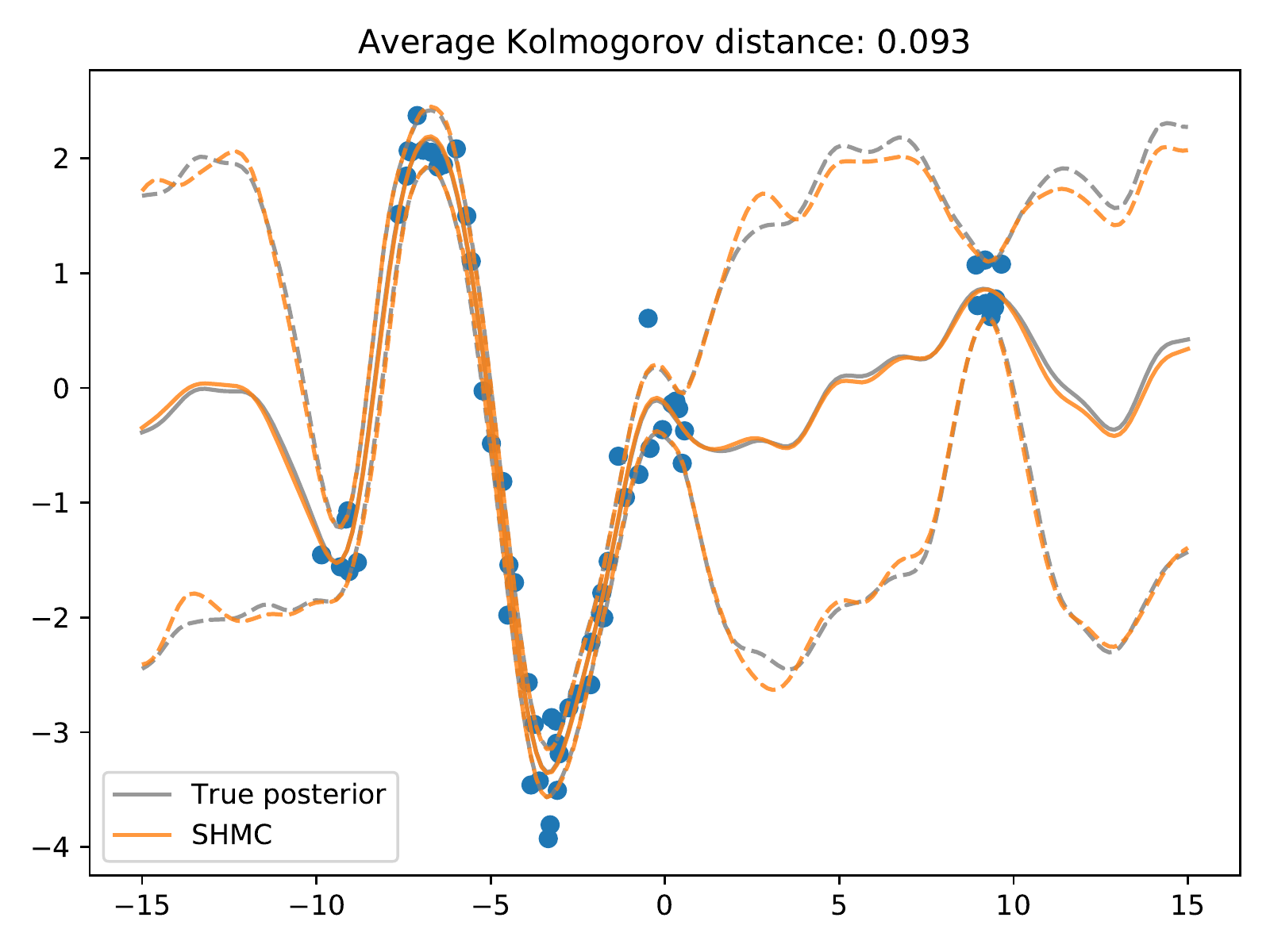}
    \caption{True and \sdehmc predictive posterior on a synthetic dataset.}
    \label{fig:synthetic}
\end{figure}

\subsection{Comparison framework and convergence to posterior}
\label{ssec:framework_appendix}

\paragraph{Comparing predictive distributions.}
In this work, we explore the behavior of a number of methods by comparing the predictive distribution given by a particular setting to the distribution of an oracle.
The comparison is performed in terms of one-dimensional predictive marginals using $200$ samples: for each test point we evaluate the Kolmogorov distance between the predictive distribution and the oracle.
We then report the average distance over the test set.
The Kolmogorov distance takes values between 0 and 1.
In \cref{fig:synthetic} we show an one-dimensional regression example where the average distance values over the test set is smaller than $0.1$.

\paragraph{Self-distance.}
When we compare the distance between empirical distribution, this will never be exactly zero.
In order to determine whether a given sample is a sufficiently good approximation for a target distribution, we have to compare the corresponding distance value to the \emph{self-distance}.
The latter can be evaluated by resampling from the oracle posterior distribution: for each oracle we consider $5$ independent runs producing $200$ samples each, which are used to estimate the Kolmogorov self-distance.
In the results of \cref{fig:methods_all}, \cref{fig:generalizedLT_all} and \cref{fig:friction_all}, the shaded area denotes the 0.05 and 0.95-quantiles of the self-distance distribution.

\paragraph{Methods summary.}
The main focus of this experimental campaign is to explore how \sdehmc is affected by changes of step size and batch size, and to observe whether the introduction of mini-batches induces a convergence bottleneck in practice or not.
For \sdehmc, we consider $4$ integrators: \leapfrog which is provably of order 1, \symmetric \citep{chen2015convergence} and \lietrotter (\cref{sec:hmcvssdehmc}) which are of order 2, and \mtthree \citep{milstein2003quasi} which is a third order integrator.
We also compare against \sghmc \citep{chen2014stochastic}.

%we acknowledge that a more realistic framework for \hmc is a framework such as the No-U-turn sampler \cite{}

\paragraph{Computational summary.}
At this point, we note that each block in \cref{fig:methods_all}, \cref{fig:generalizedLT_all} and \cref{fig:friction_all} corresponds to $5$ random splits or seeds.
For most datasets, drawing 200 samples for the smallest step size (i.e.\ $0.001$) required 15--20 minutes.
Such a small step size has been necessary in order to demonstrate convergence for small batch sizes.
A full exploration of the methods and integrators for a single split (or seed) of a single dataset required approximately one day of computation on a computer cluster featuring Intel\textregistered~Xeon\textregistered~CPU @2.00GHz.
It has not been possible to do this kind of exploration in a meaningful way for larger datasets, yet we believe that the existing results are sufficient to demonstrate our theoretical claims.

\paragraph{Convergence.}
In order to reason about convergence in the experiments that follow, we have taken great care to approximate the true posterior with \sghmc, which we treat as oracle.
For the oracles we consider full-batch and $\eta=0.005$ in all cases, except for the linear model where the true posterior has been analytically tractable.
The simulation time has been determined by adjusting the thinning parameter so that lag-1 autocorrelation (\acf) was considered as acceptable.
In practice we keep one sample every $500$ steps (i.e.\ thinning), and we discard the first $2000$ steps.
Then for each oracle, the simulation time was doubled one more time, but that resulted in no significant difference in the predictive distribution.
The summary of \acf values that correspond to the oracles used in this work can be found in \cref{tab:acf}.
As a final remark, we note that for each combination of dataset and model, all random walks cover the same simulation time.
The CPU time is thus determined by the step size $\eta$: a larger value for $\eta$ implies that a smaller number of steps is required to simulate a certain system, as the thinning parameter is adjusted accordingly.
%$0.005 \times 200 = 0.1$ units of simulation time.

\begin{table*}[ht]
\caption{Lag 1 autocorrelation for the \sghmc algorithms used as oracles.} \label{tab:acf}
\begin{center}
{\footnotesize
\begin{tabular}{lr}
DATASET          & \acf \\
\hline \\
Boston           & 0.09 \\
Concrete         & 0.23 \\
Energy           & 0.05 \\
Yacht            & 0.05 \\
Ionosphere       & 0.17 \\
Vehicle          & 0.18 \\
\end{tabular}
}
\end{center}
\end{table*}

\subsection{Extended regression and classification results}
\label{ssec:extended_results}

In this section we present the results of step size and batch size exploration for the entirety of datasets considered.
In all cases, we set $C=5$; we find this to be a reasonable choice, as we see in the exploration of \cref{ssec:explore_C}.

\Cref{fig:methods_all} focuses on the comparison between \sdehmc with different integrators (\leapfrog, \symmetric, \lietrotter and \mtthree) and \sghmc.
This is an extended version of \cref{fig:methods} of the main paper; we include all the datasets considered, as well as a more fine grained exploration of the batch size.

\Cref{fig:generalizedLT_all} is an extended version of \cref{fig:generalizedLT} of the main paper.
It contains a complete account of the exploration of the integration length $N_l$ for the (generalized) \lietrotter integrator.

\begin{figure}[h]
    \centering
    \includegraphics[width=\textwidth]{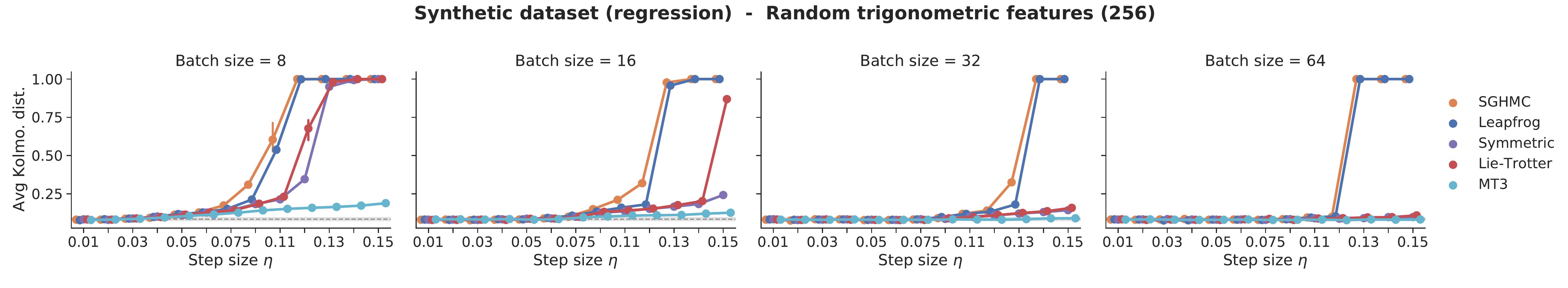}
    \includegraphics[width=\textwidth]{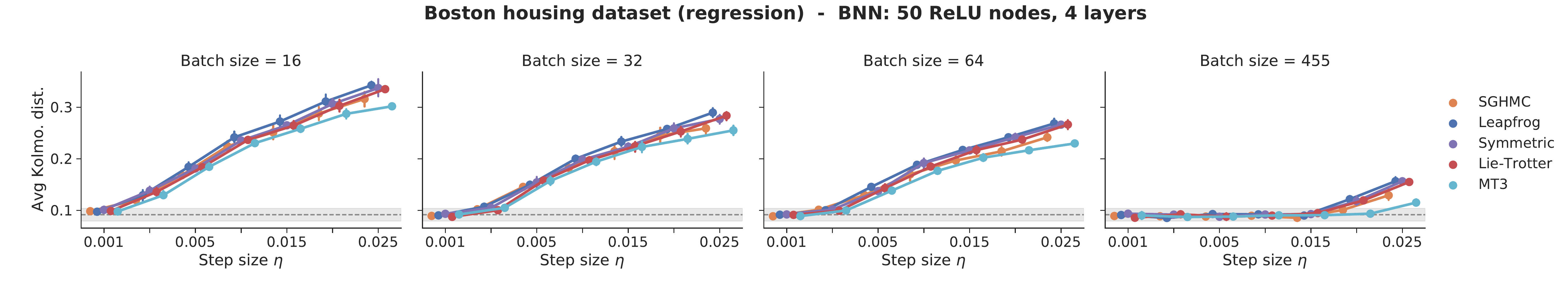}
    \includegraphics[width=\textwidth]{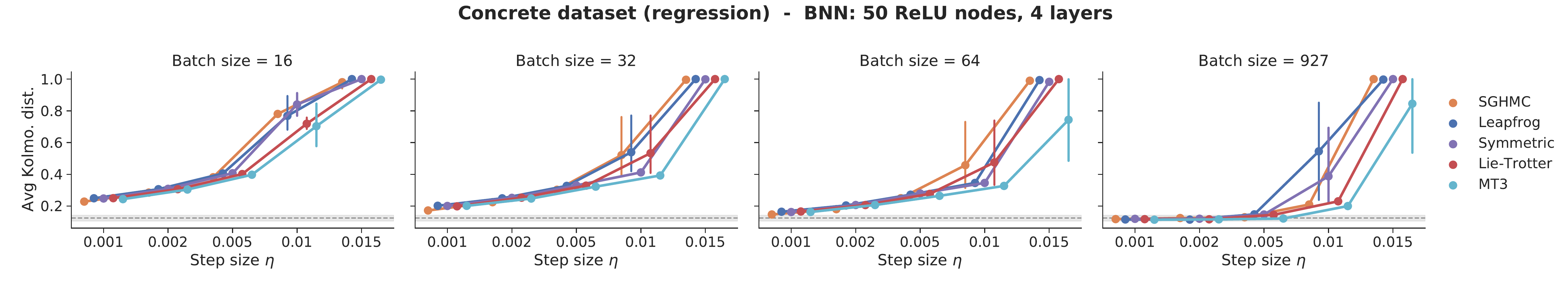}
    \includegraphics[width=\textwidth]{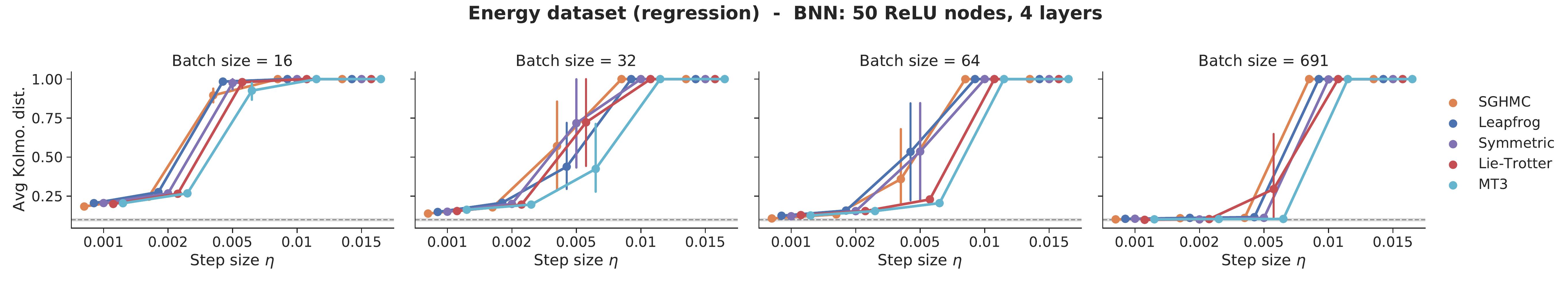}
    \includegraphics[width=\textwidth]{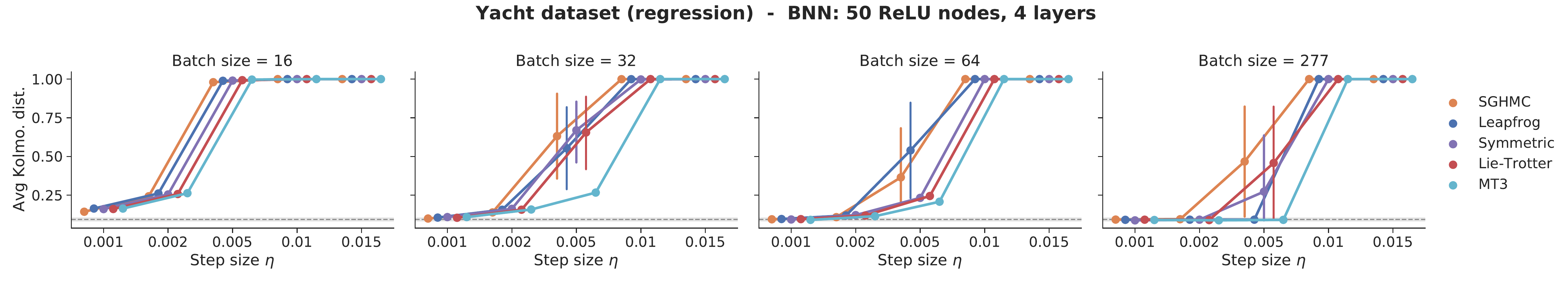}
    \includegraphics[width=\textwidth]{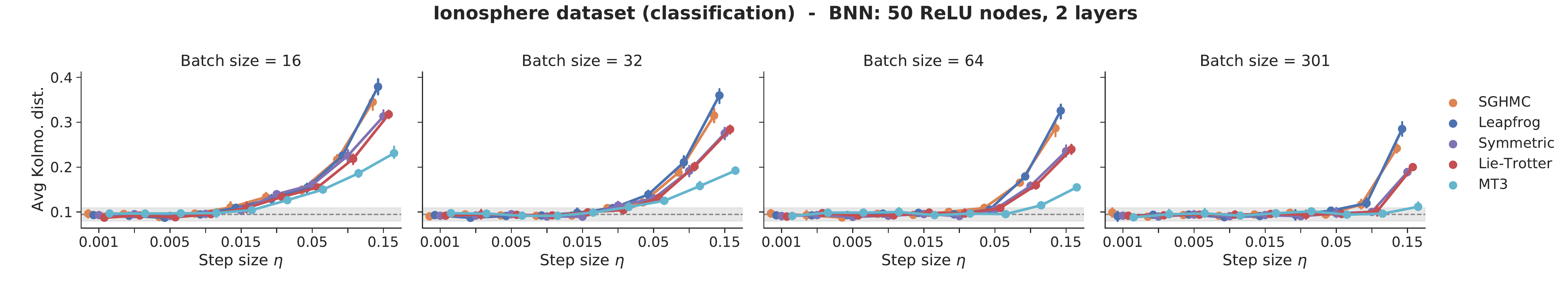}
    \includegraphics[width=\textwidth]{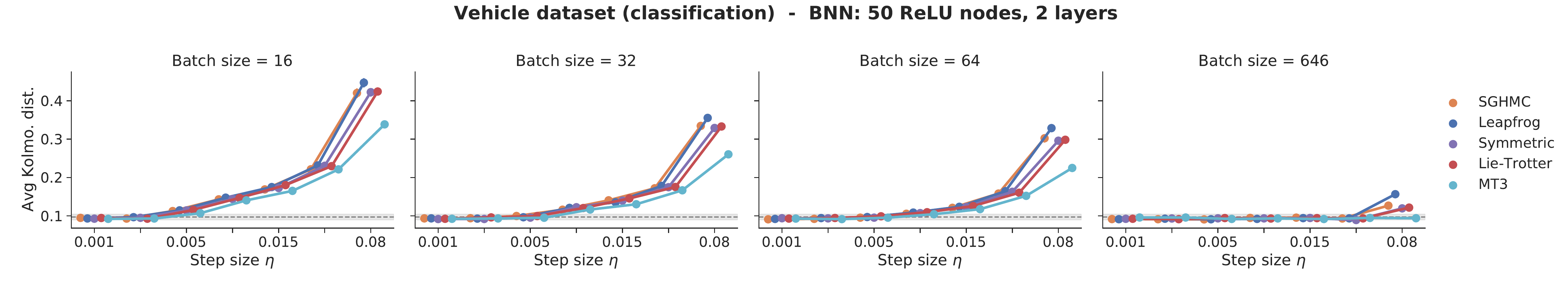}
    \caption{Exploration of step size and batch size for different Hamiltonian-based methods; the grey dotted line denotes the self-distance for the distribution of the oracle.}
    \label{fig:methods_all}
\end{figure}

\begin{figure}[h]
    \centering
    \includegraphics[width=\textwidth]{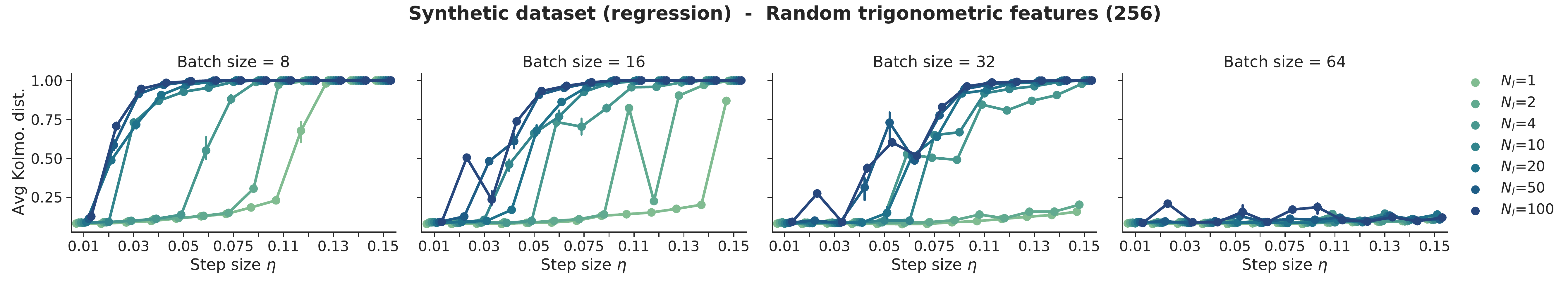}
    \includegraphics[width=\textwidth]{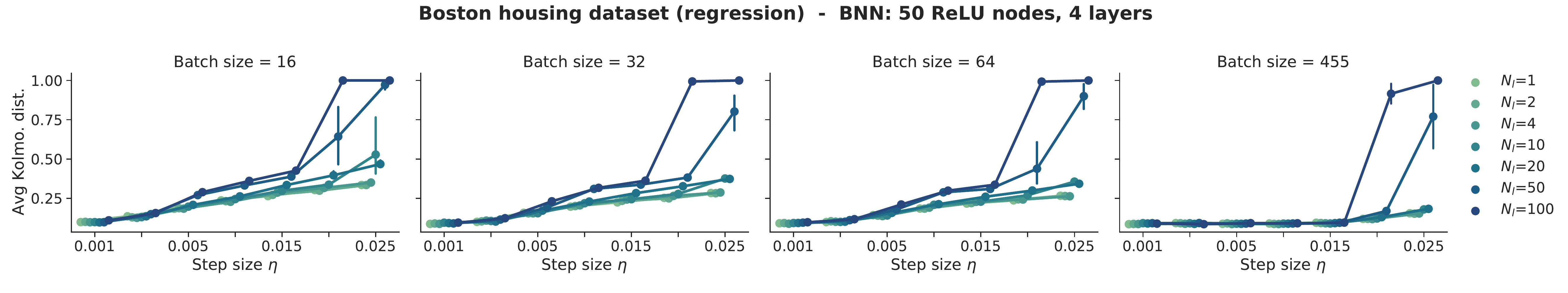}
    \includegraphics[width=\textwidth]{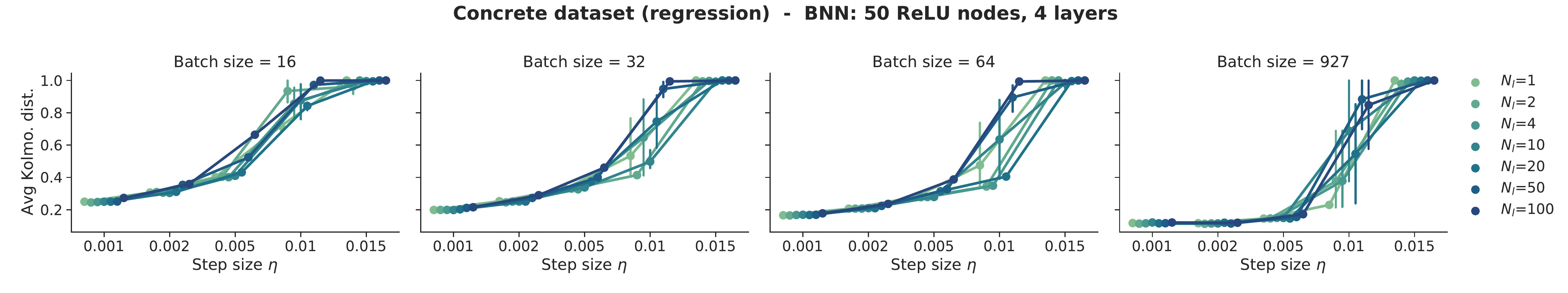}
    \includegraphics[width=\textwidth]{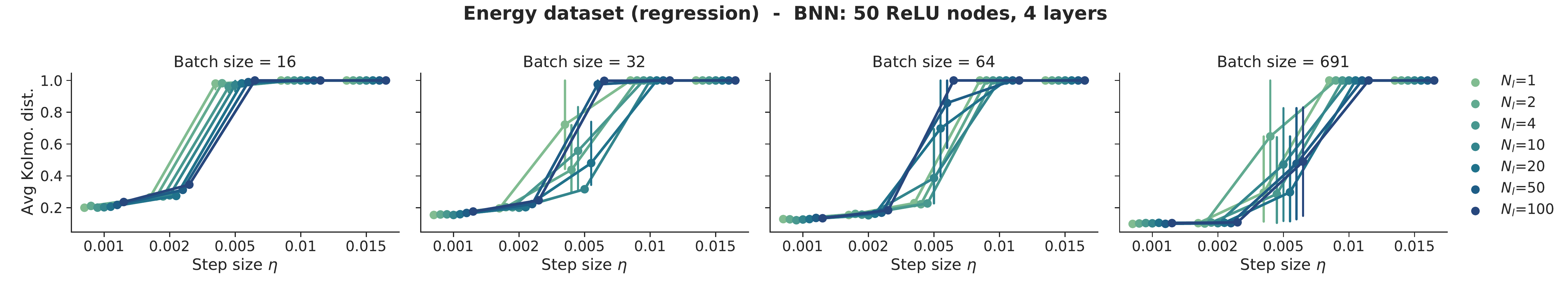}
    \includegraphics[width=\textwidth]{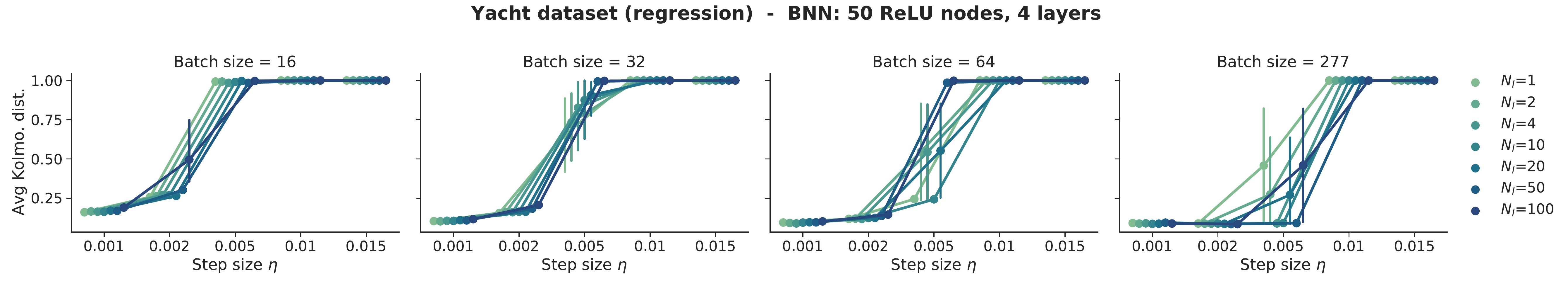}
    \includegraphics[width=\textwidth]{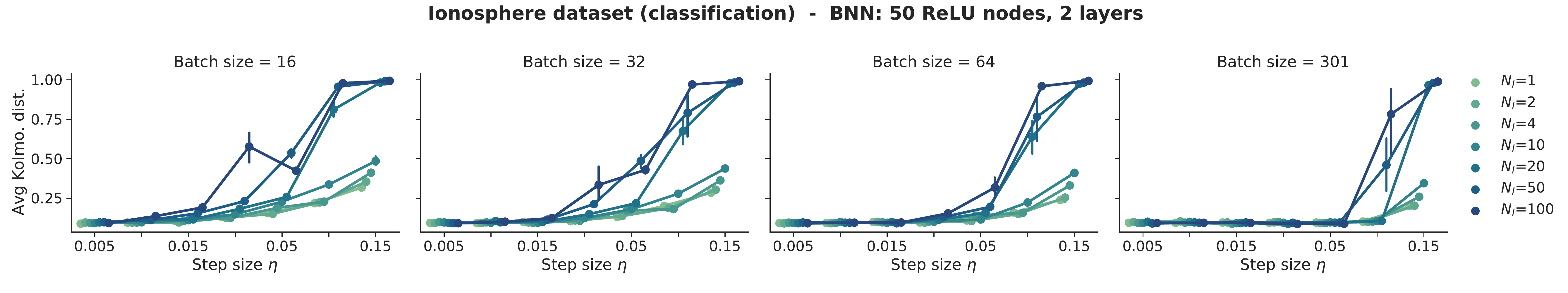}
    \includegraphics[width=\textwidth]{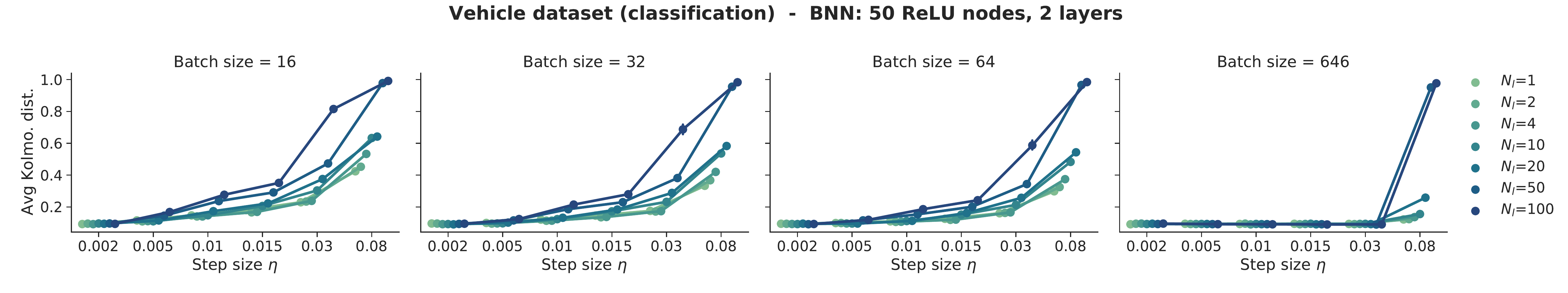}
    \caption{Generalized \lietrotter: Exploration of step size and batch size for different values of the (deterministic) integration length $N_l$. The grey dotted line denotes the self-distance for the distribution of the oracle.}
    \label{fig:generalizedLT_all}
\end{figure}

\subsection{Exploration of the friction constant $C$}
\label{ssec:explore_C}

The user-specified constant $C$ appears twice in \cref{hamsde}: in the friction term and in the stochastic diffusion term of the \sde.
Although \cref{th:stationary} implies that any choice for $C>0$ will maintain the desired stationary distribution (i.e.\ $\rho_{ss}(\thetavect) \propto \exp(-U(\thetavect))$), the transient dynamics of the \sde do change as we have seen in the sample paths of \cref{fig:sample_paths}. 
It is generally expected that as $C$ approaches $0$, the sample paths approach the deterministic behavior of an \ode.
On the other hand, if $C$ is too large, the stochastic process degenerates to the standard Brownian motion.
It is expected that either of these extremes would hurt the usability of \sde simulation as a sampling scheme, but the exact effect of $C$ is not apparent.

In this section, we experimentally examine how a certain choice for $C$ affects practical convergence to the desired posterior in conjunction with the step size $\eta$ as well as the mini-batch size.
\Cref{fig:friction_all} summarizes an extensive exploration of $C$ for all the models and datasets considered in this work using the \leapfrog integrator, with $C$ varying from $0.5$ to $100$.
Again, we measure the average Kolmogorov distance from the true predictive posterior for the test points, and we want to see whether the curve for varying $\eta$ and batch size approaches the band of self-distance.

As a general remark, the desired convergence properties are not too sensitive to the constant $C$.
Although, the optimal value for $C$ does depend on the dynamics induced by a particular dataset and model, it appears that there is a wide range of values that can be considered as acceptable.
For most of the cases considered, a value between $5$ and $10$ seems to produce reasonable results.
Nevertheless, we think that some manual exploration would still be required for a new dataset.

\begin{figure}[h]
    \centering
    \includegraphics[width=\textwidth]{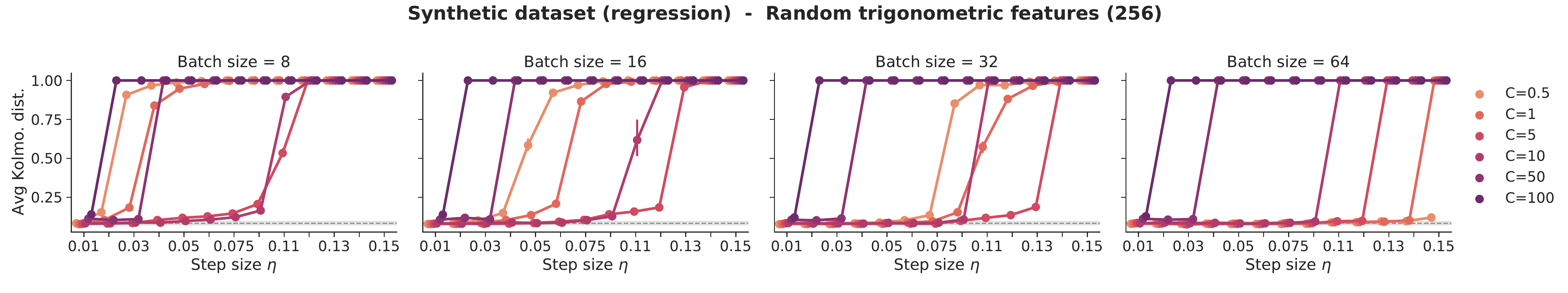}
    \includegraphics[width=\textwidth]{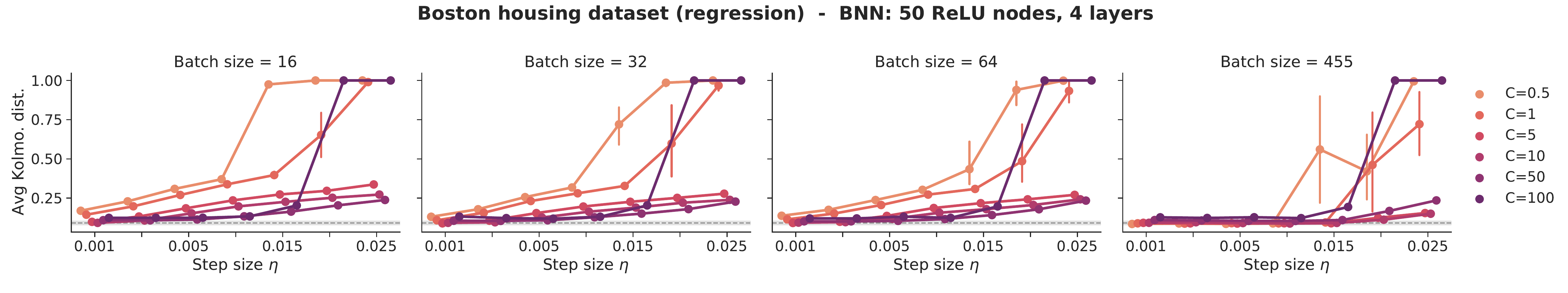}
    \includegraphics[width=\textwidth]{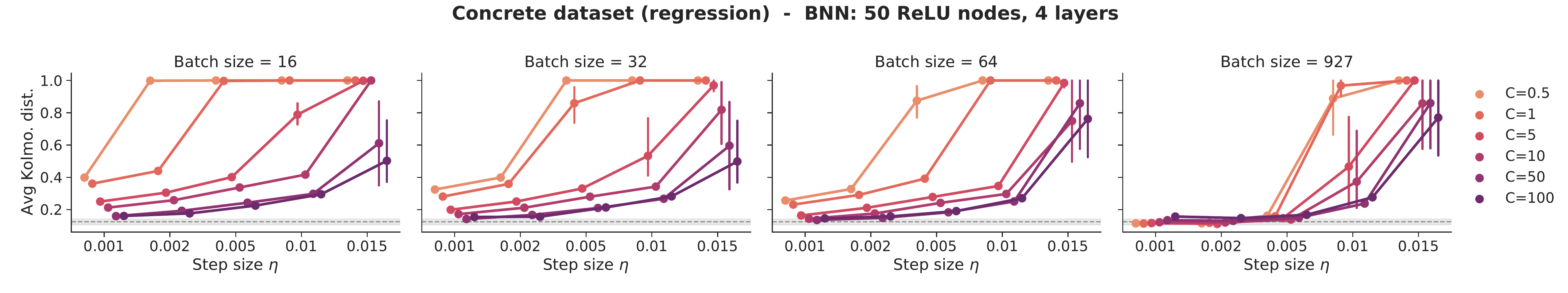}
    \includegraphics[width=\textwidth]{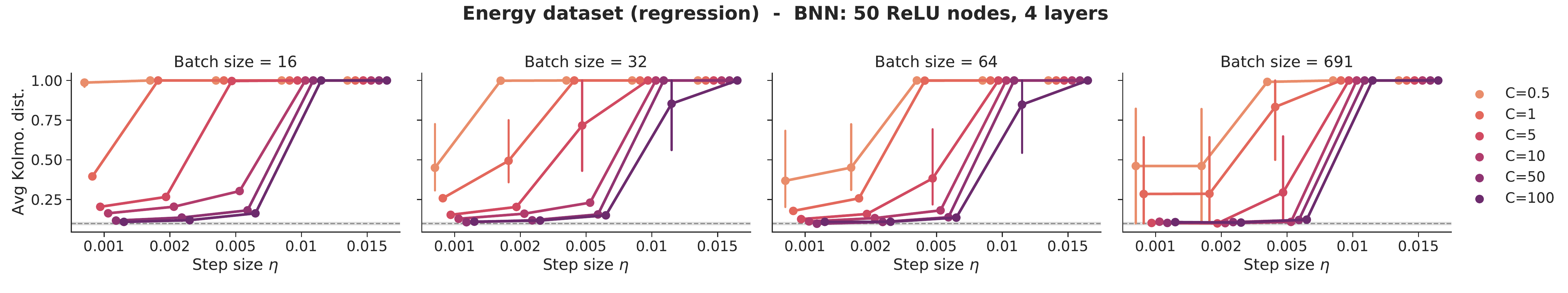}
    \includegraphics[width=\textwidth]{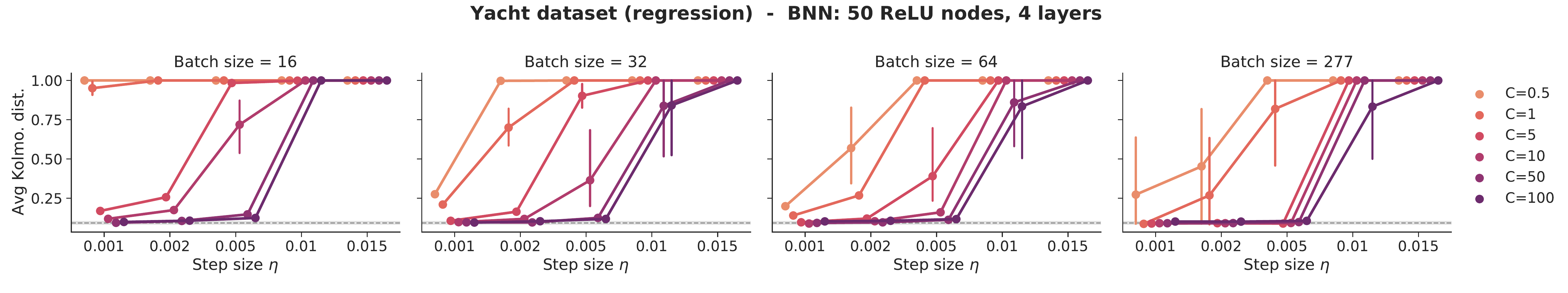}
    \includegraphics[width=\textwidth]{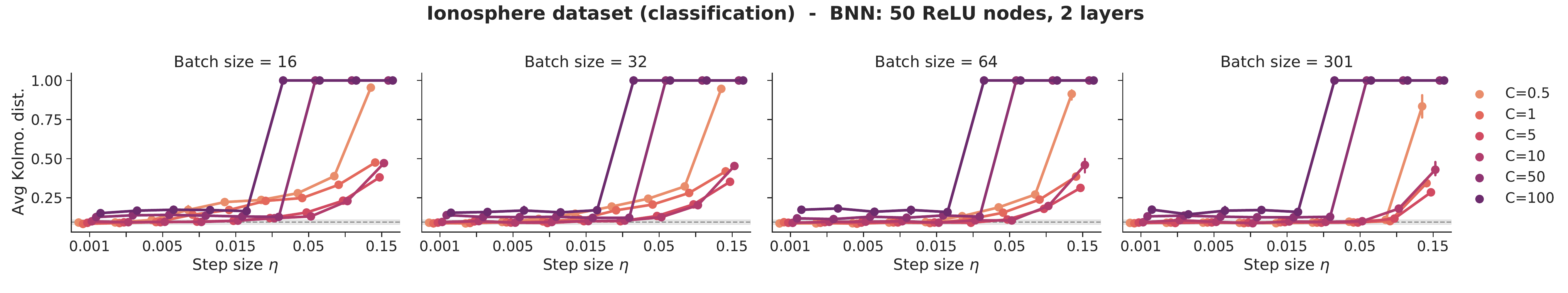}
    \includegraphics[width=\textwidth]{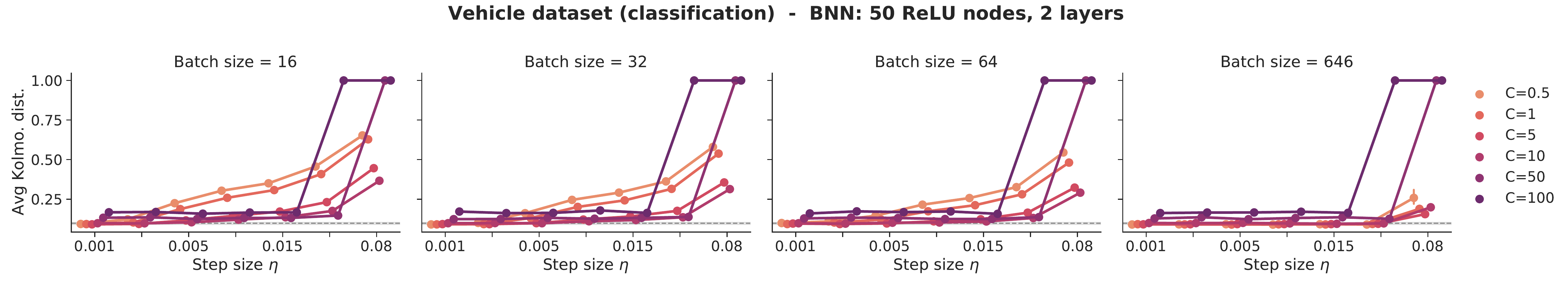}
    \caption{Exploration of step size and batch size for different values of a scalar friction coefficient $C$. The grey dotted line denotes the self-distance for the distribution of the oracle. \leapfrog is used in all cases.}
    \label{fig:friction_all}
\end{figure}

\newpage

\end{document}